\documentclass{article} 
\usepackage{iclr2027_conference,times}
\usepackage{booktabs}

\usepackage{algorithm}      
\usepackage{algpseudocode}  

\usepackage{amsmath,amsfonts,bm}

\def\eqref#1{equation~\ref{#1}}

\def\1{\bm{1}}

\newcommand{\test}{\mathcal{D_{\mathrm{test}}}}

\DeclareMathAlphabet{\mathsfit}{\encodingdefault}{\sfdefault}{m}{sl}
\SetMathAlphabet{\mathsfit}{bold}{\encodingdefault}{\sfdefault}{bx}{n}

\usepackage{hyperref}
\usepackage{url}
\usepackage{graphicx} 
\usepackage{booktabs}          
\usepackage{amsmath,amssymb}   
\usepackage{amsthm}            
\theoremstyle{plain}
\newtheorem{proposition}{Proposition}
\theoremstyle{definition}
\newtheorem{definition}{Definition}
\newtheorem{lemma}{Lemma}
\usepackage{listings}          
\usepackage{xcolor}
\usepackage{tikz}
\usetikzlibrary{arrows.meta,positioning}
\usepackage[most]{tcolorbox}   
\definecolor{seBlue}{HTML}{2A78D6}   \definecolor{seBlueBg}{HTML}{E7F0FB}
\definecolor{seOrange}{HTML}{EB6834} \definecolor{seOrangeBg}{HTML}{FDECE4}
\definecolor{seGreen}{HTML}{1BAF7A}  \definecolor{seGreenBg}{HTML}{E3F5EC}
\definecolor{seYellow}{HTML}{EDA100} \definecolor{seYellowBg}{HTML}{FFF4D8}
\definecolor{seViolet}{HTML}{4A3AA7} \definecolor{seVioletBg}{HTML}{ECE9F7}
\definecolor{seGray}{HTML}{6B7280}   \definecolor{seGrayBg}{HTML}{F1F2F4}
\definecolor{seInk}{HTML}{2C3E50}    \definecolor{seInkLight}{HTML}{6B7280}
\definecolor{seArrow}{HTML}{4B5563}
\definecolor{seTeal}{HTML}{2A9D8F}     \definecolor{seRose}{HTML}{E76F51}
\definecolor{seDeepGreen}{HTML}{2E7D32} \definecolor{seSlate}{HTML}{4B5563}

\newcommand{\sefigfont}{\fontsize{7.5}{8.8}\selectfont\sffamily}
\newcommand{\sefiglab}{\fontsize{7}{8}\selectfont\sffamily}

\tikzset{
  secard/.style={draw=seGray!70, fill=seGrayBg, rounded corners=3pt,
                 line width=0.5pt, align=center, inner sep=2pt,
                 font=\sefigfont, text=seInk},
  secard blue/.style={secard, draw=seBlue!75, fill=seBlueBg},
  secard orange/.style={secard, draw=seOrange!85, fill=seOrangeBg, dashed},
  secard green/.style={secard, draw=seGreen!85, fill=seGreenBg},
  secard violet/.style={secard, draw=seViolet!70, fill=seVioletBg},
  searrow/.style={-{Latex[length=2mm,width=1.4mm]}, line width=0.6pt, seArrow},
  searrow back/.style={searrow, dash pattern=on 2pt off 1.2pt},
  selab/.style={font=\sefiglab, text=seInkLight, inner sep=1.5pt},
  sestage/.style={secard, draw=seBlue!80, fill=seBlue!#1},
  sestage/.default=20,
}

\definecolor{leanKw}{HTML}{1F57C3}    
\definecolor{leanTac}{HTML}{7D2AA8}   
\definecolor{leanCmt}{HTML}{2E7D32}   
\definecolor{leanStr}{HTML}{A31515}   
\definecolor{leanSorry}{HTML}{C62828} 
\definecolor{leanType}{HTML}{267F99}  
\definecolor{leanBg}{HTML}{F7F9FC}
\definecolor{leanFrame}{HTML}{D5DCE6}

\lstdefinelanguage{lean4vs}{
  morekeywords=[1]{theorem,lemma,def,abbrev,structure,class,instance,example,
                   noncomputable,open,scoped,import,variable,namespace,end,section,
                   axiom,opaque,notation,universe,deriving,where},
  morekeywords=[2]{by,fun,let,have,show,from,with,at,intro,intros,exact,apply,
                   refine,rw,simp,simpa,obtain,rcases,rintro,constructor,cases,
                   induction,calc,use,subst,change,omega,linarith,nlinarith,norm_num,
                   decide,trivial,rfl,ext,funext,classical,convert,using,exact_mod_cast,
                   push_cast,simp_all,aesop,apply_rules,noncomm_ring,left,right,
                   exfalso,absurd,contradiction,match,do,if,then,else,forall,exists,
                   simp_all_arith,simp_all_arith!,simp_arith,field_simp,ring,ring_nf,
                   positivity,unfold,dsimp,specialize,by_contra,by_cases,push_neg,
                   norm_cast,split,split_ifs,gcongr,tauto,generalize,exact?,apply?},
  morekeywords=[3]{sorry},
  alsoletter={.!?},   
  sensitive=true,
  morecomment=[l]{--},
  morecomment=[s]{/-}{-/},
  morestring=[b]",
}
\lstdefinestyle{leanvs}{
  language=lean4vs,
  basicstyle=\ttfamily\footnotesize, columns=fullflexible, keepspaces=true,
  keywordstyle=\color{leanKw}, keywordstyle=[2]\color{leanTac},
  keywordstyle=[3]\color{leanSorry}\bfseries,
  commentstyle=\itshape\color{leanCmt}, stringstyle=\color{leanStr},
  emph={Type,Prop,Set,Finset}, emphstyle=\color{leanType},
  showstringspaces=false, breaklines=true, breakatwhitespace=true,
  frame=none, xleftmargin=0pt, aboveskip=0pt, belowskip=0pt, upquote=true,
  literate=
  {→}{{$\rightarrow$}}1 {↔}{{$\leftrightarrow$}}1 {∀}{{$\forall$}}1
  {∃}{{$\exists$}}1 {∧}{{$\wedge$}}1 {∨}{{$\vee$}}1 {¬}{{$\lnot$}}1
  {≤}{{$\leq$}}1 {≥}{{$\geq$}}1 {≠}{{$\neq$}}1 {∈}{{$\in$}}1
  {⊂}{{$\subset$}}1 {⊆}{{$\subseteq$}}1 {∩}{{$\cap$}}1
  {∅}{{$\emptyset$}}1 {∑}{{$\sum$}}1 {⋂}{{$\bigcap$}}1
  {∘}{{$\circ$}}1 {α}{{$\alpha$}}1 {β}{{$\beta$}}1 {π}{{$\pi$}}1
  {ℕ}{{\textcolor{leanType}{$\mathbb{N}$}}}1 {ℝ}{{\textcolor{leanType}{$\mathbb{R}$}}}1
  {ℤ}{{\textcolor{leanType}{$\mathbb{Z}$}}}1 {ℚ}{{\textcolor{leanType}{$\mathbb{Q}$}}}1
  {⟨}{{$\langle$}}1 {⟩}{{$\rangle$}}1 {×}{{$\times$}}1
  {₀}{{$_0$}}1 {₁}{{$_1$}}1 {₂}{{$_2$}}1 {—}{{\textemdash}}1,
}
\newtcblisting{leanbox}[1][]{
  listing only, listing options={style=leanvs,#1},
  before upper={\fontencoding{T1}\selectfont},   
  enhanced, breakable,
  colback=leanBg, colframe=leanFrame, boxrule=0.5pt, arc=2.5pt,
  left=4pt, right=4pt, top=3pt, bottom=3pt, before skip=6pt, after skip=6pt,
  overlay unbroken and first={\node[anchor=north east, inner sep=2.5pt,
    font=\scriptsize\sffamily, text=seInkLight] at (frame.north east) {Lean 4};},
}

\newtcblisting{leanboxmini}[1][]{
  listing only, listing options={style=leanvs,#1},
  before upper={\fontencoding{T1}\selectfont},
  enhanced, unbreakable,
  colback=leanBg, colframe=leanFrame, boxrule=0.5pt, arc=2.5pt,
  left=4pt, right=4pt, top=3pt, bottom=3pt, before skip=4pt, after skip=4pt,
  overlay={\node[anchor=north east, inner sep=2.5pt,
    font=\scriptsize\sffamily, text=seInkLight] at (frame.north east) {Lean 4};},
}

\RequirePackage{fancyvrb}
\definecolor{leanCFg}{HTML}{1F1F1F}    \definecolor{leanCTag}{HTML}{6B7280}
\definecolor{leanCKw}{HTML}{0000FF}    \definecolor{leanCDecl}{HTML}{795E26}
\definecolor{leanCCmt}{HTML}{008000}   \definecolor{leanCNum}{HTML}{098658}
\definecolor{leanCStr}{HTML}{A31515}   \definecolor{leanCSorry}{HTML}{CD3131}
\definecolor{leanCBr1}{HTML}{0431FA}   \definecolor{leanCBr2}{HTML}{319331}
\definecolor{leanCBr3}{HTML}{7B3814}
\newcommand{\leancodesize}{\fontsize{7.5}{9}\selectfont}
\newcommand{\leancodesizemini}{\fontsize{7}{8.4}\selectfont}
\newtcolorbox{leancode}[1][]{
  enhanced, breakable, coltext=leanCFg,
  before upper={\fontencoding{T1}\selectfont},
  colback=leanBg, colframe=leanFrame, boxrule=0.5pt, arc=2.5pt,
  left=5pt, right=4pt, top=4pt, bottom=4pt, before skip=6pt, after skip=6pt,
  overlay unbroken and first={\node[anchor=north east, inner sep=2.5pt,
    font=\scriptsize\sffamily, text=leanCTag] at (frame.north east) {Lean 4};},
  #1}

\definecolor{promptBg}{HTML}{F6F6F6}
\definecolor{promptBar}{HTML}{3F3F3F}
\lstdefinestyle{prompttext}{
  language={}, keywordstyle={}, commentstyle={}, stringstyle={}, identifierstyle={},
  basicstyle=\ttfamily\footnotesize, columns=fullflexible, keepspaces=true,
  breaklines=true, breakindent=1.5em, showstringspaces=false, upquote=true,
  frame=none, xleftmargin=0pt, aboveskip=0pt, belowskip=0pt,
  literate={—}{{\textemdash}}1,
}
\newtcblisting{promptbox}[1]{
  listing only, listing options={style=prompttext},
  before upper={\fontencoding{T1}\selectfont},
  enhanced, breakable, title={#1},
  colback=promptBg, colframe=promptBar, colbacktitle=promptBar, coltitle=white,
  fonttitle=\sffamily\small, titlerule=0pt, toptitle=2pt, bottomtitle=2pt,
  boxrule=0.5pt, arc=3pt, left=5pt, right=5pt, top=4pt, bottom=4pt,
  before skip=8pt, after skip=8pt,
}

\lstdefinelanguage{lean}{
  morekeywords={theorem,lemma,def,structure,class,instance,example,
                by,sorry,fun,let,have,show,from,with,at,forall,exists,
                Type,Prop,Sort,match,do,if,then,else,and,or,not,
                variable,namespace,end,open,import,section},
  sensitive=true,
  morecomment=[l]{--},
  morecomment=[s]{/-}{-/},
  morestring=[b]",
}

\title{\texttt{SemantEq}: Faithful Bidirectional Equivalence of Formal Statements}

\author{\begin{tabular}{@{}l@{\hspace{2em}}l@{\hspace{2em}}l@{}}
Xinrui Li$^{1,2}$ & Pawan Sasanka Ammanamanchi$^{1}$ & Siddharth Bhat$^{1,3,4}$ \\
Rianna Jamdar$^{1}$ & Yongxi Lin$^{1,2}$ & Stella Biderman$^{1}$
\end{tabular} \\[6pt]
{\normalfont $^{1}$EleutherAI \quad $^{2}$Carnegie Mellon University \quad $^{3}$University of Cambridge \quad $^{4}$XTX Markets}}

\newcommand{\beq}{\texttt{BEq}}
\newcommand{\beqplus}{\ensuremath{\texttt{BEq}+}}

\def\Program{\texttt{SemantEq}}
\def\Eval{\Program\texttt{-Eval}}
\def\Checker{\Program\texttt{-ProofAudit}}
\newcommand{\Lib}{\mathcal{L}}                 
\newcommand{\steps}{\rightsquigarrow_{\Lib}}          
\newcommand{\semeq}{\simeq_{Sem,\ \Lib}}            

\iclrfinalcopy 
\begin{document}

\maketitle
\lhead{Preprint}

\begin{abstract}
Statement autoformalization translates mathematics from natural language into a
proof assistant such as Lean 4, and models now do it at scale. As provers get
stronger, the burden of correctness shifts from the proof to the statement,
that is, to whether the formal statement faithfully translates its source.
Given a trusted reference formalization, the natural test is logical
equivalence, and it proves too much. $P \leftrightarrow Q$ holds trivially for
any two statements $P$ and $Q$ that the library proves, because a prover can
prove each separately and pass through $P \leftrightarrow \mathrm{True}
\leftrightarrow Q$. The stronger the prover, the more unrelated pairs pass.
Earlier methods restrict the prover and check that each proof names the other
statement. We introduce \texttt{SemantEq}, a more principled relation that instead
requires each statement to turn into the other only through a chain of
allowed steps, such as logical rearrangement and rewriting with the library's
own equations. Our programs judge a proof by the proof term Lean checks, not
by its tactic script. \texttt{SemantEq-Eval}, a search without a model, reaches 100\%
precision and 65.6\% recall on the ProofNetVerif test split under our
adjudicated labels, against 95.8\% and 61.8\% for \texttt{BEq}+, and an 8.9-fold
increase in its CPU time adds at most three pairs. A language model finds more proofs, but under two of four
prompts it proves both directions of about two in five of the non-equivalent
pairs. \texttt{SemantEq-ProofAudit} audits each model-written proof under a policy derived from the relation and
accepts none of the 104 non-equivalent pairs, and it raises recall on 407
equivalent pairs from 65.4\% to 81.1\% under the best prompt. Separating the
generation of equivalence proofs from their acceptance lets a stronger prover
raise recall without changing what an acceptance means. We release code and data at:
https://github.com/hlxr89/SemantEq
\end{abstract}

\section{Introduction}
\label{sec:intro}

Formal theorem proving has improved quickly on benchmarks such as miniF2F,
ProofNet, and PutnamBench \citep{zheng2022minif2f,proofnet,tsoukalas2024putnambench},
in part because the Lean kernel checks every proof and gives models a reliable
signal \citep{demoura2021lean4}. Models now also formalize at scale, from large
problem sets \citep{ying2024leanworkbook,lin2025goedel,wu2026stepfun} to whole
textbooks and
research papers \citep{rammal2026formalizing,zhang2026leanmarathon}, and the
Formal Conjectures project collects Lean statements of open problems
\citep{formalconjectures2025}. The kernel, however, checks a proof only against
its formal statement, not whether the statement says what the source says.
Audits of benchmarks find statements that drop hypotheses, change domains, or
become vacuous \citep{ospanov2025minif2f,ammanamanchi2026faults}, and an expert
review of an agent's formalization found it weakest at choosing definitions,
not at filling in proofs \citep{ilin2026sorries}. As provers get stronger, the
burden of correctness shifts from the proof to the statement.

At this scale no one can read every statement, so an automatic evaluator
decides which formalizations count as correct. It is also often the reward for
training the next formalizer \citep{wu2026stepfun,yu2026mathesis} or the
filter that keeps its outputs \citep{peng2026criticlean}. A formalizer trained
against an evaluator that accepts unfaithful statements therefore learns to
produce them, and one trained against an evaluator that rejects valid
rewordings is held back.

The ideal evaluator would compare a formal statement with its informal source
directly, and learned judges try to
\citep{lu2025formalalign,peng2026criticlean,rammal2026formalizing}, but their
verdicts come with no proof that anyone can check. We work in the setting where
a trusted formal reference exists, as it does when large projects formalize the
statements of their main results before the proofs
\citep{vandoorn2023sphere,tao2023pfr,commelin2024abstraction} and when
benchmarks ship references checked by people. The task is still hard, because
one informal statement has many correct formalizations, which differ in
variable names, hypothesis order, library definitions, or the form of the
conclusion. Text similarity to the reference, such as BLEU
\citep{papineni2002bleu}, therefore cannot tell a rewording from a different
claim \citep{proofnet}. The same problem arises across language model evaluation. \citet{biderman2024lessons} call it ``the Key Problem'': many syntactically different outputs express the same content, and the best available tools for detecting semantic equivalence are language models, which are the systems being evaluated. Responses fall into three groups: restricting the answer space so a program can check it, grading with a model, and grading by people. Formal mathematics seems to belong to the first group, because the kernel checks proofs. The kernel does not check statements, however, so statement autoformalization brings the Key Problem back.

A natural test is to prove that the candidate and the reference imply each
other, but the test proves too much. In Lean, \texttt{Nat.Prime 5} and
\texttt{Irrational (Real.sqrt 2)} are logically equivalent, because Mathlib
proves both. Each implication can then ignore its antecedent, the statement
it assumes, and prove its conclusion outright. A candidate that states
something true but different from the reference passes the same way, and a
stronger prover finds such proofs for more pairs. The existing evaluators
\beq{} and \beqplus{} restrict how the proof is found and check the proof
text for the antecedent's name \citep{liu2025rethinking,poiroux2025reliable}.
A restricted search can still close a goal from library facts alone, however,
and a name in the text does not show that the antecedent did any work
(Section~\ref{sec:algorithm}). The question is therefore which proofs to
accept.

\Program{} answers that question before any prover runs. It relates two
statements when each turns into the other by a chain of four kinds of step,
namely logical rearrangement, rewriting with an equation or biconditional that
the library states, a change of variables, and structural weakening
(Figure~\ref{fig:chain}). Our programs read the proof term that Lean checks,
not the tactic script that produced it, to see which steps a proof takes and
whether it depends on the other statement. Because the rule is fixed apart
from the prover, it separates the generation of equivalence proofs from their
acceptance. Any
prover can search for chains, and a stronger prover can find more of them
without changing what an acceptance means.

We study two ways to find chains. \Eval{} is a deterministic Lean metaprogram
that searches for a chain in each direction and audits the proof term it
finds, with no model and no GPU, in a median of two seconds per pair (mean
30.5~s, on the 407 equivalent pairs of Section~\ref{sec:capability}) on a
laptop CPU. On the ProofNetVerif test split, under our
adjudicated labels, it accepts no non-equivalent pair, for a precision of
100\% against 95.8\% for \beqplus{}, and its recall is 65.6\% against 61.8\%.
The published labels give the same ranking. On 100 non-equivalent pairs
assembled to test each audit, it accepts none, while \beqplus{} accepts 54
(Section~\ref{sec:evaluation}). More search barely helps. Across our grid an
8.9-fold increase in CPU time adds at most three pairs, because many of the
missed pairs need a step that the search is not built to take, such as
building a new object to plug into the other statement
(Section~\ref{sec:capability}).

The second way lets a language model write the proofs. A compiling proof
establishes the implication but not, by itself, that the two statements say
the same thing, since under prompts that ask only for a compiling proof
the model proves both directions of about two in five of the non-equivalent
pairs. \Checker{} reads each proof term and accepts a direction only when it
is built from the four kinds of defined step, apart from exceptions it
records, and it names the rule that a rejected proof breaks. It accepts none
of the 104 non-equivalent pairs and certifies 330 of the 407 equivalent pairs
under the best of four prompts, against 266 for the search
(Sections~\ref{sec:checker} and~\ref{sec:campaign}).

Our contributions are as follows.
\begin{itemize}\itemsep1pt

  \item We define \Program{}, a relation that fixes the allowed steps by their effect, and prove that it is an equivalence relation whose related pairs are logically equivalent (Section~\ref{sec:semanteq}).

  \item We build \Eval{}, a search without a model that improves on \beqplus{} in precision and recall, and find where more search stops helping (Sections~\ref{sec:eval} and~\ref{sec:evaluation}).

  \item We introduce \Checker{}, which audits proofs from any prover and separates equivalent from non-equivalent pairs where compilation does not (Section~\ref{sec:models}).
\end{itemize}

\Program{} checks agreement with a formal reference, not with the informal
source, so a candidate that copies a mistake in the reference still passes. A
rejection means only that no proof passed the search and its audits, not that the two statements
differ.

\section{Better Semantic Equivalence for Mathematical Statements}
\label{sec:beq_background}
\label{sec:algorithm}

A good notion of mathematical equivalence should accept different formulations
of the same claim.
Definitional equality, which Lean checks by computing and unfolding
definitions, is too strict. It identifies $1 + 1 = 2$ with $2 = 2$ but not
$x + y = z$ with $y + x = z$. Unrestricted provable equivalence is too loose,
as the introduction shows. \beq{} and \beqplus{} aim for a test between the
two by restricting the proofs in both directions.

\beq{} \citep{liu2025rethinking} defines equivalence as converting each
statement into the other by restricted transformations, which reduce to
definitional equality when none are allowed. To check one direction, it
assumes the other statement, the antecedent, as a theorem closed with
\texttt{sorry}, Lean's placeholder for a missing proof. It then looks for a
proof of the goal, first with \texttt{exact?}, which searches the library and
the hypotheses for a lemma that closes it, and then with a language model
limited to eleven tactics, Lean's proof commands. \beqplus{}
\citep{poiroux2025reliable} replaces the model with fixed strategies that use
the antecedent through \texttt{apply}, \texttt{convert}, or \texttt{have} and
finish with tactics such as \texttt{simp\_all\_arith!}, so it is deterministic
and runs on a CPU. Both look for the antecedent's name in the proof, although
\beqplus{} does so only for its first \texttt{exact?} attempt and inserts the
name itself in its other strategies.

Appendix~\ref{app:beqplus} shows two ways this fails on pairs that
ProofNetVerif labels not equivalent. First, a closing tactic can prove the
goal from the library alone. In the first pair each statement has its own
direct proof, so accepting both directions says nothing about whether they
state the same claim. The \beqplus{} authors report related false acceptances
and a fix that is now their default \citep[Appendix A.6.1]{poiroux2025reliable}.
Second, a name in a script does not show that the antecedent did any work. In
one direction of the second pair the antecedent never appears in the proof
term, and in the other it appears but is not needed. An evaluator therefore
has to check how a proof uses the antecedent and the library. \Program{} does
this by defining the allowed steps first, as \beq{} does, and by examining the
proof term, the proof object that Lean's kernel checks.

\subsection{\Program{}: An Equivalence Relation Defined by Allowed Steps}
\label{sec:semanteq}

Our key idea is to fix a set of allowed steps, each simple enough for a person
or a program to audit, and to call two statements equivalent when a chain of
such steps leads from each to the other. We first fix notation and then define
four kinds of step by their effect on a statement.

\paragraph{Setting.} Fix the library $\Lib$, for us Lean with Mathlib. A
\emph{statement} is a proposition over $\Lib$ with no free variables, taken up
to renaming of bound variables and definitional unfolding, and $\mathcal{S}$
is the class of statements. The \emph{atoms} of a formula $\varphi$, written
$\mathrm{atoms}(\varphi)$, are its equations and applications of predicates.
We write $S = \mathcal{C}[\varphi]$ to single out one occurrence of a formula
or term $\varphi$ in $S$, and $\mathcal{C}[\varphi']$ for $S$ with $\varphi'$
in its place. The \emph{context} $\Gamma$ of the occurrence holds the
variables quantified over it and the hypotheses it stands under. A term or
formula \emph{over} $\Gamma$ uses only the variables of $\Gamma$ and the
symbols of $\Lib$, and $\Lib, \Gamma \vdash \psi$ means that $\Lib$ proves
$\psi$ from the hypotheses in $\Gamma$. An occurrence is \emph{positive} when
the path from the root to it passes only through $\wedge$, $\vee$, $\to$, and
the bodies of quantifiers, and enters the left side of an implication an even
number of times. The \emph{prefix} of a statement is
its leading string of universal quantifiers and hypotheses, and
$S = \Pi\;\varphi$ cuts $S$ after an initial segment $\Pi$ of its prefix. A
\emph{bridge} is a theorem
$b = \forall \bar y.\; \psi_1 \to \cdots \to \psi_m \to (u \asymp v)$ of
$\Lib$, with $\asymp \in \{=, \leftrightarrow\}$, that is not a law of pure
logic. We take $\mathcal{B}$ to be the set of all bridges, so no bridge is
chosen by hand. In simpler words, $\mathcal{B}$ is the set of all theorems defined in the library whose conclusion is of the form of an equality or a  bi-conditional. An \emph{instance} $\sigma(b)$ is given by an assignment
$\sigma$ of terms to $\bar y$, with sides $\sigma(u)$ and $\sigma(v)$ and
hypotheses $\sigma(\psi_i)$. Every domain $\mathsf{D}$ below is a type, and a term or formula that a step introduces contains, as written, no proof other than a hypothesis of its context. Appendix~\ref{app:relation-notation} gives the
notation in full.

\begin{definition}[Faithful derivation]
\label{def:faithful-step}
A statement $S$ \textbf{\emph{faithfully derives}} a statement $T$ with respect to the library $\Lib$, written
$S \steps T$, when one of the following holds.
\begin{itemize}\itemsep1pt
   \item \textbf{Logical rearrangement.} $S = \mathcal{C}[\varphi]$ and
        $T = \mathcal{C}[\varphi']$, where $\mathrm{atoms}(\varphi) =
        \mathrm{atoms}(\varphi')$, $\varphi \leftrightarrow \varphi'$ holds
        under every interpretation of the symbols in these atoms, and every atom
        matters on both sides, in that for each atom some interpretation makes the
        truth value of $\varphi$, and likewise of $\varphi'$, change with that
        atom alone. Equations are atoms like any other here.
  \item \textbf{Bridge rewriting.} $S = \mathcal{C}[\sigma(u)]$ and
        $T = \mathcal{C}[\sigma(v)]$, or the same with $u$ and $v$ exchanged,
        for some bridge $b \in \mathcal{B}$ and some instance $\sigma(b)$
        whose assignment $\sigma$ consists of terms over $\Gamma$ and whose
        hypotheses satisfy $\Lib, \Gamma \vdash \sigma(\psi_i)$ for each
        $i \le m$.
  \item \textbf{Change of variables.} Either $T = \forall y{:}\mathsf{D}.\; S$
        for a new variable $y$ ranging over a domain $\mathsf{D}$, or
        $S = \Pi\;\forall x{:}\mathsf{D}.\;\varphi$ and
        $T = \Pi\;\varphi[t/x]$ for a term $t$ of $\mathsf{D}$ built from
        the variables of $\Pi$ and the symbols of $\Lib$. If $\mathsf{D}$ is a
        kind of structure on an underlying set, such as topologies on a
        set, then $t$ must be a structure that $S$ already carries on that
        set, a variable of $\Pi$ or a structure occurring in $S$.
  \item \textbf{Structural weakening.} A one-directional move of pure logic,
        at a positive occurrence for the first three and in the prefix for
        the last two. Dropping a conjunct,
        $\mathcal{C}[\varphi \wedge \psi] \steps \mathcal{C}[\varphi]$;
        adding a disjunct,
        $\mathcal{C}[\varphi] \steps \mathcal{C}[\varphi \vee \psi]$ for a
        formula $\psi$ over $\Gamma$; witnessing an existential,
        $\mathcal{C}[\varphi[t/x]] \steps \mathcal{C}[\exists x{:}\mathsf{D}.\;\varphi]$
        for a term $t$ of $\mathsf{D}$ over $\Gamma$; adding a hypothesis,
        $\Pi\;\varphi \steps \Pi\;(\psi \to \varphi)$ for a formula $\psi$
        over the variables of $\Pi$; and removing a hypothesis the others
        imply, $\Pi\;(\psi \to \varphi) \steps \Pi\;\varphi$ when $\psi$
        follows from the hypotheses of $\Pi$ in every interpretation.
\end{itemize}
\end{definition}

Every step preserves implication, and no step proves anything. Rearrangement
only reorders, a bridge only replaces a phrasing by one the library
identifies with it, a change of variables only adds a variable or fills one
in, and a weakening only removes. The steps differ in whether they can be
undone. Rearrangement and bridge rewriting preserve equivalence in context
and can be reversed, but a change of variables or a weakening can give a
special case or a weaker statement. Read one way, $S \steps T$ therefore
says only that $T$ is a rephrasing, a specialization or a weakening of $S$
by content-free means, and the relation below asks for a chain in each
direction (Figure~\ref{fig:chain}). The relation also depends on the
library, through the bridges it contains and the facts that meet their
hypotheses.

\begin{definition}[\Program]
\label{def:semanteq}
Statements $P$ and $Q$ are \textbf{\emph{\Program{}}} over $\Lib$, pronounced
``semantic'', written $P \semeq Q$, if there
are statements $S_1, \dots, S_k \in \mathcal{S}$ with $k \ge 1$, $S_1 = P$,
$S_k = Q$, and $S_i \steps S_{i+1}$ for each $i < k$, and likewise a chain
from $Q$ to $P$. In other words, $P \semeq Q$ exactly when $P \steps^{*} Q$
and $Q \steps^{*} P$, where $\steps^{*}$ is the reflexive transitive closure
of $\steps$.
\end{definition}

\begin{figure}[t]
\centering
\begin{tikzpicture}[x=1mm, y=1mm, font=\sffamily,
  st/.style={draw=seGray!70, fill=seGrayBg, rounded corners=3pt, line width=0.5pt,
             minimum width=8mm, minimum height=5.5mm, inner sep=2pt, font=\small, text=seInk},
  endpt/.style={st, draw=seInk, fill=seInk!12, text=seInk, line width=0.8pt, font=\small\bfseries},
  after rearr/.style={st, draw=seBlue!80, fill=seBlueBg},
  after bridge/.style={st, draw=seGreen!80, fill=seGreenBg},
  after chvar/.style={st, draw=seViolet!75, fill=seVioletBg},
  rev/.style={{Latex[length=1.8mm,width=1.3mm]}-{Latex[length=1.8mm,width=1.3mm]}, line width=0.7pt},
  one/.style={-{Latex[length=1.8mm,width=1.3mm]}, line width=0.7pt},
  back/.style={line width=0.6pt, seArrow},
  lab/.style={font=\sefiglab, inner sep=1pt, above=8pt},
  sub/.style={font=\scriptsize, inner sep=1pt, below=4.5pt}]
\node[endpt]        (P)  at (0,0)   {$P$};
\node[after rearr]  (S1) at (32,0)  {$S_1$};
\node[after bridge] (S2) at (64,0)  {$S_2$};
\node[after chvar]  (S3) at (96,0)  {$S_3$};
\node[endpt]        (Q)  at (128,0) {$Q$};
\draw[rev, seBlue]   (P)  -- node[lab, text=seBlue]{rearrangement}
                            node[sub, text=seBlue!80!black]{$\varphi\leftrightarrow\varphi'$} (S1);
\draw[rev, seGreen]  (S1) -- node[lab, text=seGreen!75!black]{bridge rewriting}
                            node[sub, text=seGreen!65!black]{$\sigma(u)\mapsto\sigma(v)$} (S2);
\draw[one, seViolet] (S2) -- node[lab, text=seViolet]{change of variables}
                            node[sub, text=seViolet!80!black]{$\forall x.\,\varphi\mapsto\varphi[t/x]$} (S3);
\draw[one, seYellow!85!black] (S3) -- node[lab, text=seYellow!70!black]{structural weakening}
                            node[sub, text=seYellow!60!black]{$\varphi\wedge\psi\mapsto\varphi$} (Q);
\node[st] (T1) at (96,-12.5) {$T_1$};
\node[st] (T2) at (32,-12.5) {$T_2$};
\draw[one, back] (Q.south)  to[out=-90,in=0]   (T1.east);
\draw[rev, back] (T1) -- node[font=\scriptsize\sffamily, text=seInkLight, inner sep=1pt, below=2.5pt]
                   {the chain back from $Q$ to $P$} (T2);
\draw[one, back] (T2.west) to[out=180,in=-90] (P.south);
\end{tikzpicture}
\caption{Chains witnessing $P \semeq Q$. Rearrangement and bridge rewriting
can be reversed, but a change of variables or a weakening need not be, so a
separate chain leads back from $Q$ to $P$.}
\label{fig:chain}
\end{figure}
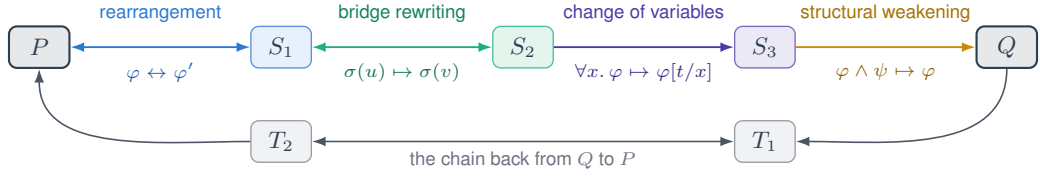

\begin{proposition}
\label{prop:semanteq}
The relation $\semeq$ is an equivalence relation. Moreover, for all
statements $P$, $Q$,
\[
  P \semeq Q \quad\Longrightarrow\quad \Lib\vdash(P\leftrightarrow Q).
\]
\end{proposition}

\begin{proof}
Each permitted step is implication-sound. $S \steps T$ gives
$\Lib \vdash S \to T$. Rearrangement and bridge rewriting replace a
subexpression by one that is equivalent to it in its context, in the second
case because the hypotheses of the bridge instance are provable there. Adding
an unused variable and specializing a universal quantifier are valid
implications, and so is each structural weakening at a positive occurrence or
in the prefix. Composing the steps of a chain gives $\Lib \vdash P \to Q$, and
chains in both directions give $\Lib \vdash P \leftrightarrow Q$. Reflexivity
is the chain of one statement, symmetry is built into
Definition~\ref{def:semanteq}, and transitivity concatenates chains.
Appendix~\ref{app:relation-proof} gives each case in detail.
\end{proof}

\begin{proposition}[Strictly finer]
\label{prop:finer}
Rearrangement, change of variables and weakening alone, with the symbols of
$\Lib$ treated as opaque, relate strictly fewer statements than logical
equivalence in $\Lib$ does. The statements $\forall x{:}\mathbb{N}.\ 0 \le x$
and $(\forall x{:}\mathbb{N}.\ 0 \le x) \wedge (\forall x{:}\mathbb{N}.\ x = x)$
are logically equivalent in $\Lib$, and no chain of these steps derives the
second from the first.
\end{proposition}

\begin{proof}
We use a property of every chain, proved as Lemma~\ref{lem:uses} in
Appendix~\ref{app:relation-proof}: if $S \steps^{*} T$, then $S \to T$
follows by logic alone from the conclusions of the bridge instances used,
each taken under the hypotheses of the context where it was applied, with
equality read as an unknown relation of which only the replacement of equals
by equals is assumed, so no chain uses that $t = t$ holds. A chain without
bridges from the first statement to the second would therefore make the first
imply the second under every interpretation of $\mathbb{N}$, $\le$ and
equality. The interpretation that makes $\mathbb{N}$ a nonempty set, $\le$ true
everywhere and equality the empty relation, with every other sort and symbol
interpreted arbitrarily, makes the first statement true and the second false.
\end{proof}

Proposition~\ref{prop:semanteq} says that \Program{} implies logical
equivalence, and Proposition~\ref{prop:finer} shows that logic alone never
identifies two statements by their truth values, so without bridges even a
provable statement such as $4 = 4$ cannot be added as a new conjunct. The rest
of the relation therefore depends on the bridges the library states, and this
is where our main assumption enters, the only place where a mathematical
judgement enters the definition. We assume that in any well-developed
library the set $\mathcal{B}$, the theorems whose conclusion is an equation or
a biconditional, reflects how the community identifies mathematical objects
and formulas of similar meaning, so that rewriting with a bridge changes the
form of a statement but not what it says. Mathlib does not meet this ideal
everywhere, but we believe $\mathcal{B}$ is a good enough proxy, and \Eval{}
and \Checker{} each make adjustments for its imperfections
(Appendices~\ref{app:eval} and~\ref{app:checker}).

\subsection{\Eval{}: a proof search without a model}
\label{sec:eval}

Given a reference $R$ and a candidate $C$, \Eval{} accepts the pair at once
if the two are identical up to renaming and definitional unfolding. Otherwise
it attempts $R \Rightarrow C$ and $C \Rightarrow R$, and it accepts the pair
only if both directions succeed.

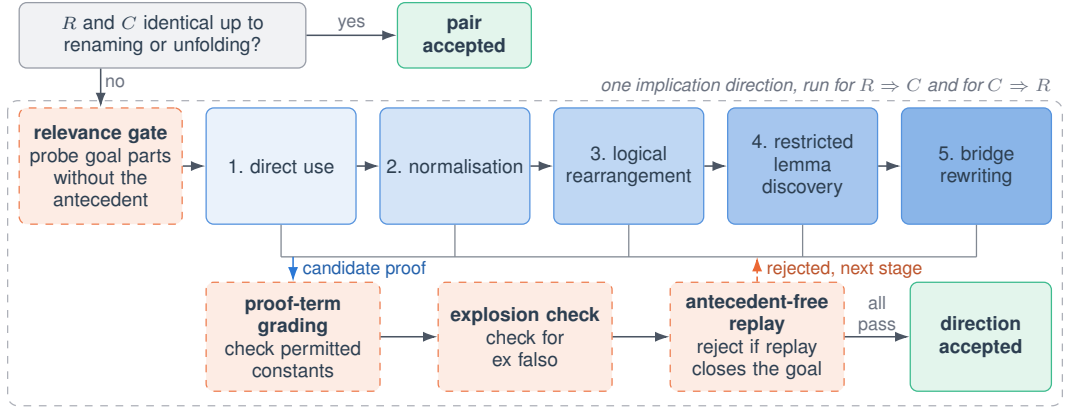
\begin{figure}[t]
\centering
\resizebox{\ifdim\width>\textwidth\textwidth\else\width\fi}{!}{%
\begin{tikzpicture}[x=1mm, y=1mm, font=\sefigfont,
  setup/.style={secard, text width=38mm, minimum height=9mm},
  gate/.style={secard orange, text width=21mm, minimum height=16mm},
  audit/.style={secard orange, text width=22.5mm, minimum height=15mm},
  accept/.style={secard green, minimum height=15mm},
  arr/.style={searrow},
  cand/.style={searrow, seBlue},
  back/.style={searrow back, seOrange},
  lab/.style={selab}]
\node[setup, anchor=west] (ident) at (0,0)
  {$R$ and $C$ identical up to\\ renaming or unfolding?};
\node[accept, anchor=west, text width=17mm, minimum height=9mm] (pacc) at (52,0)
  {\textbf{pair}\\ \textbf{accepted}};
\draw[arr] (ident) -- node[lab, above] {yes} (pacc);
\draw[seGray!55, dashed, rounded corners=4pt, line width=0.5pt] (-1.5,-9) rectangle (143.5,-51);
\node[lab, anchor=south east, font=\sefiglab\itshape] at (142.5,-8.8)
  {one implication direction, run for $R \Rightarrow C$ and for $C \Rightarrow R$};
\node[gate, anchor=west] (gate) at (0,-18)
  {\textbf{relevance gate}\\ probe goal parts\\ without the\\ antecedent};
\node[sestage=10, text width=19.2mm, minimum height=16mm, anchor=west] (s1) at (25.7,-18)  {1.~direct use};
\node[sestage=20, text width=19.2mm, minimum height=16mm, anchor=west] (s2) at (49.6,-18)  {2.~normalisation};
\node[sestage=30, text width=19.2mm, minimum height=16mm, anchor=west] (s3) at (73.5,-18)  {3.~logical\\ rearrangement};
\node[sestage=42, text width=19.2mm, minimum height=16mm, anchor=west] (s4) at (97.4,-18)  {4.~restricted\\ lemma\\ discovery};
\node[sestage=55, text width=19.2mm, minimum height=16mm, anchor=west] (s5) at (121.3,-18) {5.~bridge\\ rewriting};
\draw[arr] (gate.north |- ident.south) -- node[lab, right, pos=0.35] {no} (gate.north);
\draw[arr] (gate) -- (s1);
\draw[arr] (s1) -- (s2); \draw[arr] (s2) -- (s3);
\draw[arr] (s3) -- (s4); \draw[arr] (s4) -- (s5);
\node[audit, anchor=west] (a1) at (25.7,-41.5)
  {\textbf{proof-term}\\ \textbf{grading}\\ check permitted\\ constants};
\node[audit, anchor=west] (a2) at (57.6,-41.5)
  {\textbf{explosion check}\\ check for\\ ex falso};
\node[audit, anchor=west] (a3) at (89.5,-41.5)
  {\textbf{antecedent-free}\\ \textbf{replay}\\ reject if replay\\ closes the goal};
\node[accept, anchor=west, text width=18mm] (dacc) at (122.5,-41.5)
  {\textbf{direction}\\ \textbf{accepted}};
\draw[arr] (a1) -- (a2); \draw[arr] (a2) -- (a3);
\draw[arr] (a3) -- node[lab, above=0.5pt, align=center, inner sep=0.5pt] {all\\ pass} (dacc);
\coordinate (busy) at (0,-30.5);
\draw[line width=0.5pt, seArrow!70] (s1.south |- busy) -- (s5.south |- busy);
\foreach \s in {s1,s2,s3,s4,s5} \draw[line width=0.5pt, seArrow!70] (\s.south) -- (\s.south |- busy);
\draw[cand] (a1.north |- busy) -- node[lab, right, xshift=2pt, text=seBlue!80!black] {candidate proof} (a1.north);
\draw[back] (a3.north) -- node[lab, right, xshift=2pt, text=seOrange!80!black] {rejected, next stage} (a3.north |- busy);
\end{tikzpicture}}
\caption{One implication direction. The relevance gate runs before the
search. Each proof that a stage finds must pass the three audits below it, or
the search moves on to the next stage.}
\label{fig:pipeline}
\end{figure}

Each direction runs the five stages of Figure~\ref{fig:pipeline} in order,
from cheap to expensive, and stops at the first proof that passes the audits.
The first stage applies the antecedent to the goal directly. The second
rewrites both statements into a normal form with a fixed set of logical
identities and cast lemmas and tries again, and the third takes both apart
into their logical components and matches them piece by piece. No library
fact outside the fixed set enters these three stages. Rearrangement alone
settles a large share of real pairs, since formalizers most often differ in
presentation rather than in substance (Appendix~\ref{app:chains}). The last
two stages collect a small pool of library equations and biconditionals whose
sides match parts of the two statements, and rewrite with them in either
direction, one bridge at a time.

\paragraph{Audits.} The stages call Lean's simplifier, rewriter, and backward
search. We choose the lemmas that these procedures may
use, but we do not control how they build the proof term. A proof they find
may hide work in an auxiliary constant, close the goal by a route other than
the intended rewrite, or use the antecedent in a way that contributes
nothing. Four audits guard against such proofs (Figure~\ref{fig:pipeline}),
and passing them is the acceptance criterion of the program, not a
consequence of Proposition~\ref{prop:semanteq}. \textbf{Relevance} runs
before the search. It refuses the direction if the stage procedures prove,
without the antecedent, a part of the goal that the antecedent does not
assert. The other three audits inspect every proof that the search finds.
\textbf{Proof-term grading} expands auxiliary constants and admits only the
antecedent, logical rules, bridges, and a few further classes
(Appendix~\ref{app:eval}). The \textbf{explosion check} fails a proof that
derives \texttt{False} from its context and concludes the goal from it, but it
allows a refutation inside a negation, as in proving $\lnot Q$ by assuming
$Q$. \textbf{Antecedent-free replay} reruns every accepting transformation
with the antecedent withheld, and it discards the acceptance if the goal still
closes. Removing any one audit admits the errors it targets
(\S\ref{sec:ablation}).

\section{Evaluation}
\label{sec:evaluation}

We evaluate \Eval{} as a judge of statement equivalence. Our main benchmark is ProofNetVerif
\citep{poiroux2025reliable}, 3{,}752 pairs of a reference and a candidate
formalization in a validation and a test split, each labelled by people as
faithful or not to the informal statement. We score the 1{,}303 of the
1{,}452 test pairs that elaborate on our toolchain. Our criterion, equivalence
with the reference, is different, and recent audits have found errors in the
statements of formal benchmarks
\citep{ospanov2025minif2f,ammanamanchi2026faults}, so we review the test pairs
labelled equivalent. The review changes 102 labels, all from equivalent to
non-equivalent\footnote{We release the adjudicated relabelling of the test
split with our code.}, and we report both label sets. For 88 of the relabelled
pairs the review was prompted because no configuration of \Eval{} accepts them
and \Checker{} certifies them in at most one direction (\S\ref{sec:models}).
The other 14 came from a language-model screen and manual rulings, and a
person made every final decision (Appendix~\ref{app:data}). We compare with
\beqplus{} ported to our toolchain (Appendix~\ref{app:repro}).

\subsection{Head to head on ProofNetVerif}
\label{sec:pnv}

Table~\ref{tab:pnv} compares the two systems on the same pairs. Under the
adjudicated labels \Eval{} accepts 14 more equivalent pairs than \beqplus{},
for a precision of 100 against 95.8 and a recall of 65.6 against 61.8. The
published labels give the same ranking, and there the \beqplus{} numbers are
close to the 98.0 and 48.3 that its authors report on the whole benchmark
under an older Lean version. \Eval{} makes no false acceptance under either
label set, and \beqplus{} makes 8 and 10 (Appendix~\ref{app:beqplus} shows
two). 
On the validation split, which we do not score, we know of one false
acceptance by \Eval{} (Appendix~\ref{app:eval}).

\begin{table}[t]
\centering

\caption{Equivalence judgements on the ProofNetVerif test split, scored on
the 1{,}303 of 1{,}452 pairs that elaborate on our toolchain, under the
published labels and under our adjudicated labels.}
\label{tab:pnv}
\begin{tabular}{llrrrrrrr}
\toprule
Labels & System & TP & FP & FN & TN & Precision & Recall & F1 \\
\midrule
Published   & Ours                    & 242 & 0  & 229 & 832 & \textbf{100.0} & \textbf{51.4} & \textbf{67.9} \\
            & \beqplus{}$^{\ddagger}$ & 230 & 8  & 241 & 824 & 96.6 & 48.8 & 64.9 \\
\midrule
Adjudicated & Ours                    & 242 & 0  & 127 & 934 & \textbf{100.0} & \textbf{65.6} & \textbf{79.2} \\
            & \beqplus{}$^{\ddagger}$ & 228 & 10 & 141 & 924 & 95.8 & 61.8 & 75.1 \\
\bottomrule
\end{tabular}

\vspace{2pt}
\footnotesize $^{\ddagger}$ Ported to our toolchain with compatibility fixes and no
change of method.
\end{table}

\subsection{Audit ablation}
\label{sec:ablation}

To test whether each audit of \S\ref{sec:eval} does the work it was built
for, and whether the precision gap comes from the design rather than from the
benchmark's particular pairs, we build an \emph{ablation set} of 100
non-equivalent pairs (Appendix~\ref{app:data}). Fifty are ProofNetVerif test
pairs, including all ten that \beqplus{} accepts. Fifty are constructed in
four classes. A real statement is
given one extra claim, the most common defect among the relabelled benchmark
pairs, or it is paired with a neighbouring theorem about the same objects,
with an unrelated true fact, or with a statement whose hypotheses contradict
each other. Every added claim and every partner statement is easy to prove
from the library, so a rejection comes from the restricted steps and the
audits rather than from a lack of proving power. The set does not estimate
how often such pairs occur in practice. We remove each audit in turn, and all
four at once, and run \beqplus{} on the same pairs (Table~\ref{tab:ablation}).

Each audit catches its own family of pairs. Removing relevance admits 19
pairs, 15 of them added claims. Removing proof-term grading admits 8, 6 of
them unrelated true facts. Removing the explosion check admits 3, 2 of them
benchmark pairs, and removing the antecedent-free replay admits 2. Removing
all four admits 35, more than the 32 of the single removals together, so
removing several audits exposes failures that no single removal shows. The search alone still rejects the
other 65, including 20 of the 54 pairs that \beqplus{} accepts, so its
restricted steps and the audits both contribute.

\begin{table}[t]
\centering

\caption{Contribution of each audit of \S\ref{sec:eval} on the ablation set
of 100 non-equivalent pairs, fifty from the ProofNetVerif test split and fifty
constructed. Entries count erroneous acceptances.}
\label{tab:ablation}
\footnotesize
\setlength{\tabcolsep}{3pt}
\resizebox{\textwidth}{!}{%
\begin{tabular}{lrrrrrr}
\toprule
 & \multicolumn{1}{c}{\shortstack{ProofNetVerif\\test split (50)}}
 & \multicolumn{1}{c}{\shortstack{Added\\claim (18)}}
 & \multicolumn{1}{c}{\shortstack{Neighbouring\\theorem (14)}}
 & \multicolumn{1}{c}{\shortstack{Unrelated\\true fact (14)}}
 & \multicolumn{1}{c}{\shortstack{Vacuous\\hypotheses (4)}}
 & Total \\
\midrule
Ours, all checks                     & 0  & 0  & 0 & 0  & 0 & \textbf{0} \\
\quad without relevance              & 0  & 15 & 4 & 0  & 0 & 19 \\
\quad without proof-term grading     & 0  & 0  & 2 & 6  & 0 & 8 \\
\quad without explosion check        & 2  & 0  & 0 & 0  & 1 & 3 \\
\quad without antecedent-free replay & 0  & 0  & 1 & 1  & 0 & 2 \\
\quad without all four               & 2  & 17 & 7 & 7  & 2 & 35 \\
\midrule
\beqplus{} (our toolchain)           & 10 & 18 & 9 & 13 & 4 & 54 \\
\bottomrule
\end{tabular}}
\end{table}

\subsection{Search strength and saturation}
\label{sec:capability}

We next ask how much more search helps on the \emph{equivalent set}. It holds the
407 pairs labelled equivalent after the review, 369 ProofNetVerif test pairs
and 38 of the 40 pairs of a curated recall suite whose references come from
ProofNet\# \citep{proofnet,poiroux2025reliable}, PutnamBench
\citep{tsoukalas2024putnambench}, and
Compfiles \citep{renshaw_compfiles}
(Appendix~\ref{app:data}), so its recall is not comparable with
Table~\ref{tab:pnv}. We raise the depth of chained rewrites, the number of
bridges admitted per step, and the number of intermediate statements kept at
each depth far past their useful range (Table~\ref{tab:capability},
Appendix~\ref{app:grid}). The base configuration accepts 65.4\% of the pairs.
Doubling the bridge pool (D1+W1) adds two pairs and a second chained rewrite (D2) adds one, for 66.1\%. Eight stronger configurations add none, while the mean time per
pair grows from 50.7 to 271.3 seconds. The median stays at 2.0 seconds in
every configuration, because a larger budget only adds time to the pairs that
fail. Figure~\ref{fig:recall-runtime} shows the same plateau against a time
budget per pair, with \beqplus{} the cheaper judge below about 15 seconds per
pair and \Eval{} ahead above it.

\begin{figure}[t]
\centering
\includegraphics[width=\textwidth]{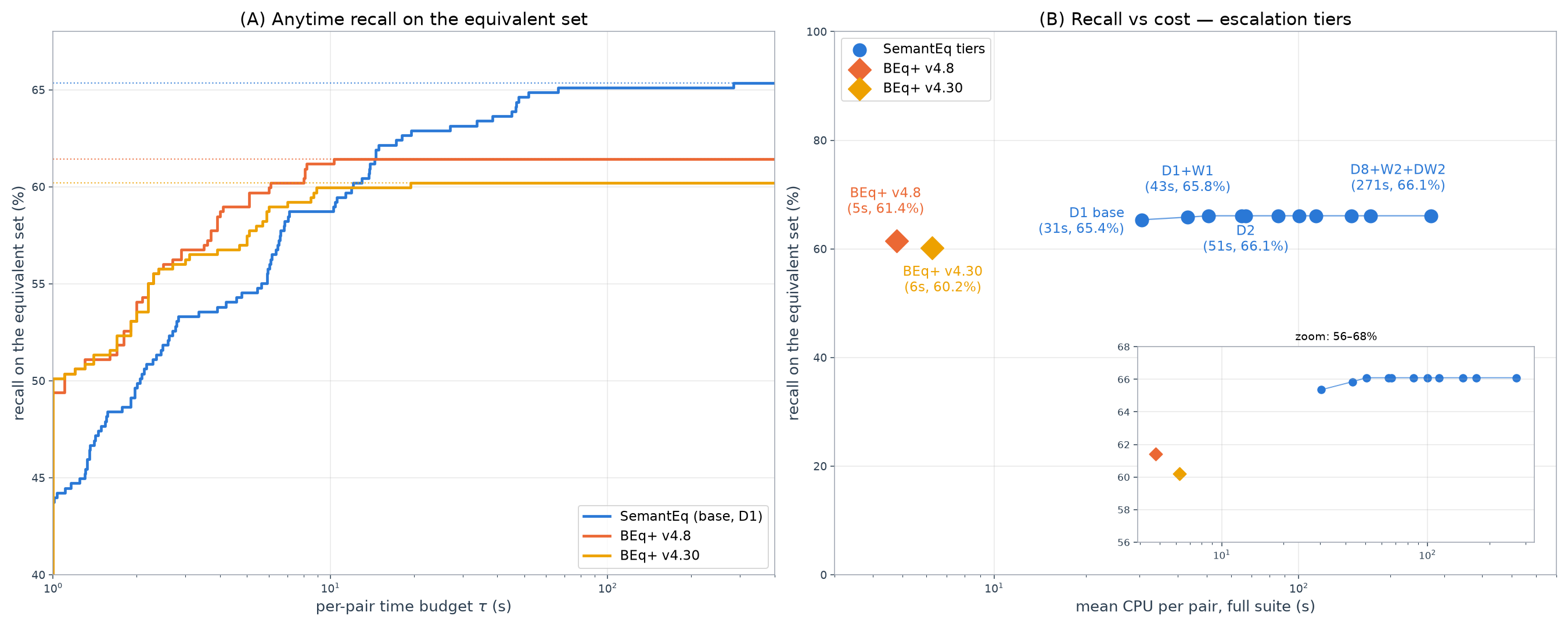}
\caption{(A) Recall attainable when each pair is allowed at most $\tau$
seconds, computed from the uncapped runs on the \emph{equivalent set}. \beqplus{} leads
below roughly 15\,s per pair, and the curves level off at $65.4\%$ for \Eval{}
and at $61.4\%$ and $60.2\%$ for \beqplus{} on Lean v4.8 and on our toolchain. (B) Recall against mean cost per pair for each search configuration
(Appendix~\ref{app:grid}).}
\label{fig:recall-runtime}
\end{figure}

The saturation is a property of this search algorithm. By design the search
never proves anything on its own. It builds no new object, proves no new
bridge, and meets a bridge's hypotheses only with local hypotheses, its
fixed rewrite set and the numeral facts of Appendix~\ref{app:eval}, so a
larger budget only lets it try more of the same steps. The pairs it misses might need an object or a structure that the
goal only implies to be built before the antecedent applies, or a bridge that the library
does not state, or states in one direction only, so the proof has to
establish the identification itself before it can rewrite along it. Both require the
prover to establish something new, which is exactly what the search is built
not to do. The bottleneck is therefore the set of operations, not the search
budget. A prover that can take these steps can close such pairs, but it then
supplies content of its own, and whether that content was used faithfully
can only be judged on the finished proof, so Section~\ref{sec:models}
changes the prover and keeps the audit fixed.

\section{Decoupling Proof Generation from Proof Acceptance}
\label{sec:models}

We test a language model as the prover on the \emph{certification set}, the
full 511 pairs labelled equivalent before the adjudication, namely 471 pairs
of the ProofNetVerif test split and the 40 recall pairs. After the
adjudication the set holds 407 equivalent and 104 non-equivalent pairs. For
each pair and each of four prompt styles, which range from rewriting only to
expedient shortcuts (Appendix~\ref{app:prompts}), GPT 5.6 Sol in Codex writes
a proof of each direction in one session per pair. A proof that compiles
establishes the implication, not that the statements say the same thing. Under the expedient and the unrestricted prompts the
model's proofs compile in both directions for 43 and 45 of the 104
non-equivalent pairs, so a judge that only checks compilation would accept
them. We therefore keep acceptance separate from generation. Every returned
proof is checked by \Checker{}, and a pair is accepted only when both
directions are certified.

\subsection{\Checker{}: auditing supplied proofs}
\label{sec:checker}

\Checker{} is built from the audits of \Eval{}. The search earned its
precision by restricting the steps a proof could take
and by auditing the proof it found. A black-box prover takes the first away, and since its search cannot be steered from inside what remains is the
audits that read the finished proof term. Lean checks that a supplied proof
is valid, and \Checker{} then tests it against an audit policy based on
Definition~\ref{def:faithful-step}. It reads the proof term that Lean builds
from the tactic script and gives each position one of three roles
(Figure~\ref{fig:checker}): the goal derivation, an argument handed to the
antecedent, or the justification of a rewrite. A constant is judged by its
shape and by the role of its position. On the goal derivation only the
antecedent, logical rules, and bridges are admitted, a bridge must convert the
goal rather than state it, and the antecedent must do work in every part of
the proof, such as each branch of a case split. Each
rejection names the check that failed, so \Checker{} can serve as a
standalone judge or give feedback to a prover in a loop
(Appendix~\ref{app:cases}).

\begin{figure}[t]
\centering
\includegraphics[width=\textwidth]{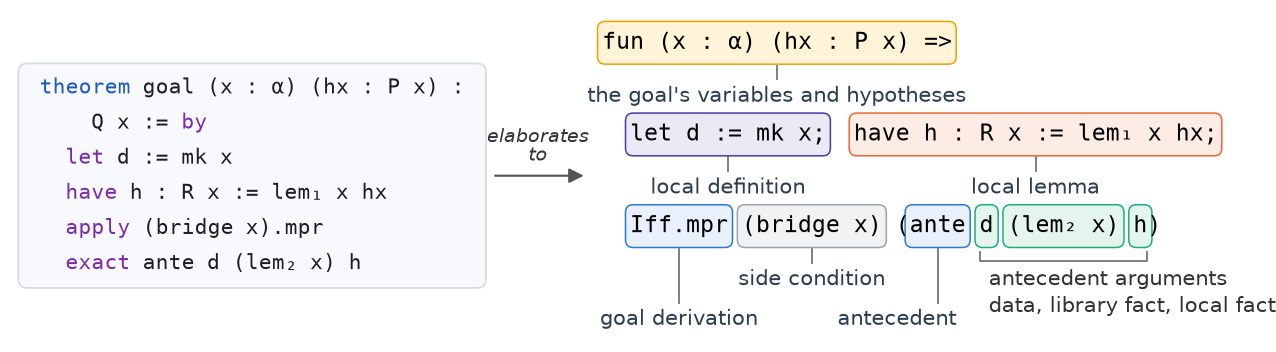}

\caption{The auditor gives each position in a proof term a role. The same
library fact can be admitted as the justification of a rewrite and refused
when it proves the goal directly.}
\label{fig:checker}
\end{figure}
Model proofs also contain material that the search never produces, such as
local lemmas and objects built along the way, and three rules cover it.
\begin{itemize}\itemsep1pt
  \item \textbf{Local lemmas.} A lemma proved from the antecedent is audited
        with its proof expanded. A lemma proved without it is judged by its
        shape like a library fact, and each direction of its proof is then
        checked as an implication proof of its own, so a direction that proves
        its conclusion without using its hypothesis fails. No proof may contain
        an unproved hole.
  \item \textbf{Instantiation of the antecedent.} The arguments of each
        application of the antecedent must come from the types, objects, and
        structures of the goal. A replay then withholds the antecedent but
        keeps the proof's own constructions, and it refuses the proof if the
        goal closes with those constructions but not from the library alone.
  \item \textbf{Library facts at a hypothesis of the antecedent.} Beyond the
        definition, the auditor tolerates a library fact that meets a
        hypothesis of the antecedent and records every such use, because a
        change of domain often leaves a fact that the library states in one
direction only. A stricter policy refuses such uses when they draw on
        no hypothesis of the goal (Appendix~\ref{app:checker}).
\end{itemize}

The auditor has no relevance audit and tolerates library facts in the
justification of a rewrite, because it sees only the final proof term. The
search has no instantiation check, and it builds no data objects, but instance
synthesis once supplied a structure that the goal does not carry
(Appendix~\ref{app:eval}).

\subsection{Results}
\label{sec:campaign}

\Checker{} certifies none of the 104 non-equivalent pairs in both directions
under any style (Table~\ref{tab:campaign}). It certifies many of them in one
direction, as it is expected to have implication proofs that adds no mathematical content when one statement is a special case of the other, and
the certified direction then shows which statement is at least as strong. The
reasons for rejection often show where they differ. A recurring pattern
applies the antecedent at an object the goal never names, or binds it and
closes the goal with a library theorem (Appendix~\ref{app:cases}).

\begin{table}[t]
\centering
\caption{Certificates on the certification set, the 407 equivalent pairs and
the 104 that our review relabelled non-equivalent. A pair is accepted only
when both directions are certified, and precision and recall are over that
acceptance. The reference rows apply the same count to \Eval{} and
\beqplus{}.}
\label{tab:campaign}
\setlength{\tabcolsep}{3.5pt}
\resizebox{\textwidth}{!}{%
\begin{tabular}{lrrrcrrrcrr}
\toprule
 & \multicolumn{3}{c}{Equivalent (407)} & & \multicolumn{3}{c}{Non-equivalent (104)} & & & \\
\cmidrule(lr){2-4}\cmidrule(lr){6-8}
 & both & one & neither & & both & one & neither & & Precision & Recall \\
\midrule
\beqplus{} & 245 & 29 & 133 & & 2 & 26 & 76 & & 99.2  & 60.2 \\
\Eval{} (default)  & 266 & 58 & 83  & & 0 & 34 & 70 & & 100.0 & 65.4 \\
\midrule
\multicolumn{11}{l}{\Checker{}, by prompt style} \\
Expedient    & 309 & 40 & 58 & & 0 & 49 & 55 & & 100.0 & 75.9 \\
Unrestricted & 322 & 53 & 32 & & 0 & 57 & 47 & & 100.0 & 79.1 \\
Faithful     & \textbf{330} & 50 & 27 & & 0 & 67 & 37 & & 100.0 & \textbf{81.1} \\
Rewrite      & 329 & 51 & 27 & & 0 & 70 & 34 & & 100.0 & 80.8 \\
\bottomrule
\end{tabular}}
\end{table}

On the 407 equivalent pairs the faithful style certifies 330, 64 more than the
266 of \Eval{}, and \beqplus{} certifies 245. The model's 330 share 255 pairs with the
search's 266, add 75 that \Eval{} never closes, and miss 11 that it does, so
together the two certify 341 of the 407 pairs (83.8\%).
The expedient style certifies 309 and leaves 58 pairs with no certified
direction, against 27 for the faithful style, so the auditor refuses
shortcuts that the other styles avoid. Under the expedient and unrestricted prompts the model wrote
proofs that compile in both directions for about two in five of them, and the
audit rejected every one. Under the stricter audit policy of
Appendix~\ref{app:checker} the faithful count falls to 323, which keeps 57 of
the 64 extra pairs. Of the 77
equivalent pairs that the faithful style misses, 57 have a compiling proof
that the auditor refused in some direction, and 20 lack an elaborating proof
in some direction. The model never saw a verdict, so these counts measure one
session per pair.

\section{Conclusion}
\label{sec:discussion}

We present \Program{}, which defines an equivalence relation
meant to capture when two formal statements are mathematically equivalent in
meaning. Earlier proof-based evaluators
accept a pair when a chosen proof procedure succeeds, whereas \Program{}
defines the allowed steps by their effect and judges the proof itself. A search
without a model improves on \beqplus{} in precision and recall, and an 8.9-fold
increase in its CPU time adds at most three pairs. A language model checked by \Checker{} certifies more
equivalent pairs, 330 of 407, while accepting no non-equivalent pair.
Separating the generation of equivalence proofs from their acceptance lets more
expressive provers raise recall without changing what an acceptance means.

\paragraph{Limitations and future work} Proposition~\ref{prop:finer} covers
chains without bridges, and a bridge whose hypotheses a library proof meets can
connect any two statements that the library proves equivalent. \Eval{}
restricts the facts that meet bridge hypotheses, and
\Checker{} accepted no non-equivalent pair, so acceptance is evidence of faithful
reformulation rather than a proof of it. The relabelled test set reflects our
own adjudication, and we tested one prover model in one session per pair and
prompt style. The method also needs a trusted reference. Evaluating a formal
statement against its informal source when no reference exists is the natural
next step.

\section*{Acknowledgements}

We gratefully acknowledge support from Renaissance Philanthropy's AI for Math Fund through the MathBench: Towards Evaluating Natural Language Proofs project. The authors thank the AI for Math Fund team for supporting open research infrastructure at the intersection of artificial intelligence and mathematics. The views expressed in this paper are those of the authors and do not necessarily reflect the views of the funders.

\subsection*{AI use statement}

In this work, we used generative AI tools for the following tasks that require
disclosure. They helped implement \Eval{} and \Checker{}, and they helped build
and clean the evaluation data, including a first screen of the ProofNetVerif
test labels by Claude Fable 5 and the assembly of our test sets. In the
experiments, GPT 5.6 Sol in OpenAI Codex wrote the equivalence proofs that
\Checker{} checks, and these proofs are the object of study
(\S\ref{sec:campaign}). We also used Claude Fable 5.1 in Claude Code, and \texttt{gpt-6-astra} in OpenAI Codex to check the proof of
Proposition~\ref{prop:semanteq} and Proposition~\ref{prop:finer}, and to find counterexamples to an earlier
version of Definition~\ref{def:faithful-step}. We have not used generative AI
tools to generate synthetic data sets or to translate, and qualitative data
analysis is not applicable to this work. Additionally, we used these tools to spell-check parts of the paper and find errors, to find related work, and to check and
merge the bibliography and check the consistency of the numbers. We have reviewed all AI-assisted work. Three authors verified and tested the code of both programs for correctness. The authors also checked every label change and every item of our test sets, and every final
label is a human decision. The authors checked the final definition and proof themselves. We take responsibility for the final content of this work, including text, claims, and artifacts produced with the aid of generative AI.

\subsection*{Reproducibility statement}
We release the code of \Eval{} and \Checker{} together with the exact Lean
and Mathlib versions they run on. The repository gives instructions to
reproduce every table in the paper and to use \Program{} as a general tool
on new pairs of statements. Appendix~\ref{app:repro} lists the toolchain,
the two changes we made to run \beqplus{} on it, the settings of the model
campaign, and the hardware we used. Appendix~\ref{app:data} describes the
datasets and the audit of the ProofNetVerif test labels, and
Appendix~\ref{app:prompts} gives the four prompts in full.

\bibliography{iclr2027_conference}
\bibliographystyle{iclr2027_conference}

\newpage

\appendix
\section{Related work}
\label{app:related}

The main text compares \Program{} only with \beq{} and \beqplus{}. Here we
cover the wider literature.

\subsection{Autoformalization systems and benchmarks}

Most benchmarks for formal mathematics give a model a formal statement and
ask for a proof, e.g., miniF2F \citep{zheng2022minif2f}, ProofNet
\citep{proofnet}, PutnamBench \citep{tsoukalas2024putnambench}, FormalMATH
\citep{yu2025formalmath}, and CombiBench \citep{liu2025combibench}. Provers
evaluated on these benchmarks have improved quickly
\citep{polu2020gptf,lample2022htps,xin2024deepseekprover,xin2024deepseekprover15,ren2025deepseekproverv2,lin2025goedel,lin2025goedelv2,wang2025kimina,hubert2025alphaproof,chen2025seedprover,achim2025aristotle}.
When the proof is in Lean, the Lean kernel checks it \citep{demoura2021lean4},
but the kernel does not check that the formal statement says what the
informal problem says. Early work on statement
autoformalization used neural machine translation
\citep{wang2018first}, and large language models made few-shot translation
practical, for statements \citep{wu2022autoformalization} and for proofs
\citep{jiang2023dsp}. Later work trains
dedicated formalizers
\citep{ying2024leanworkbook,gao2025herald,lu2024pda,liu2025atlas,yu2026mathesis,wu2026stepfun}
or improves prompted models with retrieval and prover feedback
\citep{zhang2024consistent}. Some of these systems use a language-model judge
of faithfulness to train or to filter their output
\citep{peng2026criticlean,yu2026mathesis}. The quality of
the judge then limits the quality of the formalizer. \citet{weng2025survey}
and \citet{yang2024frontier} survey the area.

\subsection{Surface metrics and source-conditioned evaluation}

The first automatic metrics compared a predicted statement with a reference
as text. BLEU \citep{papineni2002bleu} penalizes valid rephrasings, and
\citet{proofnet} found that it did not track expert judgments of accuracy.
The rate at which predictions type-check shows only that they are well
formed \citep{poiroux2025reliable}. A second group of methods uses a model to
compare the formal statement with the informal source. FormalAlign trains a
model to score their alignment \citep{lu2025formalalign}, and other
pipelines translate the formal statement back into natural language and
compare the result with the source
\citep{ying2024leanworkbook,gao2025herald}.
\citet{li2024symbolic} combine such a consistency score with symbolic
equivalence among sampled candidates. Beyond Compilation combines compilation, judgments
from several models, and human calibration, and it reports a large gap
between statements that compile and statements judged faithful
\citep{zhang2026beyondcompilation}. FormalRx predicts an alignment verdict
together with an error category, its location, and a correction
\citep{wang2026formalrx}. The judges also need checking, since
\citet{yousefzadeh2025advocate} report a large disagreement between a
back-translation judge and manual review on miniF2F.

Methods in this group can read the informal source, which \Program{} does
not. Their verdicts come from a model or a score, however, and not from a
proof. \Program{} returns a binary judgment of agreement with a formal
reference, and each acceptance is backed by a derivation that Lean checks
under explicit restrictions.

\subsection{Proof-based reference equivalence}

The closest work proves the candidate and the reference equivalent inside
the proof assistant. BEq restricts a prover model to a small set of tactics
and requires the proof of each direction to use the other statement
\citep{liu2025rethinking}. \beqplus{} replaces the model with a fixed list
of symbolic strategies, and the same work introduces ProofNetVerif
\citep{poiroux2025reliable}. \citet{li2024symbolic} and
\citet{murphy2024leaneuclid} also prove equivalences, among sampled
candidates and among geometry statements respectively. In these methods
acceptance depends on what a fixed proof procedure finds and on the text of
the proof. For \beqplus{}, Section~\ref{sec:algorithm} shows that closing
tactics may prove the goal from library facts alone, which its authors also
report, and that a \beqplus{} script can name the antecedent without using
it. \Program{} keeps the bidirectional design of BEq, and like BEq it fixes
a set of allowed transformations, but it defines them by their effect on the
statement and checks them on the proof term. \citet{poiroux2025reliable}
restrict \texttt{exact?} because library search can make ``unrelated true
theorems'' count as equivalent, which is the problem our introduction starts
from, and they report their restricted check, BEq$_\mathrm{L}$, at a precision
of 100.0 and a recall of 30.9 on ProofNetVerif. A false acceptance found by
Simon Sorg led to a guard that is now their default (their Appendix A.6.1),
and our ablation set includes that pair. \Program{} builds on this approach by
separating the generation of a proof from an audit of the elaborated proof term,
which examines the role of each library fact, the use of the antecedent in
each part of the proof, and the arguments at which it is instantiated.
\Program{}-\texttt{Eval}
searches for chains of permitted steps directly.
\Program{}-\texttt{ProofAudit} audits a proof written by any prover under
rules built on the same kinds of step (Appendix~\ref{app:checker}), and it
records every use of its tolerances. A stronger prover can therefore be used
without accepting every proof that compiles.

\subsection{Structural transformation metrics}

GTED also compares two statements through transformations
\citep{liu2025gted}. It standardizes both Lean statements, builds their
operator trees, and computes a generalized tree edit distance in which some
transformations, such as renaming a bound variable, cost less than others.
GTED and \Program{} share the aim of accepting structural rephrasings, but
they differ in what they return. GTED returns a similarity score and
produces no proof. \Program{} returns a binary judgment, and it accepts a
pair only when it has a chain of permitted steps in both directions that
Lean checks.

\subsection{Shadow statements}

ShadowBench evaluates a candidate against auxiliary statements, called
shadows, that a language model drafts and experts review for each problem
\citep{han2026shadowbench}. Its
SA-Pass metric gives full credit only when the candidate compiles, implies
each shadow, and is implied by the conjunction of the shadows. The method
does not depend on one reference formalization, but each new problem needs
new shadows. \Program{} instead applies one set of rules to every problem
and compares the candidate with a single reference. The two methods could
be combined, e.g., by using \Program{} to compare a candidate with a shadow.

\subsection{Quality of formal references}

Any reference-based metric assumes that the reference is a faithful
formalization, and recent audits show that the assumption often fails.
\citet{ospanov2025minif2f} found many discrepancies between the formal and
informal statements of miniF2F and released a corrected version.
\citet{ammanamanchi2026faults} document missing hypotheses, altered domains,
and vacuous statements across several Lean benchmarks, and the Formal
Conjectures project relies on human review to catch misformalized statements
\citep{formalconjectures2025}. A pair that our programs accept therefore
agrees with its reference, and the acceptance does not show that either
statement matches the informal source.

\subsection{Rewriting, equality saturation, and hammers}

\Program{}-\texttt{Eval} uses the same parts as ordinary Lean automation,
namely the simplifier, rewriting with library equations, and library search
with \texttt{exact?} \citep{demoura2021lean4,mathlib2020}. Rewriting modulo a
set of equations is a classical topic \citep{baader1998term}, and congruence
closure has been extended to the type theory of Lean
\citep{selsam2016congruence}. Equality
saturation, as implemented in egg, stores many equal expressions in one
shared data structure and applies rewrite rules in no fixed order
\citep{willsey2021egg}. It could make the search faster, but it does not
settle which library equations should count as bridges. Stronger automation
is also available, e.g., Aesop \citep{limperg2023aesop}, Duper
\citep{clune2024duper}, Lean-auto and LeanHammer
\citep{qian2025leanauto,zhu2025leanhammer}, and Sledgehammer in Isabelle
\citep{blanchette2016hammering}. The steps of such tools cannot easily be
restricted one by one, so \Program{} does not constrain their search.
\Program{}-\texttt{ProofAudit} checks the proofs they produce instead.

\section{The four prompt styles}
\label{app:prompts}

Every proof certified by \Checker{} was requested with one of four prompts. The prompts share
a task description, reproduced once below, and differ only in their final
``Proof style'' section, reproduced for each style. Nothing else varied
between styles.

\begin{promptbox}{Common task description}
# Task

You are contributing to a corpus of Lean 4 equivalence arguments between
alternative formalizations of the same mathematical problem.

You are given two Lean 4 files, `ref.lean` and `cand.lean`. Each contains a
preamble (imports, `open` lines, possibly local definitions or structures)
and one theorem statement whose proof is `sorry`. The two files formalize
the same informal mathematical statement, possibly in different ways.

Your job: show that the two statements are semantically equivalent by
writing TWO implication proofs, one per direction, as two self-contained
Lean files:

- `ref_imp_cand.lean` — the ref statement is assumed, the cand statement is
  proved from it;
- `cand_imp_ref.lean` — the cand statement is assumed, the ref statement is
  proved from it.

## File layout (follow exactly)

Each output file has this shape:

```
import Mathlib
<the union of the two input files' `open` / `open scoped` lines>

-- ===== unified context (both sides' local defs, merged) =====
<local definitions/structures from BOTH preambles;
 if neither side has any, keep this header comment and note that it is empty>

-- ===== assumed statement (antecedent) =====
theorem <original_name>_assumed <binders> : <statement> := by
  sorry

-- ===== goal statement, PROVED from the antecedent =====
theorem <goal_name> <binders> : <statement> := by
  <your proof>
```

Rules:

- Copy both theorem statements VERBATIM — binders, implicit and instance
  arguments, unicode notation, hypothesis names. Do not strengthen, weaken,
  reorder, or otherwise "clean up" either statement.
- Rename only the theorem identifiers: append `_assumed` to the assumed
  side's name; the proved side keeps its original name (append `_goal` if
  the two names collide).
- Unified context: include each side's local definitions/structures under
  their original names. If both sides define the same name with identical
  bodies, include it once. If the same name is defined differently on the
  two sides, rename one of them apart and adjust that side's statement
  references consistently — nothing else about the statement may change.
- The assumed theorem's proof must be exactly `sorry`, and it must be the
  only `sorry` in the file. The goal theorem's proof may use the assumed
  theorem by name and must compile.
- Each file must compile standalone against current Mathlib with
  `import Mathlib` (Lean 4 syntax).
- If, after a genuine attempt, you cannot prove a direction, leave that
  goal's proof as `sorry` with a one-line comment `-- not proved: <reason>`.

## Deliverables

Write the two files into the pair's folder, next to `ref.lean` and
`cand.lean`, named exactly `ref_imp_cand.lean` and `cand_imp_ref.lean`.
Create or overwrite them; do not modify `ref.lean` or `cand.lean`.

If a Lean toolchain with Mathlib is available in your workspace, check that
each file actually compiles (e.g. `lake env lean <file>`) and repair it
until it does; otherwise take extra care that the syntax is valid current
Lean 4 / Mathlib.
\end{promptbox}

\begin{promptbox}{Prompt style, expedient}
# Proof style

Treat each implication as an exercise in efficiency: the shortest
compiling proof wins, and how it gets there is unimportant.

- The ambient library is fair game. When the goal — or any subgoal — is a
  known result or within reach of automation, prove it directly; there is
  no requirement to route the argument through the assumed statement, and
  the assumed statement may go unused wherever it is not needed.
- Prefer powerful closing tactics, decision procedures, and computation
  over step-by-step transformation.
- Any logically valid route is acceptable, including indirect ones: argue
  by contradiction or by eliminating cases whenever that is quicker.
- Optimize for brevity, not for making the proof reflect why the two
  statements correspond.
\end{promptbox}

\begin{promptbox}{Prompt style, unrestricted}
# Proof style

No constraints beyond correctness. Use whatever proof technique gets
each implication to compile.
\end{promptbox}

\begin{promptbox}{Prompt style, faithful}
# Proof style

Write proofs that genuinely TRANSPORT the assumed statement into the
goal — each proof should be evidence that the two statements say the same
thing, not merely that the goal is true. Concretely:

- Derive every part of the goal from the assumed statement. The assumed
  theorem must carry the mathematical work everywhere; if deleting the
  `_assumed` theorem would leave your proof easily repairable, the proof
  misses the point.
- Use outside library lemmas only as re-expression bridges: equation- or
  iff-shaped facts (and definitional unfolding) that convert between the
  two statements' spellings of the same hypothesis or conclusion. Do not
  use library results that supply the mathematical content itself, and do
  not prove the goal — or any conjunct of it — from the library or from
  scratch, even where that would be easy.
- Feed the assumed statement from the goal's own hypotheses (bridged as
  needed), not from facts obtained elsewhere.
- Argue directly: no proof by contradiction that manufactures `False`, no
  discharging the goal from an impossible case, and no case split whose
  branches bypass the assumption.
- Keep the file minimal: auxiliary `have` steps only for bridges that are
  actually used.
\end{promptbox}

\begin{promptbox}{Prompt style, rewrite}
# Proof style

Write proofs that genuinely TRANSPORT the assumed statement into the
goal — each proof should be evidence that the two statements say the same
thing, not merely that the goal is true. Concretely:

- Derive every part of the goal from the assumed statement. The assumed
  theorem must carry the mathematical work everywhere; if deleting the
  `_assumed` theorem would leave your proof easily repairable, the proof
  misses the point.
- Use outside library lemmas only as re-expression bridges: equation- or
  iff-shaped facts (and definitional unfolding) that convert between the
  two statements' spellings of the same hypothesis or conclusion. Do not
  use library results that supply the mathematical content itself, and do
  not prove the goal — or any conjunct of it — from the library or from
  scratch, even where that would be easy.
- Feed the assumed statement from the goal's own hypotheses (bridged as
  needed), not from facts obtained elsewhere.
- Argue directly: no proof by contradiction that manufactures `False`, no
  discharging the goal from an impossible case, and no case split whose
  branches bypass the assumption.
- Keep the file minimal: auxiliary `have` steps only for bridges that are
  actually used.

In addition, where possible, shape each proof as a pure rewrite transport:

- `intro` the goal's binders, then align the goal with the assumed
  statement through a chain of `rw [...]` / `simp only [...]` steps using
  equation/iff lemmas — rewriting the goal, or a `have` of the instantiated
  assumption — and finish with `exact` / `apply` of the assumed statement.
- Structural glue is fine: `intro`, `exact`, `apply`, `refine`,
  `constructor`, `obtain`/`rcases` on the assumption, `And.intro`,
  `Iff.intro`/`.mp`/`.mpr`.
- Avoid everything else: `calc`, computation and decision tactics
  (`norm_num`, `omega`, `decide`, `nlinarith`, `field_simp`, `ring`),
  `aesop`, bare `simp` with the default simp set, induction, and
  constructing witnesses or data beyond what the assumption provides.
- If a direction genuinely cannot be closed in this shape, fall back to
  the transport style above for the smallest possible step and mark that
  line with `-- non-rewrite step`.
\end{promptbox}

\section{Notation and proof for the relation}
\label{app:relation}

\subsection{Notation}
\label{app:relation-notation}

Fix the library $\Lib$, here Lean with Mathlib.
The class $\mathcal{S}$ consists of propositions over $\Lib$ with no free
variables, identified up to renaming of bound variables and definitional
unfolding.
The atoms of a formula $\varphi$, written $\mathrm{atoms}(\varphi)$,
are its atomic subformulas, namely equations and predicate applications.

\paragraph{Occurrences and contexts.}
Write $S=\mathcal{C}[\varphi]$ when one occurrence of a formula or term
$\varphi$ in $S$ is selected.
The expression $\mathcal{C}[\,\cdot\,]$ is the rest of the statement with
a hole at that occurrence, and $\mathcal{C}[\varphi']$ replaces it by
$\varphi'$.
The context $\Gamma$ contains the variables quantified over the occurrence
and the hypotheses under which it stands.
It contains nothing quantified or assumed inside the occurrence.
A term or formula over $\Gamma$ uses only its variables and the symbols
of $\Lib$.
The judgement $\Lib,\Gamma\vdash\psi$ means that $\psi$ is provable from
the hypotheses of $\Gamma$ and the theorems of $\Lib$.
Substitutions are well typed and avoid variable capture. The domain of a change of variables or of a witnessed existential is a type rather than a proposition, and a term or formula that a step introduces contains, as written, no proof other than a hypothesis of its context.

\paragraph{Positive occurrences.}
An occurrence is positive if the path from the root passes only through
$\wedge$, $\vee$, $\to$, and quantifier bodies and enters the left side
of an implication an even number of times.
The path never enters a biconditional, an equation, or a function or
predicate argument.
Negation is read as $\varphi\to\mathit{False}$ when determining polarity.
The restriction applies to the first three structural weakening rules.
It does not restrict replacement by an equal or equivalent expression
under logical rearrangement or bridge rewriting.

\paragraph{Bridges.}
A bridge is a theorem $b\in\mathcal{B}\subseteq\Lib$ of the form
\[
  b=\forall\bar y.\;\psi_1\to\cdots\to\psi_m\to(u\asymp v),
  \qquad \asymp\in\{=,\leftrightarrow\},
\]
whose statement is not a schema of pure logic.
For the abstract relation, $\mathcal{B}$ contains all Mathlib theorems
whose conclusion after their hypotheses is an equation or a biconditional,
apart from pure logical schemas.
Logical identities such as commutativity of conjunction count as
rearrangements.
No bridge is chosen by hand.
An instance $\sigma(b)$ assigns terms to $\bar y$.
Its sides are $\sigma(u)$ and $\sigma(v)$, and its hypotheses are
$\sigma(\psi_1),\ldots,\sigma(\psi_m)$.
For a rewrite at context $\Gamma$, the assigned terms must be over
$\Gamma$, and every hypothesis must be provable in $\Lib,\Gamma$.
The available bridges and the facts used to discharge their hypotheses
therefore depend on the library.

\paragraph{Prefixes.}
The prefix of a statement is its leading sequence of universal
quantifiers and hypotheses.
Write $S=\Pi\;\varphi$ when $\Pi$ is an initial segment of that sequence
and $\varphi$ is the remainder.
For example, a prefix of the form
$\forall x_1.\;\psi_1\to\forall x_2.\;\cdots$
can be cut after any quantifier or hypothesis.
A term over the variables of $\Pi$ may use those variables and library
symbols.
In the existential weakening rule, the witness $t$ must have the
quantified type $\mathsf{D}$.

\paragraph{The first-order reading.} Proposition~\ref{prop:finer} and
Lemma~\ref{lem:uses} concern the steps read in a separate many-sorted
first-order syntax. Sorts may be empty, the symbols of $\Lib$ are opaque, and
statements are taken as written, without Lean's definitional identifications
or any computation that depends on proofs. An interpretation gives each sort a
set, each function symbol a function and each predicate symbol a relation, and a formula with free variables is
evaluated under an assignment of values to them, so an atom has a truth value
at each value of its variables. Equality is a relation like any other, with
only the replacement of equals by equals. These results describe this
first-order core of the steps, not every chain that Lean's full language
admits.

\subsection{Proofs of Proposition~\ref{prop:semanteq} and Lemma~\ref{lem:uses}}
\label{app:relation-proof}

\begin{proof}
Every permitted step preserves implication.
For logical rearrangement, the selected formulas are equivalent under every
interpretation, so in particular in the library's own logic.
For bridge rewriting, the assigned terms are available in the local
context $\Gamma$, and the bridge hypotheses are provable there.
The instantiated bridge therefore gives equality or equivalence in
$\Gamma$.
Replacing the selected expression preserves equivalence of the whole
statement in both cases.

For a change of variables, $S$ implies $\forall y{:}\mathsf{D}.\;S$
because $y$ is fresh.
Universal elimination also gives
$(\forall x{:}\mathsf{D}.\;\varphi)\to\varphi[t/x]$ for a term
$t{:}\mathsf{D}$ available under the prefix.
The restriction on structure arguments narrows the permitted
instantiations without affecting soundness.

For structural weakening, dropping a conjunct, adding a disjunct, and
introducing an existential with a witness $t{:}\mathsf{D}$ each give a
valid local implication.
Conjunction, disjunction, and quantifier bodies preserve implication.
The right side of an implication preserves its direction, and the left
side reverses it.
An even number of entries into left sides therefore preserves the
direction from the selected occurrence to the root.
The restriction on the path excludes positions where this argument
does not apply.
In the prefix, adding a hypothesis preserves implication.
Removing a hypothesis preserves implication when the remaining prefix
hypotheses imply it by logic alone.

Hence $S\steps T$ implies $\Lib\vdash(S\to T)$.
Composing the implications along a finite chain from $P$ to $Q$ gives
$\Lib\vdash(P\to Q)$.
Chains in both directions give $\Lib\vdash(P\leftrightarrow Q)$.

A chain consisting of one statement proves reflexivity.
The requirement of chains in both directions gives symmetry.
For transitivity, concatenate a chain from $P$ to $Q$ with one from
$Q$ to $R$, and concatenate the reverse chains from $R$ to $Q$ and
from $Q$ to $P$.
The resulting chains establish $P\semeq R$.
\end{proof}

The proof of Proposition~\ref{prop:finer} rests on the following property of
chains.

\begin{lemma}[What a chain uses]
\label{lem:uses}
In the first-order reading of Appendix~\ref{app:relation-notation}, if
$S \steps^{*} T$, then $S \to T$ follows by the rules of logic alone from
the conclusions of the bridge instances used along the chain, each taken
under the hypotheses of the context where it was applied, where throughout
equality is read as an unknown relation of which only the replacement of
equals by equals is assumed, so that every equation occurring as a formula is
an atom of unknown truth. Hence if $P \semeq Q$, then $P \leftrightarrow Q$
follows in the same way from the instances used along the two chains.
\end{lemma}

\begin{proof}
By induction on the chain, one case for each kind of step. A rearrangement
holds under every interpretation, so in particular under this reading. Each
weakening move is an implication of logic, and removing a hypothesis is
allowed only when it follows from the others under every interpretation. A
change of variables is universal generalisation or instantiation. Bridge
rewriting replaces $\sigma(u)$ by $\sigma(v)$ inside a formula. The
hypotheses of the instance are discharged by the library, which is why the
lemma assumes the instance's conclusion under the hypotheses of its context
rather than the bridge itself; from that assumption the replacement follows
by logic, for a biconditional by replacement of equivalents and for an
equation by replacement of equals inside an atom, whatever the atom's head
symbol denotes, and under a binder by the same replacement inside it, since
the conclusion is assumed for every value of the variables it binds. Composing
the steps gives $S \to T$, and for $P \semeq Q$ the chain
back gives $Q \to P$.
\end{proof}

\section{Worked chains for the definition}
\label{app:chains}

\paragraph{Admitted rearrangements.} The logical rearrangement action of
Definition~\ref{def:faithful-step} admits every equivalence that holds under
every interpretation, preserves atoms and leaves no atom idle. This covers
commutativity, associativity and distribution of the connectives, De Morgan's
laws, implication as disjunction, double negation, currying, the commutation
of like quantifiers, the distribution of $\forall$ over $\wedge$ and of
$\exists$ over $\vee$, prenex moves valid over possibly empty domains, and the unfolding of a biconditional
into two implications, each admitted only when an instance satisfies both conditions on atoms. It excludes whatever introduces or removes an atom,
such as a true conjunct with a new atom or a case distinction, any equivalence that relies on
a part being always true or always false, the laws of equality, since
equality is an atom like any other here and its symmetry or the substitution
of equals are steps only when the library states them as bridges, and the few
quantifier laws that hold only over nonempty domains, such as
$(\forall x.\,\varphi) \leftrightarrow \varphi$ for $x$ not in $\varphi$.

\paragraph{A pair settled by rearrangement alone.} Formalisers differ most
often in presentation rather than in substance, and a recurring pattern is
the same components re-ordered:

\begin{leancode}
\begin{Verbatim}[commandchars=\\\{\},fontfamily=zi4,fontsize=\leancodesize,formatcom=\color{leanCFg},baselinestretch=1.0,xleftmargin=0pt]
\textcolor{leanCCmt}{-- reference (Munkres 29.10)}
\textcolor{leanCKw}{theorem} \textcolor{leanCDecl}{exercise_29_10} \textcolor{leanCBr1}{\{}X : \textcolor{leanCKw}{Type}*\textcolor{leanCBr1}{\}}
  \textcolor{leanCBr1}{[}TopologicalSpace X\textcolor{leanCBr1}{]} \textcolor{leanCBr1}{[}T2Space X\textcolor{leanCBr1}{]} \textcolor{leanCBr1}{(}x : X\textcolor{leanCBr1}{)}
  \textcolor{leanCBr1}{(}hx : \ensuremath{\exists} U : Set X, x \ensuremath{\in} U \ensuremath{\wedge} IsOpen U \ensuremath{\wedge} \textcolor{leanCBr2}{(}\ensuremath{\exists} K : Set X, U \ensuremath{\subset} K \ensuremath{\wedge} IsCompact K\textcolor{leanCBr2}{)}\textcolor{leanCBr1}{)}
  \textcolor{leanCBr1}{(}U : Set X\textcolor{leanCBr1}{)} \textcolor{leanCBr1}{(}hU : IsOpen U\textcolor{leanCBr1}{)} \textcolor{leanCBr1}{(}hxU : x \ensuremath{\in} U\textcolor{leanCBr1}{)} :
  \ensuremath{\exists} \textcolor{leanCBr1}{(}V : Set X\textcolor{leanCBr1}{)}, IsOpen V \ensuremath{\wedge} x \ensuremath{\in} V \ensuremath{\wedge} IsCompact \textcolor{leanCBr1}{(}closure V\textcolor{leanCBr1}{)} \ensuremath{\wedge} closure V \ensuremath{\subseteq} U
\textcolor{leanCCmt}{-- candidate: same context, conclusion conjuncts permuted}
\textcolor{leanCKw}{theorem} \textcolor{leanCDecl}{exercise_29_10_alt} \textcolor{leanCBr1}{\{}X : \textcolor{leanCKw}{Type}*\textcolor{leanCBr1}{\}}
  \textcolor{leanCBr1}{[}TopologicalSpace X\textcolor{leanCBr1}{]} \textcolor{leanCBr1}{[}T2Space X\textcolor{leanCBr1}{]} \textcolor{leanCBr1}{(}x : X\textcolor{leanCBr1}{)}
  \textcolor{leanCBr1}{(}hx : \ensuremath{\exists} U : Set X, x \ensuremath{\in} U \ensuremath{\wedge} IsOpen U \ensuremath{\wedge} \textcolor{leanCBr2}{(}\ensuremath{\exists} K : Set X, U \ensuremath{\subset} K \ensuremath{\wedge} IsCompact K\textcolor{leanCBr2}{)}\textcolor{leanCBr1}{)}
  \textcolor{leanCBr1}{(}U : Set X\textcolor{leanCBr1}{)} \textcolor{leanCBr1}{(}hU : IsOpen U\textcolor{leanCBr1}{)} \textcolor{leanCBr1}{(}hxU : x \ensuremath{\in} U\textcolor{leanCBr1}{)} :
  \ensuremath{\exists} \textcolor{leanCBr1}{(}V : Set X\textcolor{leanCBr1}{)}, x \ensuremath{\in} V \ensuremath{\wedge} IsOpen V \ensuremath{\wedge} closure V \ensuremath{\subseteq} U \ensuremath{\wedge} IsCompact \textcolor{leanCBr1}{(}closure V\textcolor{leanCBr1}{)}
\end{Verbatim}
\end{leancode}

\noindent
No library lemma is needed, because the two statements are related by a
propositional rearrangement alone.
\paragraph{A pair settled by two bridges.} The following pair differs only
in where the division by $90$ sits:

\begin{leancode}
\begin{Verbatim}[commandchars=\\\{\},fontfamily=zi4,fontsize=\leancodesize,formatcom=\color{leanCFg},baselinestretch=1.0,xleftmargin=0pt]
\textcolor{leanCCmt}{-- reference (USAMO 1996 P1)}
\textcolor{leanCKw}{theorem} \textcolor{leanCDecl}{Usa1996P1_refs} :
    \textcolor{leanCBr1}{(}\textcolor{leanCNum}{1} / \textcolor{leanCBr2}{(}\textcolor{leanCNum}{90}:\ensuremath{\mathbb{R}}\textcolor{leanCBr2}{)}\textcolor{leanCBr1}{)} * \ensuremath{\sum} n \ensuremath{\in} Finset.range \textcolor{leanCNum}{90}, \textcolor{leanCBr1}{(}\textcolor{leanCNum}{2} * \textcolor{leanCBr2}{(}n+\textcolor{leanCNum}{1}\textcolor{leanCBr2}{)}\textcolor{leanCBr1}{)} * Real.sin \textcolor{leanCBr1}{(}\textcolor{leanCBr2}{(}\textcolor{leanCNum}{2} * \textcolor{leanCBr3}{(}n+\textcolor{leanCNum}{1}\textcolor{leanCBr3}{)}\textcolor{leanCBr2}{)} * Real.pi
        / \textcolor{leanCNum}{180}\textcolor{leanCBr1}{)}
    = \textcolor{leanCNum}{1} / Real.tan \textcolor{leanCBr1}{(}Real.pi / \textcolor{leanCNum}{180}\textcolor{leanCBr1}{)}
\textcolor{leanCCmt}{-- candidate: the average restated termwise (a variant of the autoformalized statement)}
\textcolor{leanCKw}{theorem} \textcolor{leanCDecl}{Usa1996P1_alt} :
    \ensuremath{\sum} n \ensuremath{\in} Finset.range \textcolor{leanCNum}{90},
      \textcolor{leanCBr1}{(}\textcolor{leanCNum}{2} * \textcolor{leanCBr2}{(}n+\textcolor{leanCNum}{1}\textcolor{leanCBr2}{)}\textcolor{leanCBr1}{)} * Real.sin \textcolor{leanCBr1}{(}\textcolor{leanCBr2}{(}\textcolor{leanCNum}{2} * \textcolor{leanCBr3}{(}n+\textcolor{leanCNum}{1}\textcolor{leanCBr3}{)}\textcolor{leanCBr2}{)} * Real.pi / \textcolor{leanCNum}{180}\textcolor{leanCBr1}{)} / \textcolor{leanCNum}{90}
    = \textcolor{leanCNum}{1} / Real.tan \textcolor{leanCBr1}{(}Real.pi / \textcolor{leanCNum}{180}\textcolor{leanCBr1}{)}
\end{Verbatim}
\end{leancode}
\noindent
Matching the two statements against the library's rewrite index finds
\texttt{one\_div\_mul\_eq\_div}, which turns $\tfrac{1}{a}\cdot b$ into
$\tfrac{b}{a}$, and \texttt{Finset.sum\_div}, which distributes division over a
finite sum. Neither equation mentions the trigonometric content, and rewriting
with them in turn takes one statement to the other. \Eval{} settles the pair in
both directions by exactly the two-step chain above at search depth two, but
not at the default depth of one (Appendix~\ref{app:eval}). The autoformalized
candidate in the recall suite also reindexes the sum, and the search does not
settle it. The chain never proves the trigonometric identity itself, which it
carries across unchanged.

\paragraph{Changing the domain of a variable.} A chain moves a bound variable
to a new domain with two changes of variables. It generalizes over a variable
of the new domain and specializes the old variable at a term built from it, and
bridges rewrite the body in between. Write $S = \forall n{:}\mathbb{N}.\; P\,\uparrow\!n$ and
$T = \forall z{:}\mathbb{Z}.\; 0 \le z \to P\,z$. The chain from $S$ to $T$
passes through the subtype of non-negative integers and uses the library's
own biconditional \texttt{Subtype.forall}:
\begin{align*}
& \forall n{:}\mathbb{N}.\; P\,\uparrow\!n \\
& \forall w{:}\{z \mid 0 \le z\}.\; \forall n{:}\mathbb{N}.\; P\,\uparrow\!n
   && \text{generalize } w \\
& \forall w.\; P\,\uparrow\!(w.\mathrm{val}.\mathrm{toNat})
   && \text{specialize } n := w.\mathrm{val}.\mathrm{toNat} \\
& \forall w.\; P\,(w.\mathrm{val})
   && \text{bridge \texttt{Int.toNat\_of\_nonneg}, condition from $w$'s type} \\
& \forall z.\; \forall h{:}0 \le z.\; P\,(\langle z, h\rangle.\mathrm{val})
   && \text{bridge \texttt{Subtype.forall}} \\
& \forall z{:}\mathbb{Z}.\; 0 \le z \to P\,z
   && \text{definitional unfolding.}
\end{align*}
The reverse direction cannot apply \texttt{Subtype.forall} backwards and
specialize $w$ at $\langle \uparrow\!n,\, p\rangle$, because the proof $p$ of
$0 \le \uparrow\!n$ would be the library theorem \texttt{Int.natCast\_nonneg},
and a term that a step introduces may contain no library proof. Two bridges
give a chain instead. $T$ unfolds to
$\forall z \in \mathrm{Ici}\,0.\; P\,z$, the bridge
\texttt{Nat.range\_cast\_int} rewrites $\mathrm{Ici}\,0$ to the range of the
cast from $\mathbb{N}$, and \texttt{Set.forall\_mem\_range} rewrites the result
to $S$. A model proof that takes the direct route instead meets a hypothesis
of the antecedent with a library fact, which \Checker{} admits under its
tolerance, and Appendix~\ref{app:cases} shows the same case for a positivity
hypothesis.
\section{Two failures of evaluation by proofs}
\label{app:beqplus}

Both pairs below are labelled not equivalent in ProofNetVerif. The label
concerns whether the candidate states the intended claim, so a pair can be
labelled not equivalent even when both statements are provable. The listings
show simplified proof steps from the \beqplus{} strategies, without setup code
or tactic wrappers. In each direction, \texttt{base\_theorem} names the
antecedent, the statement on the other side of the pair, and the goal is the
statement to prove.

The first reference states the associativity of multiplication in
\texttt{ZMod n}, Lean's type for the integers modulo $n$. The candidate states
an identity about remainders of natural numbers, which follows from
\texttt{Nat.mul\_mod} with the equality reversed.

\smallskip
\noindent
\begin{minipage}[t]{0.48\linewidth}
{\footnotesize\textbf{Reference} (associativity in $\mathbb{Z}/n\mathbb{Z}$)}
\begin{leancode}[unbreakable, before skip=4pt, after skip=4pt,
  overlay={\node[anchor=north east, inner sep=2.5pt, font=\scriptsize\sffamily, text=leanCTag]
    at (frame.north east) {Lean 4};}]
\begin{Verbatim}[commandchars=\\\{\},fontfamily=zi4,fontsize=\leancodesizemini,formatcom=\color{leanCFg},baselinestretch=1.0,xleftmargin=0pt]
\textcolor{leanCKw}{theorem} \textcolor{leanCDecl}{exercise_1_1_4} \textcolor{leanCBr1}{(}n : Nat\textcolor{leanCBr1}{)}
    \textcolor{leanCBr1}{(}a b c : ZMod n\textcolor{leanCBr1}{)} :
  \textcolor{leanCBr1}{(}a * b\textcolor{leanCBr1}{)} * c = a * \textcolor{leanCBr1}{(}b * c\textcolor{leanCBr1}{)}
\end{Verbatim}
\end{leancode}
\end{minipage}\hfill
\begin{minipage}[t]{0.48\linewidth}
{\footnotesize\textbf{Candidate} (labelled not equivalent)}
\begin{leancode}[unbreakable, before skip=4pt, after skip=4pt,
  overlay={\node[anchor=north east, inner sep=2.5pt, font=\scriptsize\sffamily, text=leanCTag]
    at (frame.north east) {Lean 4};}]
\begin{Verbatim}[commandchars=\\\{\},fontfamily=zi4,fontsize=\leancodesizemini,formatcom=\color{leanCFg},baselinestretch=1.0,xleftmargin=0pt]
\textcolor{leanCKw}{theorem} \textcolor{leanCDecl}{candidate} \textcolor{leanCBr1}{(}n : Nat\textcolor{leanCBr1}{)}
    \textcolor{leanCBr1}{(}a b c : Nat\textcolor{leanCBr1}{)} :
  \textcolor{leanCBr1}{(}a*b\textcolor{leanCBr1}{)} % n * \textcolor{leanCBr1}{(}c%n\textcolor{leanCBr1}{)} % n
    = a*b*c % n
\end{Verbatim}
\end{leancode}
\end{minipage}

\smallskip
\noindent
\begin{minipage}[t]{0.48\linewidth}
{\footnotesize\textbf{Simplified proof steps,} reference $\Rightarrow$ candidate}
\begin{leancode}[unbreakable, before skip=4pt, after skip=4pt,
  overlay={\node[anchor=north east, inner sep=2.5pt, font=\scriptsize\sffamily, text=leanCTag]
    at (frame.north east) {Lean 4};}]
\begin{Verbatim}[commandchars=\\\{\},fontfamily=zi4,fontsize=\leancodesizemini,formatcom=\color{leanCFg},baselinestretch=1.0,xleftmargin=0pt]
convert base_theorem \textcolor{leanCKw}{using} \textcolor{leanCNum}{0}
simp_all_arith!
exact \textcolor{leanCKw}{fun} n a b c =>
  base_theorem n a b c
\end{Verbatim}
\end{leancode}
\end{minipage}\hfill
\begin{minipage}[t]{0.48\linewidth}
{\footnotesize\textbf{Simplified proof steps,} candidate $\Rightarrow$ reference}
\begin{leancode}[unbreakable, before skip=4pt, after skip=4pt,
  overlay={\node[anchor=north east, inner sep=2.5pt, font=\scriptsize\sffamily, text=leanCTag]
    at (frame.north east) {Lean 4};}]
\begin{Verbatim}[commandchars=\\\{\},fontfamily=zi4,fontsize=\leancodesizemini,formatcom=\color{leanCFg},baselinestretch=1.0,xleftmargin=0pt]
convert base_theorem \textcolor{leanCKw}{using} \textcolor{leanCNum}{0}
simp_all_arith!
noncomm_ring
\end{Verbatim}
\end{leancode}
\end{minipage}
\smallskip

The \texttt{convert} step turns the goal into a comparison between the goal
and the antecedent as propositions. Simplification removes the remainder
identity but leaves an associativity obligation. In the direction on the
left, the last line proves that obligation, which is the reference itself,
with \texttt{base\_theorem}. In the direction on the right,
\texttt{noncomm\_ring} proves it without the antecedent. The candidate thus
has a direct proof by \texttt{Nat.mul\_mod} and the reference a direct proof
by \texttt{mul\_assoc}, and proofs of the two statements taken separately do
not show that the candidate states the intended claim
(Section~\ref{sec:beq_background}).

The second reference states that disjoint closed sets in a metric space have
disjoint open neighbourhoods, which \texttt{SeparatedNhds} expresses. The
candidate concludes only that the sets are disjoint, from the assumption that
their intersection is empty.

\smallskip
\noindent
\begin{minipage}[t]{0.48\linewidth}
{\footnotesize\textbf{Reference} (Rudin 2.19a)}
\begin{leancode}[unbreakable, before skip=4pt, after skip=4pt,
  overlay={\node[anchor=north east, inner sep=2.5pt, font=\scriptsize\sffamily, text=leanCTag]
    at (frame.north east) {Lean 4};}]
\begin{Verbatim}[commandchars=\\\{\},fontfamily=zi4,fontsize=\leancodesizemini,formatcom=\color{leanCFg},baselinestretch=1.0,xleftmargin=0pt]
\textcolor{leanCKw}{theorem} \textcolor{leanCDecl}{exercise_2_19a}
    \textcolor{leanCBr1}{\{}X : \textcolor{leanCKw}{Type}*\textcolor{leanCBr1}{\}} \textcolor{leanCBr1}{[}MetricSpace X\textcolor{leanCBr1}{]}
    \textcolor{leanCBr1}{(}A B : Set X\textcolor{leanCBr1}{)}
    \textcolor{leanCBr1}{(}hA : IsClosed A\textcolor{leanCBr1}{)}
    \textcolor{leanCBr1}{(}hB : IsClosed B\textcolor{leanCBr1}{)}
    \textcolor{leanCBr1}{(}hAB : Disjoint A B\textcolor{leanCBr1}{)} :
  SeparatedNhds A B
\end{Verbatim}
\end{leancode}
\end{minipage}\hfill
\begin{minipage}[t]{0.48\linewidth}
{\footnotesize\textbf{Candidate} (labelled not equivalent)}
\begin{leancode}[unbreakable, before skip=4pt, after skip=4pt,
  overlay={\node[anchor=north east, inner sep=2.5pt, font=\scriptsize\sffamily, text=leanCTag]
    at (frame.north east) {Lean 4};}]
\begin{Verbatim}[commandchars=\\\{\},fontfamily=zi4,fontsize=\leancodesizemini,formatcom=\color{leanCFg},baselinestretch=1.0,xleftmargin=0pt]
\textcolor{leanCKw}{theorem} \textcolor{leanCDecl}{candidate}
    \textcolor{leanCBr1}{\{}X : \textcolor{leanCKw}{Type}*\textcolor{leanCBr1}{\}} \textcolor{leanCBr1}{[}MetricSpace X\textcolor{leanCBr1}{]}
    \textcolor{leanCBr1}{\{}A B : Set X\textcolor{leanCBr1}{\}}
    \textcolor{leanCBr1}{(}hA : IsClosed A\textcolor{leanCBr1}{)}
    \textcolor{leanCBr1}{(}hB : IsClosed B\textcolor{leanCBr1}{)}
    \textcolor{leanCBr1}{(}h : A \ensuremath{\cap} B = \ensuremath{\emptyset}\textcolor{leanCBr1}{)} :
  Disjoint A B
\end{Verbatim}
\end{leancode}
\end{minipage}

\smallskip
\noindent
\begin{minipage}[t]{0.48\linewidth}
{\footnotesize\textbf{Simplified proof steps,} reference $\Rightarrow$ candidate}
\begin{leancode}[unbreakable, before skip=4pt, after skip=4pt,
  overlay={\node[anchor=north east, inner sep=2.5pt, font=\scriptsize\sffamily, text=leanCTag]
    at (frame.north east) {Lean 4};}]
\begin{Verbatim}[commandchars=\\\{\},fontfamily=zi4,fontsize=\leancodesizemini,formatcom=\color{leanCFg},baselinestretch=1.0,xleftmargin=0pt]
\textcolor{leanCKw}{have} this :=
  base_theorem A B hA hB
    \textcolor{leanCBr1}{(}Set.disjoint_iff_inter_eq_empty
      .mpr h\textcolor{leanCBr1}{)}
exact SeparatedNhds.disjoint this
\end{Verbatim}
\end{leancode}
\end{minipage}\hfill
\begin{minipage}[t]{0.48\linewidth}
{\footnotesize\textbf{Simplified proof steps,} candidate $\Rightarrow$ reference}
\begin{leancode}[unbreakable, before skip=4pt, after skip=4pt,
  overlay={\node[anchor=north east, inner sep=2.5pt, font=\scriptsize\sffamily, text=leanCTag]
    at (frame.north east) {Lean 4};}]
\begin{Verbatim}[commandchars=\\\{\},fontfamily=zi4,fontsize=\leancodesizemini,formatcom=\color{leanCFg},baselinestretch=1.0,xleftmargin=0pt]
\textcolor{leanCKw}{have} this : Disjoint A B := \textcolor{leanCKw}{by}
  apply_rules \textcolor{leanCBr1}{[}base_theorem\textcolor{leanCBr1}{]}
exact NormalSpace.normal
  A B hA hB this
\end{Verbatim}
\end{leancode}
\end{minipage}
\smallskip

In the direction on the left the antecedent is used but not needed. The
argument passed to it, \texttt{Set.disjoint\_iff\_inter\_eq\_empty.mpr h},
already proves the goal \texttt{Disjoint A B} from the goal's own hypothesis,
and \texttt{SeparatedNhds.disjoint} only turns the antecedent's conclusion back
into that goal. In the direction on the right the antecedent is not used at
all. \texttt{apply\_rules} proves the intermediate claim from the hypothesis
\texttt{hAB}, Lean prints that step as \texttt{have this := hAB}, and
\texttt{base\_theorem} does not appear in the proof term.
\texttt{NormalSpace.normal} then proves the reference from the hypotheses that
the sets are closed and disjoint. The right direction shows that a name in a
script need not appear in the proof term, and the left direction shows that
even an occurrence in the proof term can be unnecessary.

\section{Details of \Eval{}}
\label{app:eval}

\paragraph{Direct use.} The first stage applies the antecedent to the goal and
lets the logical rules match the two, which settles a direction when the goal
is the antecedent up to renaming, unfolding, or instantiation.

\paragraph{Normalisation.} The normalisation stage rewrites both statements
with a fixed set that carries no mathematical content: the library's cast
normalisation extension, whose lemmas are equations and biconditionals by
construction, and a curated pool of logical and set-membership identities
such as \texttt{not\_forall}, \texttt{forall\_and}, \texttt{exists\_prop},
\texttt{and\_imp}, \texttt{Set.mem\_setOf\_eq}, \texttt{Set.ext\_iff},
\texttt{Finset.mem\_range}, and \texttt{Subtype.ext\_iff}. Every member is
an equation or a biconditional, so the stage performs rearrangement and
bridge rewriting only. The antecedent itself is available as a rewrite rule,
which is how a direction closes when the goal is a rephrasing of the
antecedent's conclusion.

\paragraph{Logical rearrangement.} Rearrangement is performed by a closer
that asserts the antecedent into the local context, destructures conjunctions,
existentials, biconditionals and disjunctions on both sides, instantiates
universally quantified hypotheses at matching local objects, and searches
backwards over the logical rules with a small depth bound. No library fact
enters except through the normalisation set above.

\paragraph{Discovery.} Bridges are found by indexing: the two statements are
matched against the library's rewrite index, every fact whose left or right
side matches a subterm is collected, and the pool is filtered to facts whose
conclusion is an equation or a biconditional. The pool is ranked and capped
at forty candidates. Each candidate is materialised as an attempt that
rewrites the goal or the antecedent at the matching site, in either
orientation. Discovery runs only after the cheap
stages have failed, since those stages settle about three quarters of all
directions at a small fraction of the time.

\paragraph{Search strength.} Three parameters bound the search
(Table~\ref{tab:capability}). $D$ is the depth of chained rewrites, the
number of bridges applied in sequence before the closer runs; $W_k$ doubles
$k$ times the two per-step bridge budgets, the candidate bridges tried and the
lemmas discovered, which are 80 and 64 by default; $DW_k$ doubles $k$ times
the number of distinct intermediate statements retained at each depth and
carried to the next, 24 by default, and raises the budget spent per depth
level by a factor of four each time. The default is $D = 1$ with the base
budgets.

\paragraph{Proof-term grading.} Every constant of an accepted proof term is
classified, with local auxiliaries expanded through the kernel environment so
that nothing hides inside a tactic's private lemma. Admitted are the
antecedent; non-propositional constants, which are vocabulary; eliminators;
schematic constants, whose type is built from sorts, variables and the core
logical connectives and which therefore carry no domain content; facts whose
conclusion is an equation, a biconditional, a heterogeneous equation or a
disequality, or a class instance; and a fixed set of sixteen numeral facts,
such as \texttt{le\_refl}, \texttt{zero\_le\_one},
\texttt{two\_ne\_zero} and \texttt{Nat.zero\_le}, which discharge the
conditions of cast bridges and can relate no two statements on their own.
Any other constant fails the proof. The same classification filters the
discovery pool before any rewrite is attempted, so closer and audit agree by
construction.

\paragraph{No proof ex falso.} The proof term is walked with a negation depth.
Entering a lambda increases the depth only when its body proves
\texttt{False} and the introduced hypothesis actually occurs in that
derivation, which is the shape of an honest refutation. A \texttt{False}-typed
subterm at depth zero, and any of the eliminators \texttt{False.elim},
\texttt{absurd} and their kin at depth zero, fail the proof. A
self-contradictory antecedent paired with its own rearrangement is thus still
accepted, since its proof rearranges components and never derives
\texttt{False}, while the pair below is refused in both directions:

\begin{leancode}
\begin{Verbatim}[commandchars=\\\{\},fontfamily=zi4,fontsize=\leancodesize,formatcom=\color{leanCFg},baselinestretch=1.0,xleftmargin=0pt]
\textcolor{leanCCmt}{-- reference}
\textcolor{leanCKw}{theorem} \textcolor{leanCDecl}{r} : \ensuremath{\exists} n : \ensuremath{\mathbb{N}}, Even n \ensuremath{\wedge} \ensuremath{\lnot} Even n
\textcolor{leanCCmt}{-- candidate}
\textcolor{leanCKw}{theorem} \textcolor{leanCDecl}{c} : \ensuremath{\exists} x : \ensuremath{\mathbb{R}}, x = \textcolor{leanCNum}{0} \ensuremath{\wedge} x \ensuremath{\neq} \textcolor{leanCNum}{0}
\end{Verbatim}
\end{leancode}
 \paragraph{Relevance audit.} Before any search runs, this audit asks which parts of the goal actually need the antecedent. The goal is split into obligations, namely its conjuncts, the two sides of a biconditional, the body of an existential with its witness generalised again, and the domain of every hypothesis at any depth, and the split is taken on both the raw goal and its normalised form. An obligation that the antecedent asserts as one of its own components is kept. Every other obligation is attempted with the antecedent withheld, by the same proof procedure that each stage of the ladder in \S\ref{sec:eval} runs, with the same lemma pool and the same time budget. If any such attempt succeeds, that part of the goal is provable on its own, and the direction is refused without being searched. The pair below is refused this way, because its second conjunct is provable by reflexivity.

\begin{leancode}
\begin{Verbatim}[commandchars=\\\{\},fontfamily=zi4,fontsize=\leancodesize,formatcom=\color{leanCFg},baselinestretch=1.0,xleftmargin=0pt]
\textcolor{leanCCmt}{-- reference}
\textcolor{leanCKw}{theorem} \textcolor{leanCDecl}{r} : Irrational \textcolor{leanCBr1}{(}Real.sqrt \textcolor{leanCNum}{2}\textcolor{leanCBr1}{)}
\textcolor{leanCCmt}{-- candidate}
\textcolor{leanCKw}{theorem} \textcolor{leanCDecl}{c} : Irrational \textcolor{leanCBr1}{(}Real.sqrt \textcolor{leanCNum}{2}\textcolor{leanCBr1}{)} \ensuremath{\wedge} \textcolor{leanCBr1}{(}\textcolor{leanCNum}{2} + \textcolor{leanCNum}{2} = \textcolor{leanCNum}{4}\textcolor{leanCBr1}{)}
\end{Verbatim}
\end{leancode}

\noindent

\paragraph{Antecedent-free replay.} Every accepting attempt records the
transformation it applied, the lemma, orientation and site. The identical
transformation is replayed with the antecedent absent from the context, and
an acceptance whose replay still closes the goal is discarded.

\paragraph{Determinism.} Every judgement is a function of the two statements
alone. Per-attempt budgets are absolute rather than shares of a shrinking
remainder, the relevance probe is evaluated once per direction before the
ladder, and no stage depends on wall-clock time. Two runs on the same
statements give the same verdict on any machine.

\paragraph{A known false acceptance.} On the validation split, which
Table~\ref{tab:pnv} does not score, \Eval{} wrongly accepts one pair, Munkres
18.8b. The candidate adds a preorder on the space that its conclusion never
uses. In the direction from the candidate, the search supplies the
specialization preorder of the space, a structure the goal does not carry,
which the change-of-variables clause of Definition~\ref{def:faithful-step}
is meant to exclude. The instantiation check of \Checker{} refuses the same proof term, and
the search has no such check.

\section{Details of \Checker{}, and what it accepts beyond the definition}
\label{app:checker}

\paragraph{Reading a proof term.} \Checker{}, which we call the auditor, reads the proof term that Lean elaborates from a tactic script, not the script itself. The proof term is the fully explicit proof that Lean's kernel checks. The auditor walks the term from the root and gives every position one of three roles, shown in Figure~\ref{fig:checker}. First, the \emph{goal derivation} consists of the positions that produce the goal, namely the root, the body of every leading lambda and of every local binding, and the head and the arguments of each application. When the proof converts along an equation or a biconditional, the proof being converted also belongs to the goal derivation. Second, the \emph{antecedent arguments} are the arguments handed to the antecedent, or to a fact obtained from it, when it is applied. Third, the \emph{justification} of a conversion, labelled side condition in Figure~\ref{fig:checker}, is the equation or biconditional along which a rewrite by \texttt{rw} or \texttt{simp}, or a use of \texttt{Iff.mp} or \texttt{Iff.mpr}, converts the proof. A constant is judged by the role of the position it occupies and by the shape of its statement, not by its name. On the goal derivation only the antecedent, logical rules and bridges are admitted, and any other library fact fails the proof. A library bridge is admitted there only as a step of a conversion, or as an argument of a rule that combines equations, such as \texttt{Eq.trans} or \texttt{congrArg}. Applied anywhere else on the goal derivation, with no mention of the antecedent in its application, it fails the proof, because its conclusion then states part of the goal outright instead of rewriting the antecedent into it. The rule concerns library facts, not local lemmas. Library facts at the other two roles fall under the tolerance described below. The checks for unproved holes, for proof ex falso and on applications of the antecedent run at every position, justifications included.

\paragraph{Use of the antecedent in every part.} The auditor requires the antecedent to do work in each part of the proof, and it finds the parts by their position in the proof term. Not every mention of the antecedent counts as a use. A mention does not count when it is the discarded argument of \texttt{Function.const}, the value of a \texttt{have} or \texttt{let} whose name never occurs again, or part of a subterm whose statement is \texttt{True} or a closed equation with two sides that are equal by definition. A part whose proof has no use of the antecedent, and whose statement does not mention the antecedent, fails the proof. The parts are the following: (i) the whole proof, and on the goal derivation the rest of the proof after each \texttt{have} or \texttt{let}; (ii) each branch of a case split on anything not derived from the antecedent, and each branch of a case split on the antecedent, or on a fact derived from it, when nothing that the split produces is used; (iii) a proof passed as a lambda to a logical rule on the goal derivation, and the pointwise proof in an extensionality step such as \texttt{funext} or \texttt{Set.ext}; and (iv) every other component of a conjunction or a biconditional built on the goal derivation, unless it lies under a rewrite that the antecedent justifies. The rules above do not apply inside a justification, where a related rule applies instead. The related rule refuses a rewrite that turns a part of the goal into \texttt{True}, or drops it from a conjunction, by means of a proof of the part that does not use the antecedent.

\paragraph{Local lemmas.} A local lemma proved from the antecedent is expanded in place, and its proof is audited at the role of the position where the lemma is used. A local lemma proved without the antecedent, and any lemma in the file whose proof is left open, counts as an outside fact. Its statement is classified by shape and admitted or refused at the role where it is used, like a library fact. When the shape admits it, as an equation or a biconditional, its proof is read as well, and each direction of the proof is checked as an implication proof in its own right, with the direction's hypothesis standing in for the antecedent. A direction that proves a part of the goal without using its hypothesis fails the proof. The rule refuses a proof that states, as a local lemma, the equivalence of two statements that the library proves, proves each side outright, and rewrites along it. It also catches an added claim inside such a lemma. Every library fact that the proof of a local lemma uses is recorded like any other.

\paragraph{Instantiation of the antecedent.} Each time the proof applies the antecedent by name, the auditor checks the arguments of the application against the goal's own objects, which it fixes by aligning the two statements. It first counts the data variables that each statement binds, meaning the variables that are not proofs. When the goal binds fewer, it also counts those that appear once the goal's conclusion is unfolded. The goal's own objects are the first data variables that the proof introduces for the goal, as many as the larger count, together with the objects that the proof obtains from them by destructuring a hypothesis or by a local definition. A data variable introduced later is not one of the goal's own objects. The checks on type and data arguments concern only an argument that appears in the statement that the application delivers. First, a type argument, meaning an argument whose own type is a universe such as \texttt{Type} or \texttt{Prop}, must belong to the goal. If it contains a free variable, it must be one of the goal's own objects or a type that the goal statement mentions. If it is closed, it must share a constant other than a logical connective with the goal statement, and \texttt{Prop} itself passes only when the goal statement mentions \texttt{Prop}. Neither rule applies when the goal statement quantifies existentially over a type. Second, every data variable in a data argument must be one of the goal's own objects, and a data argument that the delivered statement uses as a type must itself be one of the goal's own objects or appear in the goal statement. Proofs and closed terms are not restricted. Third, a structure argument, meaning an instance in Lean's sense, is checked when it contains a free variable, its class concerns one of the goal's own objects, and it is not itself one of the goal's own instances. It passes if it is equal by definition to the goal's instance of the same class or appears in the goal statement. Otherwise, if its class can have more than one instance, it fails when the goal carries its own instance of the class for the same object. If its class has at most one instance, it fails unless Lean's instance search can build it with classical decidability available, or the goal statement assumes finiteness with a \texttt{Finite} or \texttt{Fintype} instance. Fourth, a witness that the proof writes out for an existential hypothesis of the antecedent passes if it contains a variable not derived from the antecedent, or if it is a closed term that the goal statement mentions. A closed witness foreign to the goal statement fails the proof, e.g., a fresh constant when the goal's own hypothesis already provides the object, and so does a witness built only from variables derived from the antecedent. The checks do not follow a local name given to a partial application of the antecedent, so an argument passed through such a name is not checked.

\paragraph{Replay without the antecedent.} The instantiation checks are backed by a replay, which runs only on a proof that passes every other check. The replay withholds the antecedent and keeps, by value, the local definitions and local facts that the proof builds without it. It then asks a weak prover, Lean's \texttt{assumption} followed by \texttt{exact?}, to close the goal from what remains, and a proof that uses the antecedent, another assumption or the goal itself does not count. The replay first uses the definitions and facts at the head of the proof, up to the first one that involves the antecedent. If the prover fails there, the replay follows the proof's introductions, its rewrites that do not use the antecedent, and the branches of its case splits. It probes at most four of the goals it reaches. If the prover closes such a goal with the proof's constructions, but cannot close it, or the original goal, from the library alone, the antecedent did no work and the proof is refused. If the prover closes the goal from the library alone, the replay neither refuses nor records the direction.

\paragraph{The tolerance, and a stricter audit policy.} \Eval{} meets the hypotheses of a bridge only with local hypotheses, its fixed rewrite set and the numeral facts of Appendix~\ref{app:eval}, and on the goal derivation the auditor admits only the antecedent, logical rules and bridges. The auditor tolerates library facts in two other positions. At an antecedent argument, a library fact may meet a hypothesis of the antecedent. A change of domain often leaves such a hypothesis, e.g., $0 \le \uparrow\!n$ in Appendix~\ref{app:chains}, and refusing the fact would reject a faithful proof for a gap in the library rather than for a fault of the proof. Inside a justification, a library fact of any shape is tolerated, because the current auditor does not distinguish a proof of a bridge's hypothesis from other content there. In both positions, every use of a fact that is neither a bridge nor a logical rule is recorded. A use at an antecedent argument is marked \emph{severe} when the fact, together with its arguments, draws on no local hypothesis and on nothing derived from the antecedent, even if it mentions the goal's variables. A use that draws on a hypothesis of the goal stays within what the goal grants, while a severe use meets a hypothesis of the antecedent with a fact that the goal never states. The stricter policy counts every severe use as a failure, at the top level and inside the proofs of local lemmas alike. The auditor offers this policy as a mode, and Table~\ref{tab:strict} reports the campaign of \S\ref{sec:campaign} under it.

\begin{table}[h]
\centering
\caption{The campaign of Table~\ref{tab:campaign} under the stricter audit
policy, in which a severe use, a library fact meeting a hypothesis of
the antecedent while drawing on no local hypothesis, fails the proof wherever
it occurs. Pairs are counted as both directions, one direction, or neither,
and precision and recall are over pair-level acceptance, as in
Table~\ref{tab:campaign}.}
\label{tab:strict}
\setlength{\tabcolsep}{3.5pt}
\resizebox{\textwidth}{!}{%
\begin{tabular}{lrrrcrrrcrr}
\toprule
 & \multicolumn{3}{c}{Equivalent (407)} & & \multicolumn{3}{c}{Non-equivalent (104)} & & & \\
\cmidrule(lr){2-4}\cmidrule(lr){6-8}
Prompt style & both & one & neither & & both & one & neither & & Precision & Recall \\
\midrule
Expedient    & 303 & 45 & 59 & & 0 & 46 & 58 & & 100.0 & 74.4 \\
Unrestricted & 317 & 56 & 34 & & 0 & 50 & 54 & & 100.0 & 77.9 \\
Faithful     & \textbf{323} & 56 & 28 & & 0 & 59 & 45 & & 100.0 & \textbf{79.4} \\
Rewrite      & 323 & 56 & 28 & & 0 & 64 & 40 & & 100.0 & 79.4 \\
\bottomrule
\end{tabular}}
\end{table}

Across the four styles, 199 proof directions carry a tolerated use at a hypothesis of the antecedent, 125 through a library constant and 74 only through a local lemma proved without the antecedent. The severe mark applies to 16, 13, 10, and 14 of them in the faithful, rewrite, expedient, and unrestricted styles, and 34 of the 53 severe uses draw on no local object at all. The facts involved are the openness of a fixed interval, the continuity of the identity on the reals, the openness of the whole space, the compactness of a singleton, and numeral facts such as the primality of a specific number. On the equivalent pairs, the stricter policy removes 7, 6, 6, and 5 pairs from the count certified in both directions, in the faithful, rewrite, expedient, and unrestricted styles. In each pair removed, the candidate carries a hypothesis that the library proves on its own, such as the continuity of the identity, and the direction that drops the hypothesis meets it with the library fact. Two further pairs lose their only certified direction, one of them in three styles, and in the fourth that pair is among those that lose both directions, because a numeral fact meets a hypothesis of the antecedent inside the proof of a local lemma, a use that the default policy records but only the stricter policy fails. On the non-equivalent pairs, the severe direction was the only certified one, so the pairs concerned move from one direction to neither. Under either policy, no non-equivalent pair is certified in both directions. Counting every tolerated use at a hypothesis of the antecedent as a failure, severe or not, lowers the faithful count further, to 303 of 407, but it also fails the uses that draw on a hypothesis of the goal.

\paragraph{Differences from the audits of \Eval{}.} The auditor has no
relevance audit that asks the library whether a part of the goal is provable
on its own. It judges the proof it is given, and it refuses a part of the goal
whose proof never depends on the antecedent, which covers a candidate that
only adds a claim the library proves on its own whenever the model proves
that claim without the antecedent. It does not grade the justification of a
rewrite, so a library fact can enter there, although a part of the goal
collapsed to a truth value inside a justification is refused. The auditor
reports every axiom a proof uses beyond the three standard ones but does not
refuse on them. One accepted proof, in the unrestricted style, uses
\texttt{native\_decide} to show that \texttt{ZMod 2} has two elements, a
true side fact that \texttt{decide} also proves. In the other direction, the
search does not check how the antecedent is instantiated, and the auditor
does.

\section{Datasets and adjudication}
\label{app:data}

\paragraph{ProofNetVerif.} The benchmark provides 3{,}752 reference and
candidate pairs with human labels in a validation and a test split. We score
the test split, 1{,}452 pairs of which 1{,}303 elaborate under our toolchain
(Appendix~\ref{app:repro}), under its published labels and under the
adjudicated labels described next. We reviewed the test pairs labelled
equivalent, and candidates for review came from a language-model screen and
from the results of our two programs. First, a language model, Claude Fable 5,
screened the pairs labelled equivalent as a pre-filter. Second, we reviewed the
pairs that no configuration of \Eval{} accepts and that the campaign of
\S\ref{sec:campaign} certifies in at most one direction. A person made every
final decision. The review changed 102 labels, all from equivalent to
non-equivalent, since pairs labelled not equivalent were not reviewed. The
released provenance shows that the review of 88 of the 102 was prompted by the
second source. The other 14 came from the language-model screen and from
manual rulings. At least eight changes follow from our criterion rather than
from an error in the benchmark. Four Artin 2.8.6 candidates state that the
centre of the product equals the product of the centres, where the reference
states an isomorphism, and four Rudin 1.4 candidates add a linear order
(Appendix~\ref{app:cases}). Only two relabelled pairs are ones that \beqplus{}
accepts. Both are Munkres 29.4, and each candidate states that
$\mathbb{N} \to [0,1]$ with the product topology is not locally compact, which
is false, since that space is compact. We report both label sets, and they give the same ranking of
the systems. Our criterion is equivalence with the reference, in the sense that
the candidate makes the same claim as the reference. The criterion is a human
judgement, and it is stricter than provable equivalence, as the last case of
Appendix~\ref{app:cases} shows. ProofNetVerif labels whether a candidate
faithfully formalizes the informal statement, so a candidate can be faithful to
the text without making the same claim as the reference, e.g., when the
reference states a more general result (Appendix~\ref{app:cases}). We did not
review the validation split, and Table~\ref{tab:pnv} does not score it.

\paragraph{The ablation set.} The 100 non-equivalent pairs of
Table~\ref{tab:ablation} comprise 50 ProofNetVerif test-split pairs and 50 constructed
pairs. The ProofNetVerif half contains all ten pairs \beqplus{} accepts on the
adjudicated test split under our toolchain, and 40 pairs sampled from the
remaining non-equivalent pairs. The constructed half has four classes, each the
smallest form of a failure met elsewhere in the paper. In the 18 \emph{added
claim} pairs, the candidate adds to a real statement a claim that the reference
does not make, e.g., a conjoined fact, an extra hypothesis, a tightened bound or
a stronger conclusion. When the added claim is a library fact, both
implications are provable, so an evaluator that finds both proofs can refuse
the pair only by checking how they use the antecedent. Added claims are the most
common defect among the 102 relabelled benchmark pairs, and the relevance audit
targets them. In the 14 \emph{neighbouring theorem} pairs, the candidate is a
different true theorem about the same objects, phrased like the reference, such
as Lagrange's theorem against Cauchy's. Such pairs show the first failure of
\S\ref{sec:algorithm}, and the two \beqplus{} acceptances of
Appendix~\ref{app:beqplus} are examples. In the 14 \emph{unrelated true fact}
pairs, the candidate is a fact provable on its own, e.g., a decidable arithmetic
identity, a restatement of a library theorem, or a statement about a decoy
instance, and the introduction gives an example. One of these pairs is the
pair found by Simon Sorg that \citet{poiroux2025reliable} report in their
Appendix A.6.1. \beqplus{} rejects it because of the guard added for it, and
\Eval{} needs proof-term grading to reject it. In the 4 \emph{vacuous
hypotheses} pairs, the hypotheses contradict each other or contain two
unrelated contradictions, like the vacuous statements that benchmark audits
report, and the explosion check targets them.

\paragraph{Recall suite.} Forty pairs: fifteen hand-written alternative
statements, each realising one named transformation of its reference, and
twenty-five candidates autoformalized blind from the natural-language
statement, without sight of the reference, for problems from Compfiles,
ProofNet\# and PutnamBench, each with a hand-proved equivalence witness. Two
pairs were removed on review, one as one-directional, its reference
under-formalizing the problem, and the IMO 1983 pair, whose candidate, as the
revision loop with the auditor in Appendix~\ref{app:cases} showed, claims what
the reference does not, leaving 38.

\paragraph{The certification set.} The model was run on 511 pairs: the 471
test-split pairs labelled equivalent before adjudication that elaborate on
our toolchain, and the 40 recall pairs. After adjudication these split into
407 equivalent pairs, 369 from ProofNetVerif and the 38 recall pairs, and 104
non-equivalent pairs, the 102 relabelled pairs and the two removed recall pairs.
\section{Full capability grid}
\label{app:grid}

Table~\ref{tab:capability} is the excerpt of the grid that Section~\ref{sec:capability} discusses, and Table~\ref{tab:capability-full} gives every configuration.

\begin{table}[t]
\centering
\caption{Effect of search strength on recall and cost over the \emph{equivalent set}. $D$ is the depth of chained
rewrites; $W_k$ and $DW_k$ double $k$ times the per-step bridge budgets and
the number of intermediate statements retained at each depth. Recall is not
comparable with Table~\ref{tab:pnv}, which does not contain the recall suite. The
full grid is Table~\ref{tab:capability-full}.}
\label{tab:capability}
\begin{tabular}{lrrrrrr}
\toprule
Configuration & Median & Mean & Max & Total & New & Recall \\
               & (s)   & (s)  & (s) & CPU   &     & \\
\midrule
BEq+  &   1.0 &   6.3 &   74 & 0.71\,h          & --- & 60.2\% \\
\midrule
D1 (base)     &   2.0 &  30.5 &  489 & \textbf{3.45\,h} & ---  & \textbf{65.4\%} \\
D1$+$W1       &   2.0 &  43.2 &  541 & 4.88\,h          & $+2$ & 65.8\% \\
D2            &   2.0 &  50.7 &  606 & 5.73\,h          & $+1$ & \textbf{66.1\%} \\
D8$+$W2$+$DW2 &   2.0 & 271.3 & 3126 & \textbf{30.67\,h} & 0   & 66.1\% \\
\bottomrule
\end{tabular}
\end{table}

\begin{table}[h]
\centering
\caption{Effect of search strength on recall and cost over the 407
adjudicated-equivalent pairs, all configurations. $D$, $W$ and $DW$ are
defined in Appendix~\ref{app:eval}. ``New'' counts pairs accepted by the
configuration and by none cheaper.}
\label{tab:capability-full}
\begin{tabular}{lrrrrrr}
\toprule
Configuration & Median & Mean & Max & Total & New & Recall \\
               & (s)   & (s)  & (s) & CPU   &     & \\
\midrule
\multicolumn{7}{l}{\emph{Reference point}}\\
BEq+ &   1.0 &   6.3 &   74 & 0.71\,h          & --- & 60.2\% \\
\midrule
D1 (base)      &   2.0 &  30.5 &  489 & \textbf{3.45\,h}  & ---  & \textbf{65.4\%} \\
D1$+$W1        &   2.0 &  43.2 &  541 & 4.88\,h           & $+2$ & 65.8\% \\
D2             &   2.0 &  50.7 &  606 & 5.73\,h           & $+1$ & \textbf{66.1\%} \\
D2$+$W1        &   2.0 &  65.0 &  642 & 7.34\,h           & 0    & 66.1\% \\
D2$+$W1$+$DW1  &   2.0 &  67.0 &  834 & 7.58\,h           & 0    & 66.1\% \\
D4             &   2.0 &  85.6 &  751 & 9.68\,h           & 0    & 66.1\% \\
D4$+$W1        &   2.0 & 100.2 &  808 & 11.33\,h          & 0    & 66.1\% \\
D4$+$W1$+$DW1  &   2.0 & 113.7 &  880 & 12.85\,h          & 0    & 66.1\% \\
D8$+$W1        &   2.0 & 148.7 & 1174 & 16.82\,h          & 0    & 66.1\% \\
D4$+$W2$+$DW2  &   2.0 & 172.2 & 2634 & 19.47\,h          & 0    & 66.1\% \\
D8$+$W2$+$DW2  &   2.0 & 271.3 & 3126 & \textbf{30.67\,h} & 0    & 66.1\% \\
\bottomrule
\end{tabular}
\end{table}

\section{Case studies of rejected model proofs}
\label{app:cases}

The non-equivalent pairs are instructive because the model often proves them
anyway, in one direction or with a proof that looks complete. The rejected
proof then shows where the two statements differ. We show three, each refused
by a different rule, and then an autoformalized pair whose label we changed
after a revision loop with the auditor. The three campaign proofs are printed
as the model wrote them, except that the candidate's theorem name,
\texttt{dummy} in ProofNetVerif, is printed as \texttt{candidate}. The proofs in
the last case were written in our development sessions, as described there.

\paragraph{Instantiating the antecedent at a foreign type.} The reference
formalizes Rudin's exercise 1.4, that a lower bound of a nonempty set is at
most an upper bound, for a partial order. The candidate states the same shape
for a linear order and an abstract element pair, which is closer to Rudin,
whose ordered sets are total. The reference is the more general statement, so
the two formal statements are not equivalent, although the candidate is a
faithful reading of the exercise. The model proves the
reference from the candidate anyway.

\begin{leancode}
\begin{Verbatim}[commandchars=\\\{\},fontfamily=zi4,fontsize=\leancodesize,formatcom=\color{leanCFg},baselinestretch=1.0,xleftmargin=0pt]
\textcolor{leanCCmt}{-- ===== assumed statement (antecedent) =====}
\textcolor{leanCKw}{theorem} \textcolor{leanCDecl}{candidate_assumed} \textcolor{leanCBr1}{\{}\ensuremath{\alpha} : \textcolor{leanCKw}{Type}*\textcolor{leanCBr1}{\}} \textcolor{leanCBr1}{[}LinearOrder \ensuremath{\alpha}\textcolor{leanCBr1}{]} \textcolor{leanCBr1}{\{}E : Set \ensuremath{\alpha}\textcolor{leanCBr1}{\}} \textcolor{leanCBr1}{(}hE : E.Nonempty\textcolor{leanCBr1}{)}
    \textcolor{leanCBr1}{\{}\ensuremath{\alpha}_bound : \ensuremath{\alpha}\textcolor{leanCBr1}{\}} \textcolor{leanCBr1}{(}h\ensuremath{\alpha} : \ensuremath{\forall} e \ensuremath{\in} E, \ensuremath{\alpha}_bound \ensuremath{\leq} e\textcolor{leanCBr1}{)} \textcolor{leanCBr1}{\{}\ensuremath{\beta}_bound : \ensuremath{\alpha}\textcolor{leanCBr1}{\}} \textcolor{leanCBr1}{(}h\ensuremath{\beta} : \ensuremath{\forall} e \ensuremath{\in} E, e \ensuremath{\leq} \ensuremath{\beta}_bound\textcolor{leanCBr1}{)}
    : \ensuremath{\alpha}_bound \ensuremath{\leq} \ensuremath{\beta}_bound := \textcolor{leanCKw}{by}
  \textcolor{leanCSorry}{sorry}

\textcolor{leanCCmt}{-- ===== goal statement, PROVED from the antecedent =====}
\textcolor{leanCKw}{theorem} \textcolor{leanCDecl}{exercise_1_4}
\textcolor{leanCBr1}{(}\ensuremath{\alpha} : \textcolor{leanCKw}{Type}*\textcolor{leanCBr1}{)} \textcolor{leanCBr1}{[}PartialOrder \ensuremath{\alpha}\textcolor{leanCBr1}{]}
\textcolor{leanCBr1}{(}s : Set \ensuremath{\alpha}\textcolor{leanCBr1}{)}
\textcolor{leanCBr1}{(}x y : \ensuremath{\alpha}\textcolor{leanCBr1}{)}
\textcolor{leanCBr1}{(}h\ensuremath{\sb{0}} : Set.Nonempty s\textcolor{leanCBr1}{)}
\textcolor{leanCBr1}{(}h\ensuremath{\sb{1}} : x \ensuremath{\in} lowerBounds s\textcolor{leanCBr1}{)}
\textcolor{leanCBr1}{(}h\ensuremath{\sb{2}} : y \ensuremath{\in} upperBounds s\textcolor{leanCBr1}{)}
: x \ensuremath{\leq} y := \textcolor{leanCKw}{by}
  classical
  obtain \textcolor{leanCBr1}{\ensuremath{\langle}}a, ha\textcolor{leanCBr1}{\ensuremath{\rangle}} := h\ensuremath{\sb{0}}
  \textcolor{leanCKw}{have} h := candidate_assumed
    \textcolor{leanCBr1}{(}E := \textcolor{leanCBr2}{\{}\textcolor{leanCBr3}{(}x \ensuremath{\leq} a : \textcolor{leanCKw}{Prop}\textcolor{leanCBr3}{)}\textcolor{leanCBr2}{\}}\textcolor{leanCBr1}{)}
    \textcolor{leanCBr1}{(}\ensuremath{\alpha}_bound := True\textcolor{leanCBr1}{)}
    \textcolor{leanCBr1}{(}\ensuremath{\beta}_bound := \textcolor{leanCBr2}{(}x \ensuremath{\leq} y\textcolor{leanCBr2}{)}\textcolor{leanCBr1}{)}
    \textcolor{leanCBr1}{(}Set.singleton_nonempty \textcolor{leanCBr2}{(}x \ensuremath{\leq} a : \textcolor{leanCKw}{Prop}\textcolor{leanCBr2}{)}\textcolor{leanCBr1}{)}
    \textcolor{leanCBr1}{(}\textcolor{leanCKw}{by}
      intro p hp
      subst p
      exact \textcolor{leanCKw}{fun} _ => h\ensuremath{\sb{1}} ha\textcolor{leanCBr1}{)}
    \textcolor{leanCBr1}{(}\textcolor{leanCKw}{by}
      intro p hp
      subst p
      exact \textcolor{leanCKw}{fun} hxa => hxa.trans \textcolor{leanCBr2}{(}h\ensuremath{\sb{2}} ha\textcolor{leanCBr2}{)}\textcolor{leanCBr1}{)}
  exact h True.intro
\end{Verbatim}
\end{leancode}

The candidate is a theorem about an arbitrary linearly ordered set, and
the proof instantiates that set at \texttt{Prop}, with the order's
elements standing in for propositions. The goal statement never mentions
\texttt{Prop}, but \texttt{Prop} appears in the statement that the application
of the antecedent delivers. The check on type arguments
(Appendix~\ref{app:checker}) therefore refuses the proof, and it also flags the
argument \texttt{True}, whose type \texttt{Prop} is a universe. The structure
check plays no part in the refusal, because it examines only instances on the
goal's own objects.

\paragraph{Instantiating the antecedent at a foreign object.} Munkres 22.5
states that the restriction of an open map to an open subspace is an open map.
The candidate states only that the image of the open set is open, which is a
single instance of the map being open and says nothing about the subspace. The
model proves the reference from the candidate.

\begin{leancode}
\begin{Verbatim}[commandchars=\\\{\},fontfamily=zi4,fontsize=\leancodesize,formatcom=\color{leanCFg},baselinestretch=1.0,xleftmargin=0pt]
\textcolor{leanCCmt}{-- ===== assumed statement (antecedent) =====}
\textcolor{leanCKw}{theorem} \textcolor{leanCDecl}{candidate_assumed} \textcolor{leanCBr1}{\{}X Y : \textcolor{leanCKw}{Type}*\textcolor{leanCBr1}{\}} \textcolor{leanCBr1}{[}TopologicalSpace X\textcolor{leanCBr1}{]} \textcolor{leanCBr1}{[}TopologicalSpace Y\textcolor{leanCBr1}{]} \textcolor{leanCBr1}{(}p :
    X \ensuremath{\rightarrow} Y\textcolor{leanCBr1}{)} \textcolor{leanCBr1}{(}hf : IsOpenMap p\textcolor{leanCBr1}{)} \textcolor{leanCBr1}{(}A : Set X\textcolor{leanCBr1}{)} \textcolor{leanCBr1}{(}hA : IsOpen A\textcolor{leanCBr1}{)} : IsOpen \textcolor{leanCBr1}{(}p '' A\textcolor{leanCBr1}{)} := \textcolor{leanCKw}{by}
  \textcolor{leanCSorry}{sorry}

\textcolor{leanCCmt}{-- ===== goal statement, PROVED from the antecedent =====}
\textcolor{leanCKw}{theorem} \textcolor{leanCDecl}{exercise_22_5} \textcolor{leanCBr1}{\{}X Y : \textcolor{leanCKw}{Type}*\textcolor{leanCBr1}{\}} \textcolor{leanCBr1}{[}TopologicalSpace X\textcolor{leanCBr1}{]}
  \textcolor{leanCBr1}{[}TopologicalSpace Y\textcolor{leanCBr1}{]} \textcolor{leanCBr1}{(}p : X \ensuremath{\rightarrow} Y\textcolor{leanCBr1}{)} \textcolor{leanCBr1}{(}hp : IsOpenMap p\textcolor{leanCBr1}{)}
  \textcolor{leanCBr1}{(}A : Set X\textcolor{leanCBr1}{)} \textcolor{leanCBr1}{(}hA : IsOpen A\textcolor{leanCBr1}{)} : IsOpenMap \textcolor{leanCBr1}{(}p \ensuremath{\circ} Subtype.val : A \ensuremath{\rightarrow} Y\textcolor{leanCBr1}{)} := \textcolor{leanCKw}{by}

  intro U hU
  simpa \textcolor{leanCBr1}{[}Function.comp_def, Set.image_image\textcolor{leanCBr1}{]} \textcolor{leanCKw}{using}
    candidate_assumed p hp \textcolor{leanCBr1}{(}Subtype.val '' U\textcolor{leanCBr1}{)} \textcolor{leanCBr1}{(}hA.isOpenMap_subtype_val U hU\textcolor{leanCBr1}{)}
\end{Verbatim}
\end{leancode}

The candidate's open-set argument is filled with \texttt{Subtype.val '' U},
where \texttt{U} is a set that exists only after the goal's own conclusion is
unfolded. The proof thus applies the antecedent to an object that it builds
itself, and it uses the candidate as a general schema rather than as the
statement that corresponds to the reference. Both statements bind six data
variables, namely the two spaces, their topologies, the map and the set, so the
auditor takes the first six data variables that the proof introduces as the
goal's own objects. The set \texttt{U} is the seventh, introduced by
\texttt{intro} once \texttt{IsOpenMap} is unfolded. The argument
\texttt{Subtype.val '' U} contains \texttt{U} and appears in the statement that
the antecedent delivers, so the check on data arguments
(Appendix~\ref{app:checker}) refuses the proof.

\paragraph{Not using the antecedent at all.} Artin 2.8.6 states that the centre of a product group is isomorphic to the product of the centres. The candidate states that the centre of the product, taken as a set, equals the product of the two centres as sets. An isomorphism between two groups and an equality of two sets are different claims, and the pair is not equivalent. The model proves the candidate from the reference in two lines.

\begin{leancode}
\begin{Verbatim}[commandchars=\\\{\},fontfamily=zi4,fontsize=\leancodesize,formatcom=\color{leanCFg},baselinestretch=1.0,xleftmargin=0pt]
\textcolor{leanCCmt}{-- ===== assumed statement (antecedent) =====}
\textcolor{leanCKw}{theorem} \textcolor{leanCDecl}{exercise_2_8_6_assumed} \textcolor{leanCBr1}{\{}G H : \textcolor{leanCKw}{Type}*\textcolor{leanCBr1}{\}} \textcolor{leanCBr1}{[}Group G\textcolor{leanCBr1}{]} \textcolor{leanCBr1}{[}Group H\textcolor{leanCBr1}{]} :
  Nonempty \textcolor{leanCBr1}{(}center \textcolor{leanCBr2}{(}G \ensuremath{\times} H\textcolor{leanCBr2}{)} \ensuremath{\simeq}* \textcolor{leanCBr2}{(}center G\textcolor{leanCBr2}{)} \ensuremath{\times} \textcolor{leanCBr2}{(}center H\textcolor{leanCBr2}{)}\textcolor{leanCBr1}{)} := \textcolor{leanCKw}{by}
  \textcolor{leanCSorry}{sorry}

\textcolor{leanCCmt}{-- ===== goal statement, PROVED from the antecedent =====}
\textcolor{leanCKw}{theorem} \textcolor{leanCDecl}{candidate} \textcolor{leanCBr1}{\{}G H : \textcolor{leanCKw}{Type}*\textcolor{leanCBr1}{\}} \textcolor{leanCBr1}{[}Group G\textcolor{leanCBr1}{]} \textcolor{leanCBr1}{[}Group H\textcolor{leanCBr1}{]} : \textcolor{leanCBr1}{(}center \textcolor{leanCBr2}{(}G \ensuremath{\times} H\textcolor{leanCBr2}{)} : Set \textcolor{leanCBr2}{(}G \ensuremath{\times} H\textcolor{leanCBr2}{)}\textcolor{leanCBr1}{)} =
    \textcolor{leanCBr1}{(}center G : Set G\textcolor{leanCBr1}{)} \ensuremath{\times}\ensuremath{\sp{\mathrm{s}}} \textcolor{leanCBr1}{(}center H : Set H\textcolor{leanCBr1}{)} := \textcolor{leanCKw}{by}
  rcases exercise_2_8_6_assumed \textcolor{leanCBr1}{(}G := G\textcolor{leanCBr1}{)} \textcolor{leanCBr1}{(}H := H\textcolor{leanCBr1}{)} \textcolor{leanCKw}{with} \textcolor{leanCBr1}{\ensuremath{\langle}}_e\textcolor{leanCBr1}{\ensuremath{\rangle}}
  exact Set.center_prod
\end{Verbatim}
\end{leancode}

The antecedent is extracted and bound to \texttt{\_e}, which never appears again, and the goal is then closed by the library theorem \texttt{Set.center\_prod} alone. The first line can be deleted and the proof still compiles. The auditor refuses the proof under two rules of Appendix~\ref{app:checker}, each of which suffices on its own. First, the proof destructures the antecedent and never uses what it extracts. The branch of the case split must then use the antecedent itself, and it does not. Second, \texttt{Set.center\_prod} is an equation, so the grading admits it as a bridge. The proof, however, applies it on the goal derivation outside any rewrite and without the antecedent. A bridge used in such a position states the goal outright, and the auditor refuses it.

Such rejected proofs prompted most of the label review
(Appendix~\ref{app:data}). The benchmark labels all three pairs above
equivalent, and the model's attempts, read as evidence rather than as
certificates, show where the statements of each pair differ. A person decided
every label.

\paragraph{A label changed after a revision loop with the auditor.} The auditor can also help to judge a statement, by revising a proof from the auditor's stated reasons and examining the faults that remain. IMO 1983 Problem 1 asks for all functions on the positive reals with a given functional equation and a limit condition. The Compfiles reference formalizes the domain with the subtype of positive reals, so positivity is carried by the type. The candidate was autoformalized blind from the natural-language statement, without sight of the reference. It works over the plain reals and restores positivity as a conjunct on the condition side of its biconditional. Both read as formalizations of the problem. A hand-written Lean proof establishes that the two are logically equivalent, and a strong model asked to compare the statements called them equivalent, on the ground that positivity follows from membership in the solution set. The pair therefore entered our recall suite labelled equivalent.

\begin{leancode}
\begin{Verbatim}[commandchars=\\\{\},fontfamily=zi4,fontsize=\leancodesize,formatcom=\color{leanCFg},baselinestretch=1.0,xleftmargin=0pt]
\textcolor{leanCCmt}{-- reference: positivity carried by the type \ensuremath{\mathbb{R}}+}
\textcolor{leanCKw}{abbrev} \textcolor{leanCDecl}{PosReal} : \textcolor{leanCKw}{Type} := \textcolor{leanCBr1}{\{} x : \ensuremath{\mathbb{R}} // \textcolor{leanCNum}{0} < x \textcolor{leanCBr1}{\}}

\textcolor{leanCKw}{notation} \textcolor{leanCStr}{"\ensuremath{\mathbb{R}}+"} => PosReal

\textcolor{leanCKw}{def} \textcolor{leanCDecl}{SolutionSet} : Set \textcolor{leanCBr1}{(}\ensuremath{\mathbb{R}}+ \ensuremath{\rightarrow} \ensuremath{\mathbb{R}}+\textcolor{leanCBr1}{)} := \textcolor{leanCBr1}{\{} \textcolor{leanCKw}{fun} x \ensuremath{\mapsto} \textcolor{leanCNum}{1} / x \textcolor{leanCBr1}{\}}
\textcolor{leanCKw}{theorem} \textcolor{leanCDecl}{Imo1983P1_refs} \textcolor{leanCBr1}{(}f : \ensuremath{\mathbb{R}}+ \ensuremath{\rightarrow} \ensuremath{\mathbb{R}}+\textcolor{leanCBr1}{)} :
    f \ensuremath{\in} SolutionSet \ensuremath{\leftrightarrow}
    \textcolor{leanCBr1}{(}\textcolor{leanCBr2}{(}\ensuremath{\forall} x y, f \textcolor{leanCBr3}{(}x * f y\textcolor{leanCBr3}{)} = y * f x\textcolor{leanCBr2}{)} \ensuremath{\wedge}
     \textcolor{leanCBr2}{(}\ensuremath{\forall} \ensuremath{\varepsilon}, \ensuremath{\exists} x, \ensuremath{\forall} y, x < y \ensuremath{\rightarrow} f y < \ensuremath{\varepsilon}\textcolor{leanCBr2}{)}\textcolor{leanCBr1}{)}
\textcolor{leanCCmt}{-- candidate: plain \ensuremath{\mathbb{R}}, positivity spelled out as a conjunct}
\textcolor{leanCKw}{def} \textcolor{leanCDecl}{Imo1983P1_solution} : Set \textcolor{leanCBr1}{(}\ensuremath{\mathbb{R}} \ensuremath{\rightarrow} \ensuremath{\mathbb{R}}\textcolor{leanCBr1}{)} := \textcolor{leanCBr1}{\{}f | \ensuremath{\forall} x : \ensuremath{\mathbb{R}}, \textcolor{leanCNum}{0} < x \ensuremath{\rightarrow} f x = \textcolor{leanCNum}{1} / x\textcolor{leanCBr1}{\}}
\textcolor{leanCKw}{theorem} \textcolor{leanCDecl}{Imo1983P1} \textcolor{leanCBr1}{(}f : \ensuremath{\mathbb{R}} \ensuremath{\rightarrow} \ensuremath{\mathbb{R}}\textcolor{leanCBr1}{)} :
    \textcolor{leanCBr1}{(}\textcolor{leanCBr2}{(}\ensuremath{\forall} x : \ensuremath{\mathbb{R}}, \textcolor{leanCNum}{0} < x \ensuremath{\rightarrow} \textcolor{leanCNum}{0} < f x\textcolor{leanCBr2}{)} \ensuremath{\wedge}
      \textcolor{leanCBr2}{(}\ensuremath{\forall} x y : \ensuremath{\mathbb{R}}, \textcolor{leanCNum}{0} < x \ensuremath{\rightarrow} \textcolor{leanCNum}{0} < y \ensuremath{\rightarrow} f \textcolor{leanCBr3}{(}x * f y\textcolor{leanCBr3}{)} = y * f x\textcolor{leanCBr2}{)} \ensuremath{\wedge}
      Tendsto f atTop \textcolor{leanCBr2}{(}nhds \textcolor{leanCNum}{0}\textcolor{leanCBr2}{)}\textcolor{leanCBr1}{)} \ensuremath{\leftrightarrow}
    f \ensuremath{\in} Imo1983P1_solution
\end{Verbatim}
\end{leancode}

The proofs of every prompt style were rejected in both directions, so we ran the auditor in a loop on the direction from the reference. After each rejection, an assistant language model in our development sessions wrote a new proof from the auditor's stated reasons alone. Four rounds removed every other fault that the auditor reported. One fault remained, which we judged to come from the statement rather than from the proof. The candidate's biconditional has positivity on its condition side, so its direction from membership in the solution set says that every solution is positive. The claim is true, and the library proves it. The reference never states it, however, because its functions are positive by type. The two placements are different claims, which coincide here only because the particular solution set entails positivity. A constraint that one statement carries in a type is re-encoded faithfully as a hypothesis, not as a conjunct. A control supports the diagnosis. On the candidate with positivity moved into a hypothesis and nothing else changed, both directions are certified, the direction from the reference at the first attempt and the reverse after three rounds. The reverse direction meets a hypothesis of the antecedent with a positivity fact under the tolerance of Appendix~\ref{app:checker}, the same kind of step as the reverse direction in Appendix~\ref{app:chains}. The hand-written witness and the strong model judged whether the conjunct is true, and the auditor judged where it comes from. The case marks a boundary of our criterion, which distinguishes a condition that the reference carries in a type from the same condition stated in the candidate's conclusion. Under our labelling criterion (Appendix~\ref{app:data}), the candidate makes a claim that the reference does not make, so we relabelled the pair non-equivalent and removed it from the recall suite, and all counts in the paper reflect the ruling. The final proof of the loop and the two proofs of the variant are printed below.

\begin{leancode}
\begin{Verbatim}[commandchars=\\\{\},fontfamily=zi4,fontsize=\leancodesize,formatcom=\color{leanCFg},baselinestretch=1.0,xleftmargin=0pt]
\textcolor{leanCCmt}{-- reference \ensuremath{\Rightarrow} candidate, the proof after four rounds of revision (still rejected)}
\textcolor{leanCCmt}{-- ===== assumed statement (antecedent) =====}
\textcolor{leanCKw}{theorem} \textcolor{leanCDecl}{Imo1983P1_refs_assumed} \textcolor{leanCBr1}{(}f : \ensuremath{\mathbb{R}}+ \ensuremath{\rightarrow} \ensuremath{\mathbb{R}}+\textcolor{leanCBr1}{)} :
    f \ensuremath{\in} SolutionSet \ensuremath{\leftrightarrow}
    \textcolor{leanCBr1}{(}\textcolor{leanCBr2}{(}\ensuremath{\forall} x y, f \textcolor{leanCBr3}{(}x * f y\textcolor{leanCBr3}{)} = y * f x\textcolor{leanCBr2}{)} \ensuremath{\wedge}
     \textcolor{leanCBr2}{(}\ensuremath{\forall} \ensuremath{\varepsilon}, \ensuremath{\exists} x, \ensuremath{\forall} y, x < y \ensuremath{\rightarrow} f y < \ensuremath{\varepsilon}\textcolor{leanCBr2}{)}\textcolor{leanCBr1}{)} := \textcolor{leanCKw}{by}
  \textcolor{leanCSorry}{sorry}

\textcolor{leanCCmt}{-- ===== goal statement, PROVED from the antecedent =====}
\textcolor{leanCKw}{theorem} \textcolor{leanCDecl}{Imo1983P1} \textcolor{leanCBr1}{(}f : \ensuremath{\mathbb{R}} \ensuremath{\rightarrow} \ensuremath{\mathbb{R}}\textcolor{leanCBr1}{)} :
    \textcolor{leanCBr1}{(}\textcolor{leanCBr2}{(}\ensuremath{\forall} x : \ensuremath{\mathbb{R}}, \textcolor{leanCNum}{0} < x \ensuremath{\rightarrow} \textcolor{leanCNum}{0} < f x\textcolor{leanCBr2}{)} \ensuremath{\wedge}
      \textcolor{leanCBr2}{(}\ensuremath{\forall} x y : \ensuremath{\mathbb{R}}, \textcolor{leanCNum}{0} < x \ensuremath{\rightarrow} \textcolor{leanCNum}{0} < y \ensuremath{\rightarrow} f \textcolor{leanCBr3}{(}x * f y\textcolor{leanCBr3}{)} = y * f x\textcolor{leanCBr2}{)} \ensuremath{\wedge}
      Tendsto f atTop \textcolor{leanCBr2}{(}nhds \textcolor{leanCNum}{0}\textcolor{leanCBr2}{)}\textcolor{leanCBr1}{)} \ensuremath{\leftrightarrow}
    f \ensuremath{\in} Imo1983P1_solution := \textcolor{leanCKw}{by}
  constructor
  \ensuremath{\cdot} rintro \textcolor{leanCBr1}{\ensuremath{\langle}}hf_pos, hf_eq, hf_lim\textcolor{leanCBr1}{\ensuremath{\rangle}}
    \textcolor{leanCKw}{let} g : \ensuremath{\mathbb{R}}+ \ensuremath{\rightarrow} \ensuremath{\mathbb{R}}+ := \textcolor{leanCKw}{fun} x => \textcolor{leanCBr1}{\ensuremath{\langle}}f x, hf_pos x x.property\textcolor{leanCBr1}{\ensuremath{\rangle}}
    \textcolor{leanCKw}{have} hiff := Imo1983P1_refs_assumed g
    \textcolor{leanCKw}{have} hg_eq : \ensuremath{\forall} x y, g \textcolor{leanCBr1}{(}x * g y\textcolor{leanCBr1}{)} = y * g x := \textcolor{leanCKw}{by}
      intro x y
      apply Subtype.ext
      exact hf_eq x y x.property y.property
    \textcolor{leanCKw}{have} hg_tail : \ensuremath{\forall} \ensuremath{\varepsilon}, \ensuremath{\exists} x, \ensuremath{\forall} y, x < y \ensuremath{\rightarrow} g y < \ensuremath{\varepsilon} := \textcolor{leanCKw}{by}
      intro \ensuremath{\varepsilon}
      obtain \textcolor{leanCBr1}{\ensuremath{\langle}}N, hN\textcolor{leanCBr1}{\ensuremath{\rangle}} := \textcolor{leanCBr1}{(}Metric.tendsto_atTop.1 hf_lim\textcolor{leanCBr1}{)} \ensuremath{\varepsilon} \ensuremath{\varepsilon}.property
      refine \textcolor{leanCBr1}{\ensuremath{\langle}}\textcolor{leanCBr2}{\ensuremath{\langle}}max N \textcolor{leanCNum}{1}, lt_of_lt_of_le zero_lt_one \textcolor{leanCBr3}{(}le_max_right N \textcolor{leanCNum}{1}\textcolor{leanCBr3}{)}\textcolor{leanCBr2}{\ensuremath{\rangle}}, ?_\textcolor{leanCBr1}{\ensuremath{\rangle}}
      intro y hxy
      change f y < \ensuremath{\varepsilon}
      \textcolor{leanCKw}{have} hyN : N \ensuremath{\leq} \textcolor{leanCBr1}{(}y : \ensuremath{\mathbb{R}}\textcolor{leanCBr1}{)} := le_trans \textcolor{leanCBr1}{(}le_max_left N \textcolor{leanCNum}{1}\textcolor{leanCBr1}{)} \textcolor{leanCBr1}{(}le_of_lt hxy\textcolor{leanCBr1}{)}
      simpa \textcolor{leanCBr1}{[}Real.dist_eq, abs_of_pos \textcolor{leanCBr2}{(}hf_pos y y.property\textcolor{leanCBr2}{)}\textcolor{leanCBr1}{]} \textcolor{leanCKw}{using} hN y hyN
    \textcolor{leanCKw}{have} himp := hiff.mpr
    \textcolor{leanCKw}{have} hg_mem := himp \textcolor{leanCBr1}{\ensuremath{\langle}}hg_eq, hg_tail\textcolor{leanCBr1}{\ensuremath{\rangle}}
    change \ensuremath{\forall} x : \ensuremath{\mathbb{R}}, \textcolor{leanCNum}{0} < x \ensuremath{\rightarrow} f x = \textcolor{leanCNum}{1} / x
    intro x hx
    exact congrArg Subtype.val \textcolor{leanCBr1}{(}congrFun hg_mem \textcolor{leanCBr2}{\ensuremath{\langle}}x, hx\textcolor{leanCBr2}{\ensuremath{\rangle}}\textcolor{leanCBr1}{)}
  \ensuremath{\cdot} intro hf_mem
    change \ensuremath{\forall} x : \ensuremath{\mathbb{R}}, \textcolor{leanCNum}{0} < x \ensuremath{\rightarrow} f x = \textcolor{leanCNum}{1} / x \textcolor{leanCKw}{at} hf_mem
    \textcolor{leanCKw}{let} g : \ensuremath{\mathbb{R}}+ \ensuremath{\rightarrow} \ensuremath{\mathbb{R}}+ := \textcolor{leanCKw}{fun} x => \textcolor{leanCNum}{1} / x
    \textcolor{leanCKw}{have} hiff := Imo1983P1_refs_assumed g
    \textcolor{leanCKw}{have} hmp := hiff.mp
    \textcolor{leanCKw}{have} hg := hmp rfl
    refine \textcolor{leanCBr1}{\ensuremath{\langle}}?_, ?_, ?_\textcolor{leanCBr1}{\ensuremath{\rangle}}
    \ensuremath{\cdot} intro x hx
      \textcolor{leanCKw}{have} hval : \textcolor{leanCBr1}{(}\textcolor{leanCBr2}{(}g \textcolor{leanCBr3}{\ensuremath{\langle}}x, hx\textcolor{leanCBr3}{\ensuremath{\rangle}} : \ensuremath{\mathbb{R}}+\textcolor{leanCBr2}{)} : \ensuremath{\mathbb{R}}\textcolor{leanCBr1}{)} = \textcolor{leanCNum}{1} / x := rfl
      rw \textcolor{leanCBr1}{[}hf_mem x hx, \ensuremath{\leftarrow} hval\textcolor{leanCBr1}{]}
      exact \textcolor{leanCBr1}{(}g \textcolor{leanCBr2}{\ensuremath{\langle}}x, hx\textcolor{leanCBr2}{\ensuremath{\rangle}}\textcolor{leanCBr1}{)}.\textcolor{leanCNum}{2}
    \ensuremath{\cdot} intro x y hx hy
      \textcolor{leanCKw}{have} h := congrArg Subtype.val \textcolor{leanCBr1}{(}hg.1 \textcolor{leanCBr2}{\ensuremath{\langle}}x, hx\textcolor{leanCBr2}{\ensuremath{\rangle}} \textcolor{leanCBr2}{\ensuremath{\langle}}y, hy\textcolor{leanCBr2}{\ensuremath{\rangle}}\textcolor{leanCBr1}{)}
      rw \textcolor{leanCBr1}{[}hf_mem y hy, hf_mem x hx,
        hf_mem \textcolor{leanCBr2}{(}x * \textcolor{leanCBr3}{(}\textcolor{leanCNum}{1} / y\textcolor{leanCBr3}{)}\textcolor{leanCBr2}{)} \textcolor{leanCBr2}{(}mul_pos hx \textcolor{leanCBr3}{(}one_div_pos.mpr hy\textcolor{leanCBr3}{)}\textcolor{leanCBr2}{)}\textcolor{leanCBr1}{]}
      exact h
    \ensuremath{\cdot} apply Metric.tendsto_atTop.2
      intro \ensuremath{\varepsilon} h\ensuremath{\varepsilon}
      obtain \textcolor{leanCBr1}{\ensuremath{\langle}}x, hx\textcolor{leanCBr1}{\ensuremath{\rangle}} := hg.2 \textcolor{leanCBr1}{\ensuremath{\langle}}\ensuremath{\varepsilon}, h\ensuremath{\varepsilon}\textcolor{leanCBr1}{\ensuremath{\rangle}}
      refine \textcolor{leanCBr1}{\ensuremath{\langle}}\textcolor{leanCBr2}{(}x : \ensuremath{\mathbb{R}}\textcolor{leanCBr2}{)} + \textcolor{leanCNum}{1}, ?_\textcolor{leanCBr1}{\ensuremath{\rangle}}
      intro y hy
      \textcolor{leanCKw}{have} hy_pos : \textcolor{leanCNum}{0} < y := lt_of_lt_of_le x.property \textcolor{leanCBr1}{(}le_trans
          \textcolor{leanCBr2}{(}le_add_of_nonneg_right zero_le_one\textcolor{leanCBr2}{)} hy\textcolor{leanCBr1}{)}
      \textcolor{leanCKw}{have} hxy : x < \textcolor{leanCBr1}{(}\textcolor{leanCBr2}{\ensuremath{\langle}}y, hy_pos\textcolor{leanCBr2}{\ensuremath{\rangle}} : \ensuremath{\mathbb{R}}+\textcolor{leanCBr1}{)} := \textcolor{leanCKw}{by}
        change \textcolor{leanCBr1}{(}x : \ensuremath{\mathbb{R}}\textcolor{leanCBr1}{)} < y
        exact lt_of_lt_of_le \textcolor{leanCBr1}{(}lt_add_one \textcolor{leanCBr2}{(}x : \ensuremath{\mathbb{R}}\textcolor{leanCBr2}{)}\textcolor{leanCBr1}{)} hy
      \textcolor{leanCKw}{have} htail := hx \textcolor{leanCBr1}{\ensuremath{\langle}}y, hy_pos\textcolor{leanCBr1}{\ensuremath{\rangle}} hxy
      \textcolor{leanCKw}{have} htail' : \textcolor{leanCBr1}{(}\textcolor{leanCNum}{1} / y : \ensuremath{\mathbb{R}}\textcolor{leanCBr1}{)} < \ensuremath{\varepsilon} := \textcolor{leanCKw}{by}
        exact_mod_cast htail
      rw \textcolor{leanCBr1}{[}hf_mem y hy_pos, Real.dist_eq, sub_zero,
        abs_of_pos \textcolor{leanCBr2}{(}one_div_pos.mpr hy_pos\textcolor{leanCBr2}{)}\textcolor{leanCBr1}{]}
      exact htail'
\end{Verbatim}
\end{leancode}
\begin{leancode}
\begin{Verbatim}[commandchars=\\\{\},fontfamily=zi4,fontsize=\leancodesize,formatcom=\color{leanCFg},baselinestretch=1.0,xleftmargin=0pt]
\textcolor{leanCCmt}{-- variant candidate (positivity as a hypothesis), reference \ensuremath{\Rightarrow} candidate (PASS)}
\textcolor{leanCCmt}{-- ===== assumed statement (antecedent) =====}
\textcolor{leanCKw}{theorem} \textcolor{leanCDecl}{Imo1983P1_refs_assumed} \textcolor{leanCBr1}{(}f : \ensuremath{\mathbb{R}}+ \ensuremath{\rightarrow} \ensuremath{\mathbb{R}}+\textcolor{leanCBr1}{)} :
    f \ensuremath{\in} SolutionSet \ensuremath{\leftrightarrow}
    \textcolor{leanCBr1}{(}\textcolor{leanCBr2}{(}\ensuremath{\forall} x y, f \textcolor{leanCBr3}{(}x * f y\textcolor{leanCBr3}{)} = y * f x\textcolor{leanCBr2}{)} \ensuremath{\wedge}
     \textcolor{leanCBr2}{(}\ensuremath{\forall} \ensuremath{\varepsilon}, \ensuremath{\exists} x, \ensuremath{\forall} y, x < y \ensuremath{\rightarrow} f y < \ensuremath{\varepsilon}\textcolor{leanCBr2}{)}\textcolor{leanCBr1}{)} := \textcolor{leanCKw}{by}
  \textcolor{leanCSorry}{sorry}

\textcolor{leanCCmt}{-- ===== goal statement, PROVED from the antecedent =====}
\textcolor{leanCKw}{theorem} \textcolor{leanCDecl}{Imo1983P1} \textcolor{leanCBr1}{(}f : \ensuremath{\mathbb{R}} \ensuremath{\rightarrow} \ensuremath{\mathbb{R}}\textcolor{leanCBr1}{)} \textcolor{leanCBr1}{(}hf_pos : \ensuremath{\forall} x : \ensuremath{\mathbb{R}}, \textcolor{leanCNum}{0} < x \ensuremath{\rightarrow} \textcolor{leanCNum}{0} < f x\textcolor{leanCBr1}{)} :
    \textcolor{leanCBr1}{(}\textcolor{leanCBr2}{(}\ensuremath{\forall} x y : \ensuremath{\mathbb{R}}, \textcolor{leanCNum}{0} < x \ensuremath{\rightarrow} \textcolor{leanCNum}{0} < y \ensuremath{\rightarrow} f \textcolor{leanCBr3}{(}x * f y\textcolor{leanCBr3}{)} = y * f x\textcolor{leanCBr2}{)} \ensuremath{\wedge}
      Tendsto f atTop \textcolor{leanCBr2}{(}nhds \textcolor{leanCNum}{0}\textcolor{leanCBr2}{)}\textcolor{leanCBr1}{)} \ensuremath{\leftrightarrow}
    f \ensuremath{\in} Imo1983P1_solution := \textcolor{leanCKw}{by}
  constructor
  \ensuremath{\cdot} rintro \textcolor{leanCBr1}{\ensuremath{\langle}}hf_eq, hf_lim\textcolor{leanCBr1}{\ensuremath{\rangle}}
    \textcolor{leanCKw}{let} g : \ensuremath{\mathbb{R}}+ \ensuremath{\rightarrow} \ensuremath{\mathbb{R}}+ := \textcolor{leanCKw}{fun} x => \textcolor{leanCBr1}{\ensuremath{\langle}}f x, hf_pos x x.property\textcolor{leanCBr1}{\ensuremath{\rangle}}
    \textcolor{leanCKw}{have} hiff := Imo1983P1_refs_assumed g
    \textcolor{leanCKw}{have} hg_eq : \ensuremath{\forall} x y, g \textcolor{leanCBr1}{(}x * g y\textcolor{leanCBr1}{)} = y * g x := \textcolor{leanCKw}{by}
      intro x y
      apply Subtype.ext
      exact hf_eq x y x.property y.property
    \textcolor{leanCKw}{have} hg_tail : \ensuremath{\forall} \ensuremath{\varepsilon}, \ensuremath{\exists} x, \ensuremath{\forall} y, x < y \ensuremath{\rightarrow} g y < \ensuremath{\varepsilon} := \textcolor{leanCKw}{by}
      intro \ensuremath{\varepsilon}
      obtain \textcolor{leanCBr1}{\ensuremath{\langle}}N, hN\textcolor{leanCBr1}{\ensuremath{\rangle}} := \textcolor{leanCBr1}{(}Metric.tendsto_atTop.1 hf_lim\textcolor{leanCBr1}{)} \ensuremath{\varepsilon} \ensuremath{\varepsilon}.property
      refine \textcolor{leanCBr1}{\ensuremath{\langle}}\textcolor{leanCBr2}{\ensuremath{\langle}}max N \textcolor{leanCNum}{1}, lt_of_lt_of_le zero_lt_one \textcolor{leanCBr3}{(}le_max_right N \textcolor{leanCNum}{1}\textcolor{leanCBr3}{)}\textcolor{leanCBr2}{\ensuremath{\rangle}}, ?_\textcolor{leanCBr1}{\ensuremath{\rangle}}
      intro y hxy
      change f y < \ensuremath{\varepsilon}
      \textcolor{leanCKw}{have} hyN : N \ensuremath{\leq} \textcolor{leanCBr1}{(}y : \ensuremath{\mathbb{R}}\textcolor{leanCBr1}{)} := le_trans \textcolor{leanCBr1}{(}le_max_left N \textcolor{leanCNum}{1}\textcolor{leanCBr1}{)} \textcolor{leanCBr1}{(}le_of_lt hxy\textcolor{leanCBr1}{)}
      simpa \textcolor{leanCBr1}{[}Real.dist_eq, abs_of_pos \textcolor{leanCBr2}{(}hf_pos y y.property\textcolor{leanCBr2}{)}\textcolor{leanCBr1}{]} \textcolor{leanCKw}{using} hN y hyN
    \textcolor{leanCKw}{have} himp := hiff.mpr
    \textcolor{leanCKw}{have} hg_mem := himp \textcolor{leanCBr1}{\ensuremath{\langle}}hg_eq, hg_tail\textcolor{leanCBr1}{\ensuremath{\rangle}}
    change \ensuremath{\forall} x : \ensuremath{\mathbb{R}}, \textcolor{leanCNum}{0} < x \ensuremath{\rightarrow} f x = \textcolor{leanCNum}{1} / x
    intro x hx
    exact congrArg Subtype.val \textcolor{leanCBr1}{(}congrFun hg_mem \textcolor{leanCBr2}{\ensuremath{\langle}}x, hx\textcolor{leanCBr2}{\ensuremath{\rangle}}\textcolor{leanCBr1}{)}
  \ensuremath{\cdot} intro hf_mem
    change \ensuremath{\forall} x : \ensuremath{\mathbb{R}}, \textcolor{leanCNum}{0} < x \ensuremath{\rightarrow} f x = \textcolor{leanCNum}{1} / x \textcolor{leanCKw}{at} hf_mem
    \textcolor{leanCKw}{let} g : \ensuremath{\mathbb{R}}+ \ensuremath{\rightarrow} \ensuremath{\mathbb{R}}+ := \textcolor{leanCKw}{fun} x => \textcolor{leanCNum}{1} / x
    \textcolor{leanCKw}{have} hiff := Imo1983P1_refs_assumed g
    \textcolor{leanCKw}{have} hmp := hiff.mp
    \textcolor{leanCKw}{have} hg := hmp rfl
    refine \textcolor{leanCBr1}{\ensuremath{\langle}}?_, ?_\textcolor{leanCBr1}{\ensuremath{\rangle}}
    \ensuremath{\cdot} intro x y hx hy
      \textcolor{leanCKw}{have} h := congrArg Subtype.val \textcolor{leanCBr1}{(}hg.1 \textcolor{leanCBr2}{\ensuremath{\langle}}x, hx\textcolor{leanCBr2}{\ensuremath{\rangle}} \textcolor{leanCBr2}{\ensuremath{\langle}}y, hy\textcolor{leanCBr2}{\ensuremath{\rangle}}\textcolor{leanCBr1}{)}
      rw \textcolor{leanCBr1}{[}hf_mem y hy, hf_mem x hx,
        hf_mem \textcolor{leanCBr2}{(}x * \textcolor{leanCBr3}{(}\textcolor{leanCNum}{1} / y\textcolor{leanCBr3}{)}\textcolor{leanCBr2}{)} \textcolor{leanCBr2}{(}mul_pos hx \textcolor{leanCBr3}{(}one_div_pos.mpr hy\textcolor{leanCBr3}{)}\textcolor{leanCBr2}{)}\textcolor{leanCBr1}{]}
      exact h
    \ensuremath{\cdot} apply Metric.tendsto_atTop.2
      intro \ensuremath{\varepsilon} h\ensuremath{\varepsilon}
      obtain \textcolor{leanCBr1}{\ensuremath{\langle}}x, hx\textcolor{leanCBr1}{\ensuremath{\rangle}} := hg.2 \textcolor{leanCBr1}{\ensuremath{\langle}}\ensuremath{\varepsilon}, h\ensuremath{\varepsilon}\textcolor{leanCBr1}{\ensuremath{\rangle}}
      refine \textcolor{leanCBr1}{\ensuremath{\langle}}\textcolor{leanCBr2}{(}x : \ensuremath{\mathbb{R}}\textcolor{leanCBr2}{)} + \textcolor{leanCNum}{1}, ?_\textcolor{leanCBr1}{\ensuremath{\rangle}}
      intro y hy
      \textcolor{leanCKw}{have} hy_pos : \textcolor{leanCNum}{0} < y := lt_of_lt_of_le x.property \textcolor{leanCBr1}{(}le_trans
          \textcolor{leanCBr2}{(}le_add_of_nonneg_right zero_le_one\textcolor{leanCBr2}{)} hy\textcolor{leanCBr1}{)}
      \textcolor{leanCKw}{have} hxy : x < \textcolor{leanCBr1}{(}\textcolor{leanCBr2}{\ensuremath{\langle}}y, hy_pos\textcolor{leanCBr2}{\ensuremath{\rangle}} : \ensuremath{\mathbb{R}}+\textcolor{leanCBr1}{)} := \textcolor{leanCKw}{by}
        change \textcolor{leanCBr1}{(}x : \ensuremath{\mathbb{R}}\textcolor{leanCBr1}{)} < y
        exact lt_of_lt_of_le \textcolor{leanCBr1}{(}lt_add_one \textcolor{leanCBr2}{(}x : \ensuremath{\mathbb{R}}\textcolor{leanCBr2}{)}\textcolor{leanCBr1}{)} hy
      \textcolor{leanCKw}{have} htail := hx \textcolor{leanCBr1}{\ensuremath{\langle}}y, hy_pos\textcolor{leanCBr1}{\ensuremath{\rangle}} hxy
      \textcolor{leanCKw}{have} htail' : \textcolor{leanCBr1}{(}\textcolor{leanCNum}{1} / y : \ensuremath{\mathbb{R}}\textcolor{leanCBr1}{)} < \ensuremath{\varepsilon} := \textcolor{leanCKw}{by}
        exact_mod_cast htail
      rw \textcolor{leanCBr1}{[}hf_mem y hy_pos, Real.dist_eq, sub_zero,
        abs_of_pos \textcolor{leanCBr2}{(}one_div_pos.mpr hy_pos\textcolor{leanCBr2}{)}\textcolor{leanCBr1}{]}
      exact htail'
\end{Verbatim}
\end{leancode}
\begin{leancode}
\begin{Verbatim}[commandchars=\\\{\},fontfamily=zi4,fontsize=\leancodesize,formatcom=\color{leanCFg},baselinestretch=1.0,xleftmargin=0pt]
\textcolor{leanCCmt}{-- variant candidate (positivity as a hypothesis), candidate \ensuremath{\Rightarrow} reference (PASS)}
\textcolor{leanCCmt}{-- ===== assumed statement (antecedent) =====}
\textcolor{leanCKw}{theorem} \textcolor{leanCDecl}{Imo1983P1_assumed} \textcolor{leanCBr1}{(}f : \ensuremath{\mathbb{R}} \ensuremath{\rightarrow} \ensuremath{\mathbb{R}}\textcolor{leanCBr1}{)} \textcolor{leanCBr1}{(}hf_pos : \ensuremath{\forall} x : \ensuremath{\mathbb{R}}, \textcolor{leanCNum}{0} < x \ensuremath{\rightarrow} \textcolor{leanCNum}{0} < f x\textcolor{leanCBr1}{)} :
    \textcolor{leanCBr1}{(}\textcolor{leanCBr2}{(}\ensuremath{\forall} x y : \ensuremath{\mathbb{R}}, \textcolor{leanCNum}{0} < x \ensuremath{\rightarrow} \textcolor{leanCNum}{0} < y \ensuremath{\rightarrow} f \textcolor{leanCBr3}{(}x * f y\textcolor{leanCBr3}{)} = y * f x\textcolor{leanCBr2}{)} \ensuremath{\wedge}
      Tendsto f atTop \textcolor{leanCBr2}{(}nhds \textcolor{leanCNum}{0}\textcolor{leanCBr2}{)}\textcolor{leanCBr1}{)} \ensuremath{\leftrightarrow}
    f \ensuremath{\in} Imo1983P1_solution := \textcolor{leanCKw}{by}
  \textcolor{leanCSorry}{sorry}
\textcolor{leanCCmt}{-- ===== goal statement, PROVED from the antecedent =====}
\textcolor{leanCKw}{theorem} \textcolor{leanCDecl}{Imo1983P1_refs} \textcolor{leanCBr1}{(}f : \ensuremath{\mathbb{R}}+ \ensuremath{\rightarrow} \ensuremath{\mathbb{R}}+\textcolor{leanCBr1}{)} :
    f \ensuremath{\in} SolutionSet \ensuremath{\leftrightarrow}
    \textcolor{leanCBr1}{(}\textcolor{leanCBr2}{(}\ensuremath{\forall} x y, f \textcolor{leanCBr3}{(}x * f y\textcolor{leanCBr3}{)} = y * f x\textcolor{leanCBr2}{)} \ensuremath{\wedge}
     \textcolor{leanCBr2}{(}\ensuremath{\forall} \ensuremath{\varepsilon}, \ensuremath{\exists} x, \ensuremath{\forall} y, x < y \ensuremath{\rightarrow} f y < \ensuremath{\varepsilon}\textcolor{leanCBr2}{)}\textcolor{leanCBr1}{)} := \textcolor{leanCKw}{by}
  \textcolor{leanCKw}{let} F : \ensuremath{\mathbb{R}} \ensuremath{\rightarrow} \ensuremath{\mathbb{R}} := \textcolor{leanCKw}{fun} x => \textcolor{leanCKw}{if} hx : \textcolor{leanCNum}{0} < x \textcolor{leanCKw}{then} \ensuremath{\uparrow}\textcolor{leanCBr1}{(}f \textcolor{leanCBr2}{\ensuremath{\langle}}x, hx\textcolor{leanCBr2}{\ensuremath{\rangle}}\textcolor{leanCBr1}{)} \textcolor{leanCKw}{else} \textcolor{leanCNum}{0}
  \textcolor{leanCKw}{have} hF_pos : \ensuremath{\forall} x : \ensuremath{\mathbb{R}}, \textcolor{leanCNum}{0} < x \ensuremath{\rightarrow} \textcolor{leanCNum}{0} < F x := \textcolor{leanCKw}{by}
    intro x hx
    \textcolor{leanCKw}{have} hFx : F x = \textcolor{leanCBr1}{(}f \textcolor{leanCBr2}{\ensuremath{\langle}}x, hx\textcolor{leanCBr2}{\ensuremath{\rangle}} : \ensuremath{\mathbb{R}}\textcolor{leanCBr1}{)} := dif_pos hx
    rw \textcolor{leanCBr1}{[}hFx\textcolor{leanCBr1}{]}
    exact \textcolor{leanCBr1}{(}f \textcolor{leanCBr2}{\ensuremath{\langle}}x, hx\textcolor{leanCBr2}{\ensuremath{\rangle}}\textcolor{leanCBr1}{)}.\textcolor{leanCNum}{2}
  \textcolor{leanCKw}{have} hiff := Imo1983P1_assumed F hF_pos
  \textcolor{leanCKw}{have} hmp := hiff.mp
  \textcolor{leanCKw}{have} hmpr := hiff.mpr
  constructor
  \ensuremath{\cdot} intro hf_mem
    \textcolor{leanCKw}{have} hF_mem : F \ensuremath{\in} Imo1983P1_solution := \textcolor{leanCKw}{by}
      \textcolor{leanCKw}{show} \ensuremath{\forall} x : \ensuremath{\mathbb{R}}, \textcolor{leanCNum}{0} < x \ensuremath{\rightarrow} F x = \textcolor{leanCNum}{1} / x
      intro x hx
      \textcolor{leanCKw}{have} hFx : F x = \textcolor{leanCBr1}{(}f \textcolor{leanCBr2}{\ensuremath{\langle}}x, hx\textcolor{leanCBr2}{\ensuremath{\rangle}} : \ensuremath{\mathbb{R}}\textcolor{leanCBr1}{)} := dif_pos hx
      \textcolor{leanCKw}{have} hval : \textcolor{leanCBr1}{(}f \textcolor{leanCBr2}{\ensuremath{\langle}}x, hx\textcolor{leanCBr2}{\ensuremath{\rangle}} : \ensuremath{\mathbb{R}}\textcolor{leanCBr1}{)} = \textcolor{leanCBr1}{(}\textcolor{leanCBr2}{(}\textcolor{leanCNum}{1} / \textcolor{leanCBr3}{(}\textcolor{leanCBr1}{\ensuremath{\langle}}x, hx\textcolor{leanCBr1}{\ensuremath{\rangle}} : \ensuremath{\mathbb{R}}+\textcolor{leanCBr3}{)} : \ensuremath{\mathbb{R}}+\textcolor{leanCBr2}{)} : \ensuremath{\mathbb{R}}\textcolor{leanCBr1}{)} :=
        congrArg Subtype.val \textcolor{leanCBr1}{(}congrFun hf_mem \textcolor{leanCBr2}{\ensuremath{\langle}}x, hx\textcolor{leanCBr2}{\ensuremath{\rangle}}\textcolor{leanCBr1}{)}
      \textcolor{leanCKw}{have} hdiv : \textcolor{leanCBr1}{(}\textcolor{leanCBr2}{(}\textcolor{leanCNum}{1} / \textcolor{leanCBr3}{(}\textcolor{leanCBr1}{\ensuremath{\langle}}x, hx\textcolor{leanCBr1}{\ensuremath{\rangle}} : \ensuremath{\mathbb{R}}+\textcolor{leanCBr3}{)} : \ensuremath{\mathbb{R}}+\textcolor{leanCBr2}{)} : \ensuremath{\mathbb{R}}\textcolor{leanCBr1}{)} = \textcolor{leanCNum}{1} / x := \textcolor{leanCKw}{by}
        simp
      rw \textcolor{leanCBr1}{[}hFx, hval, hdiv\textcolor{leanCBr1}{]}
    \textcolor{leanCKw}{have} hconj := hmpr hF_mem
    refine \textcolor{leanCBr1}{\ensuremath{\langle}}?_, ?_\textcolor{leanCBr1}{\ensuremath{\rangle}}
    \ensuremath{\cdot} intro x y
      \textcolor{leanCKw}{have} hkey := hconj.1 \textcolor{leanCBr1}{(}x : \ensuremath{\mathbb{R}}\textcolor{leanCBr1}{)} \textcolor{leanCBr1}{(}y : \ensuremath{\mathbb{R}}\textcolor{leanCBr1}{)} x.2 y.2
      \textcolor{leanCKw}{have} hFy : F \textcolor{leanCBr1}{(}y : \ensuremath{\mathbb{R}}\textcolor{leanCBr1}{)} = \textcolor{leanCBr1}{(}f y : \ensuremath{\mathbb{R}}\textcolor{leanCBr1}{)} := dif_pos y.2
      \textcolor{leanCKw}{have} hFx : F \textcolor{leanCBr1}{(}x : \ensuremath{\mathbb{R}}\textcolor{leanCBr1}{)} = \textcolor{leanCBr1}{(}f x : \ensuremath{\mathbb{R}}\textcolor{leanCBr1}{)} := dif_pos x.2
      \textcolor{leanCKw}{have} hFxy : F \textcolor{leanCBr1}{(}\textcolor{leanCBr2}{(}x : \ensuremath{\mathbb{R}}\textcolor{leanCBr2}{)} * \textcolor{leanCBr2}{(}f y : \ensuremath{\mathbb{R}}\textcolor{leanCBr2}{)}\textcolor{leanCBr1}{)} = \textcolor{leanCBr1}{(}f \textcolor{leanCBr2}{(}x * f y\textcolor{leanCBr2}{)} : \ensuremath{\mathbb{R}}\textcolor{leanCBr1}{)} :=
        dif_pos \textcolor{leanCBr1}{(}mul_pos x.2 \textcolor{leanCBr2}{(}f y\textcolor{leanCBr2}{)}.\textcolor{leanCNum}{2}\textcolor{leanCBr1}{)}
      rw \textcolor{leanCBr1}{[}hFy, hFx, hFxy\textcolor{leanCBr1}{]} \textcolor{leanCKw}{at} hkey
      exact Subtype.ext hkey
    \ensuremath{\cdot} intro \ensuremath{\varepsilon}
      obtain \textcolor{leanCBr1}{\ensuremath{\langle}}N, hN\textcolor{leanCBr1}{\ensuremath{\rangle}} := Metric.tendsto_atTop.1 hconj.2 \textcolor{leanCBr1}{(}\ensuremath{\varepsilon} : \ensuremath{\mathbb{R}}\textcolor{leanCBr1}{)} \ensuremath{\varepsilon}.2
      \textcolor{leanCKw}{have} hx\ensuremath{\sb{0}} : \textcolor{leanCBr1}{(}\textcolor{leanCNum}{0} < max N \textcolor{leanCNum}{1}\textcolor{leanCBr1}{)} = True := eq_true \textcolor{leanCBr1}{(}lt_of_lt_of_le one_pos \textcolor{leanCBr2}{(}le_max_right
          N \textcolor{leanCNum}{1}\textcolor{leanCBr2}{)}\textcolor{leanCBr1}{)}
      refine \textcolor{leanCBr1}{\ensuremath{\langle}}\textcolor{leanCBr2}{\ensuremath{\langle}}max N \textcolor{leanCNum}{1}, of_eq_true hx\ensuremath{\sb{0}}\textcolor{leanCBr2}{\ensuremath{\rangle}}, ?_\textcolor{leanCBr1}{\ensuremath{\rangle}}
      intro y hxy
      \textcolor{leanCKw}{have} hdy := hN \textcolor{leanCBr1}{(}y : \ensuremath{\mathbb{R}}\textcolor{leanCBr1}{)} \textcolor{leanCBr1}{(}le_of_lt \textcolor{leanCBr2}{(}lt_of_le_of_lt \textcolor{leanCBr3}{(}le_max_left N \textcolor{leanCNum}{1}\textcolor{leanCBr3}{)} hxy\textcolor{leanCBr2}{)}\textcolor{leanCBr1}{)}
      \textcolor{leanCKw}{have} hFy : F \textcolor{leanCBr1}{(}y : \ensuremath{\mathbb{R}}\textcolor{leanCBr1}{)} = \textcolor{leanCBr1}{(}f y : \ensuremath{\mathbb{R}}\textcolor{leanCBr1}{)} := dif_pos y.2
      \textcolor{leanCKw}{have} hdist : dist \textcolor{leanCBr1}{(}F \textcolor{leanCBr2}{(}y : \ensuremath{\mathbb{R}}\textcolor{leanCBr2}{)}\textcolor{leanCBr1}{)} \textcolor{leanCNum}{0} = \textcolor{leanCBr1}{(}f y : \ensuremath{\mathbb{R}}\textcolor{leanCBr1}{)} := \textcolor{leanCKw}{by}
        rw \textcolor{leanCBr1}{[}Real.dist_eq, sub_zero, hFy, abs_of_pos \textcolor{leanCBr2}{(}f y\textcolor{leanCBr2}{)}.\textcolor{leanCNum}{2}\textcolor{leanCBr1}{]}
      rw \textcolor{leanCBr1}{[}hdist\textcolor{leanCBr1}{]} \textcolor{leanCKw}{at} hdy
      exact hdy
  \ensuremath{\cdot} rintro \textcolor{leanCBr1}{\ensuremath{\langle}}hf_eq, hf_tail\textcolor{leanCBr1}{\ensuremath{\rangle}}
    \textcolor{leanCKw}{have} hF_eq : \ensuremath{\forall} x y : \ensuremath{\mathbb{R}}, \textcolor{leanCNum}{0} < x \ensuremath{\rightarrow} \textcolor{leanCNum}{0} < y \ensuremath{\rightarrow} F \textcolor{leanCBr1}{(}x * F y\textcolor{leanCBr1}{)} = y * F x := \textcolor{leanCKw}{by}
      intro x y hx hy
      \textcolor{leanCKw}{have} hFy : F y = \textcolor{leanCBr1}{(}f \textcolor{leanCBr2}{\ensuremath{\langle}}y, hy\textcolor{leanCBr2}{\ensuremath{\rangle}} : \ensuremath{\mathbb{R}}\textcolor{leanCBr1}{)} := dif_pos hy
      \textcolor{leanCKw}{have} hFx : F x = \textcolor{leanCBr1}{(}f \textcolor{leanCBr2}{\ensuremath{\langle}}x, hx\textcolor{leanCBr2}{\ensuremath{\rangle}} : \ensuremath{\mathbb{R}}\textcolor{leanCBr1}{)} := dif_pos hx
      \textcolor{leanCKw}{have} hFxy : F \textcolor{leanCBr1}{(}x * \textcolor{leanCBr2}{(}f \textcolor{leanCBr3}{\ensuremath{\langle}}y, hy\textcolor{leanCBr3}{\ensuremath{\rangle}} : \ensuremath{\mathbb{R}}\textcolor{leanCBr2}{)}\textcolor{leanCBr1}{)} = \textcolor{leanCBr1}{(}f \textcolor{leanCBr2}{(}\textcolor{leanCBr3}{(}\textcolor{leanCBr1}{\ensuremath{\langle}}x, hx\textcolor{leanCBr1}{\ensuremath{\rangle}} : \ensuremath{\mathbb{R}}+\textcolor{leanCBr3}{)} * f \textcolor{leanCBr3}{\ensuremath{\langle}}y, hy\textcolor{leanCBr3}{\ensuremath{\rangle}}\textcolor{leanCBr2}{)} : \ensuremath{\mathbb{R}}\textcolor{leanCBr1}{)} :=
        dif_pos \textcolor{leanCBr1}{(}mul_pos hx \textcolor{leanCBr2}{(}f \textcolor{leanCBr3}{\ensuremath{\langle}}y, hy\textcolor{leanCBr3}{\ensuremath{\rangle}}\textcolor{leanCBr2}{)}.\textcolor{leanCNum}{2}\textcolor{leanCBr1}{)}
      \textcolor{leanCKw}{have} hkey : \textcolor{leanCBr1}{(}f \textcolor{leanCBr2}{(}\textcolor{leanCBr3}{(}\textcolor{leanCBr1}{\ensuremath{\langle}}x, hx\textcolor{leanCBr1}{\ensuremath{\rangle}} : \ensuremath{\mathbb{R}}+\textcolor{leanCBr3}{)} * f \textcolor{leanCBr3}{\ensuremath{\langle}}y, hy\textcolor{leanCBr3}{\ensuremath{\rangle}}\textcolor{leanCBr2}{)} : \ensuremath{\mathbb{R}}\textcolor{leanCBr1}{)} = \textcolor{leanCBr1}{(}\textcolor{leanCBr2}{(}\textcolor{leanCBr3}{\ensuremath{\langle}}y, hy\textcolor{leanCBr3}{\ensuremath{\rangle}} : \ensuremath{\mathbb{R}}+\textcolor{leanCBr2}{)} * f \textcolor{leanCBr2}{\ensuremath{\langle}}x, hx\textcolor{leanCBr2}{\ensuremath{\rangle}} :
          \ensuremath{\mathbb{R}}\textcolor{leanCBr1}{)} :=
        congrArg Subtype.val \textcolor{leanCBr1}{(}hf_eq \textcolor{leanCBr2}{\ensuremath{\langle}}x, hx\textcolor{leanCBr2}{\ensuremath{\rangle}} \textcolor{leanCBr2}{\ensuremath{\langle}}y, hy\textcolor{leanCBr2}{\ensuremath{\rangle}}\textcolor{leanCBr1}{)}
      \textcolor{leanCKw}{have} hmul : \textcolor{leanCBr1}{(}\textcolor{leanCBr2}{(}\textcolor{leanCBr3}{\ensuremath{\langle}}y, hy\textcolor{leanCBr3}{\ensuremath{\rangle}} : \ensuremath{\mathbb{R}}+\textcolor{leanCBr2}{)} * f \textcolor{leanCBr2}{\ensuremath{\langle}}x, hx\textcolor{leanCBr2}{\ensuremath{\rangle}} : \ensuremath{\mathbb{R}}\textcolor{leanCBr1}{)} = y * \textcolor{leanCBr1}{(}f \textcolor{leanCBr2}{\ensuremath{\langle}}x, hx\textcolor{leanCBr2}{\ensuremath{\rangle}} : \ensuremath{\mathbb{R}}\textcolor{leanCBr1}{)} := rfl
      rw \textcolor{leanCBr1}{[}hFy, hFx, hFxy, hkey, hmul\textcolor{leanCBr1}{]}
    \textcolor{leanCKw}{have} hF_lim : Tendsto F atTop \textcolor{leanCBr1}{(}nhds \textcolor{leanCNum}{0}\textcolor{leanCBr1}{)} := \textcolor{leanCKw}{by}
      rw \textcolor{leanCBr1}{[}Metric.tendsto_atTop\textcolor{leanCBr1}{]}
      intro \ensuremath{\delta} h\ensuremath{\delta}
      obtain \textcolor{leanCBr1}{\ensuremath{\langle}}x, hx\textcolor{leanCBr1}{\ensuremath{\rangle}} := hf_tail \textcolor{leanCBr1}{\ensuremath{\langle}}\ensuremath{\delta}, h\ensuremath{\delta}\textcolor{leanCBr1}{\ensuremath{\rangle}}
      refine \textcolor{leanCBr1}{\ensuremath{\langle}}\textcolor{leanCBr2}{(}x : \ensuremath{\mathbb{R}}\textcolor{leanCBr2}{)} + \textcolor{leanCNum}{1}, ?_\textcolor{leanCBr1}{\ensuremath{\rangle}}
      intro y hy
      \textcolor{leanCKw}{have} hy_pos : \textcolor{leanCNum}{0} < y := lt_of_lt_of_le \textcolor{leanCBr1}{(}lt_trans x.2 \textcolor{leanCBr2}{(}lt_add_one \textcolor{leanCBr3}{(}x : \ensuremath{\mathbb{R}}\textcolor{leanCBr3}{)}\textcolor{leanCBr2}{)}\textcolor{leanCBr1}{)} hy
      \textcolor{leanCKw}{have} hxlt : \textcolor{leanCBr1}{(}x : \ensuremath{\mathbb{R}}\textcolor{leanCBr1}{)} < y := lt_of_lt_of_le \textcolor{leanCBr1}{(}lt_add_one \textcolor{leanCBr2}{(}x : \ensuremath{\mathbb{R}}\textcolor{leanCBr2}{)}\textcolor{leanCBr1}{)} hy
      \textcolor{leanCKw}{have} hfy := hx \textcolor{leanCBr1}{\ensuremath{\langle}}y, hy_pos\textcolor{leanCBr1}{\ensuremath{\rangle}} hxlt
      \textcolor{leanCKw}{have} hFy : F y = \textcolor{leanCBr1}{(}f \textcolor{leanCBr2}{\ensuremath{\langle}}y, hy_pos\textcolor{leanCBr2}{\ensuremath{\rangle}} : \ensuremath{\mathbb{R}}\textcolor{leanCBr1}{)} := dif_pos hy_pos
      \textcolor{leanCKw}{have} hdist : dist \textcolor{leanCBr1}{(}F y\textcolor{leanCBr1}{)} \textcolor{leanCNum}{0} = \textcolor{leanCBr1}{(}f \textcolor{leanCBr2}{\ensuremath{\langle}}y, hy_pos\textcolor{leanCBr2}{\ensuremath{\rangle}} : \ensuremath{\mathbb{R}}\textcolor{leanCBr1}{)} := \textcolor{leanCKw}{by}
        rw \textcolor{leanCBr1}{[}Real.dist_eq, sub_zero, hFy, abs_of_pos \textcolor{leanCBr2}{(}f \textcolor{leanCBr3}{\ensuremath{\langle}}y, hy_pos\textcolor{leanCBr3}{\ensuremath{\rangle}}\textcolor{leanCBr2}{)}.\textcolor{leanCNum}{2}\textcolor{leanCBr1}{]}
      rw \textcolor{leanCBr1}{[}hdist\textcolor{leanCBr1}{]}
      exact hfy
    \textcolor{leanCKw}{have} hF_mem := hmp \textcolor{leanCBr1}{\ensuremath{\langle}}hF_eq, hF_lim\textcolor{leanCBr1}{\ensuremath{\rangle}}
    \textcolor{leanCKw}{show} f = \textcolor{leanCKw}{fun} x => \textcolor{leanCNum}{1} / x
    funext x
    \textcolor{leanCKw}{have} hfx : F \textcolor{leanCBr1}{(}x : \ensuremath{\mathbb{R}}\textcolor{leanCBr1}{)} = \textcolor{leanCBr1}{(}f x : \ensuremath{\mathbb{R}}\textcolor{leanCBr1}{)} := dif_pos x.2
    \textcolor{leanCKw}{have} hFx1 : F \textcolor{leanCBr1}{(}x : \ensuremath{\mathbb{R}}\textcolor{leanCBr1}{)} = \textcolor{leanCNum}{1} / \textcolor{leanCBr1}{(}x : \ensuremath{\mathbb{R}}\textcolor{leanCBr1}{)} := hF_mem \textcolor{leanCBr1}{(}x : \ensuremath{\mathbb{R}}\textcolor{leanCBr1}{)} x.2
    \textcolor{leanCKw}{have} hdiv : \textcolor{leanCBr1}{(}\textcolor{leanCBr2}{(}\textcolor{leanCNum}{1} / x : \ensuremath{\mathbb{R}}+\textcolor{leanCBr2}{)} : \ensuremath{\mathbb{R}}\textcolor{leanCBr1}{)} = \textcolor{leanCNum}{1} / \textcolor{leanCBr1}{(}x : \ensuremath{\mathbb{R}}\textcolor{leanCBr1}{)} := \textcolor{leanCKw}{by}
      simp
    \textcolor{leanCKw}{have} hfinal : \textcolor{leanCBr1}{(}f x : \ensuremath{\mathbb{R}}\textcolor{leanCBr1}{)} = \textcolor{leanCBr1}{(}\textcolor{leanCBr2}{(}\textcolor{leanCNum}{1} / x : \ensuremath{\mathbb{R}}+\textcolor{leanCBr2}{)} : \ensuremath{\mathbb{R}}\textcolor{leanCBr1}{)} := \textcolor{leanCKw}{by}
      rw \textcolor{leanCBr1}{[}\ensuremath{\leftarrow} hfx, hFx1, hdiv\textcolor{leanCBr1}{]}
    exact Subtype.ext hfinal
\end{Verbatim}
\end{leancode}
\section{Reproducibility}
\label{app:repro}

\paragraph{Toolchain.} \Eval{} and \Checker{} run on Lean
\texttt{v4.30.0-rc1} with Mathlib at revision \texttt{0c154d67}. Every
statement in the paper is elaborated under this toolchain, and the
elaboration filter of Appendix~\ref{app:data} refers to it.

\paragraph{The \beqplus{} port.} The released implementation of \beqplus{}
is vendored verbatim and never edited. A copy carries exactly two
toolchain-compatibility changes, each marked in the source. The first maps
the retired tactic \texttt{simp\_all\_arith!} to its current spelling
\texttt{simp\_all!\ +arith +decide}, which is the same simplifier with the
same options and the only retired tactic the method uses. The second retunes
the memory containment of the per-pair Lean server for the larger current
environment. Nothing about the method changes. The ported copy produces
the \beqplus{} numbers of our tables.

\paragraph{The model campaign.} Proofs were written by OpenAI Codex CLI
0.151.0 with model \texttt{gpt-5.6-sol} at reasoning effort \texttt{medium},
six pairs in parallel and one session per pair and prompt style, on 30 August
2026, against Lean \texttt{v4.28.1} with Mathlib at revision \texttt{1f9fffd5},
the current release at the time. The returned proofs were then checked by
\Checker{} under our toolchain, and a direction whose proof no longer
elaborates there counts as a rejection. Of the 4{,}088 proof directions, the
model left 340 unproved, and 62 no longer elaborate on our toolchain. Per
style the model used between 80 and 85
million input tokens, of which 87\% were cache hits, and about 1.5 million
output tokens, at a wall time of about two hours.

\paragraph{Hardware and cost.} Every run of \Eval{} and \Checker{} was
produced on a single consumer laptop CPU with no GPU. The costs of
Table~\ref{tab:capability-full} are total CPU time over the 407 pairs, 3.45
hours for the base configuration and 30.67 hours for the strongest. Every
judgement is deterministic, as described in Appendix~\ref{app:eval}, so a
re-run reproduces every verdict, while the timings depend on the machine.

\paragraph{Release.} The code of both programs, the adjudicated labels of
the ProofNetVerif test split with the provenance of each change, the
ablation set and the recall suite, the four prompts, and every model-written proof
with its verdict are released with the paper.

\end{document}